\documentclass[11pt]{article}

\usepackage{microtype}

\usepackage{subcaption}
\usepackage{booktabs} % for professional tables
\usepackage{enumitem}

\usepackage[margin=1in,letterpaper]{geometry}
\usepackage[T1]{fontenc}
\usepackage{lmodern}
\usepackage{bbm}
\usepackage{amsmath}
\usepackage{amssymb}
\usepackage{amsthm}
\usepackage{graphicx}
\usepackage{algorithm}
\usepackage{algpseudocode}
\usepackage[hypertexnames=false]{hyperref}

\newenvironment{proofof}[1]{\noindent{\bf Proof of #1:}}{\hfill$\qed$\par}

\usepackage{tikz}

\usetikzlibrary{decorations.pathmorphing}
\usetikzlibrary{decorations.markings}
\usetikzlibrary{shapes.gates.logic.US,trees,positioning,arrows}

\tikzset{
	arn/.style = {circle, white, draw=black, fill=gray!30, inner sep = 10.5},
	arn_t/.style = {circle, white, draw=black, very thick, fill=gray!30, inner sep = 11.0},
	arn_l/.style = {circle, white, draw=black, very thick, fill=black, inner sep = 2},
	photon/.style={draw=black, very thick, dashed},
	electron/.style={draw=black, very thick},
	tr/.style={buffer gate US,thick,draw,fill=gray!60,rotate=90,	anchor=east,minimum width=2.25cm},
	br/.style={buffer gate US,thick,draw,fill=gray!60,rotate=90,	anchor=east,minimum width=4.5cm},
	brr/.style={buffer gate US,draw,fill=gray!60,rotate=90,	anchor=east,minimum width=4.5cm, opacity = 0.6},
	trr/.style={buffer gate US,thick,draw,fill=gray!60,rotate=90,	anchor=east,minimum width=2.25cm, opacity = 0.6},
	trrr/.style={buffer gate US,draw,fill=white!60,rotate=90,	anchor=east,minimum width=2.25cm, opacity = 0.5}
}

\newtheorem{theorem}{Theorem}
\newtheorem{lemma}{Lemma}
\theoremstyle{definition}
\newtheorem{definition}{Definition}[section]

\newcommand{\EE}{\mathbb{E}}
\theoremstyle{plain}
\newtheorem{thm}{\protect\theoremname}
\theoremstyle{plain}
\newtheorem{claim}[thm]{\protect\claimname}
\theoremstyle{plain}

\theoremstyle{plain}

\theoremstyle{plain}
\newtheorem{cor}[thm]{\protect\corollaryname}
\theoremstyle{definition}

\theoremstyle{definition}

\theoremstyle{definition}

\theoremstyle{plain}

\makeatother

\providecommand{\claimname}{Claim}
\providecommand{\lemmaname}{Lemma}
\providecommand{\propositionname}{Proposition}
\providecommand{\theoremname}{Theorem}
\providecommand{\corollaryname}{Corollary}
\providecommand{\definitionname}{Definition}
\providecommand{\assumptionname}{Assumption}
\providecommand{\remarkname}{Remark}

\global\long\def\RR{\mathbb{R}}

\newcommand{\NTKS}{{\sc NTKSketch} }

\newcommand{\CNTKS}{{\sc CNTKSketch} }

\newcommand{\TSRHT}{{\sc TensorSRHT} }

\newcommand{\SRHT}{{\sc SRHT} }

\newcommand{\PolyS}{{\sc PolySketch} }

\title{Learning with Neural Tangent Kernels in Near Input Sparsity Time}
\author{Amir Zandieh\\Max-Planck Institute for Informatics\\ azandieh@mpi-inf.mpg.de}

{\makeatletter
	\gdef\xxxmark{%
		\expandafter\ifx\csname @mpargs\endcsname\relax % in minipage?
		\expandafter\ifx\csname @captype\endcsname\relax % in figure/caption?
		\marginpar{xxx}% not in a caption or minipage, can use marginpar
		\else
		xxx % notice trailing space
		\fi
		\else
		xxx % notice trailing space-
		\fi}
	\gdef\xxx{\@ifnextchar[\xxx@lab\xxx@nolab}
	\long\gdef\xxx@lab[#1]#2{{\bf [\xxxmark #2 ---{\sc #1}]}}
	\long\gdef\xxx@nolab#1{{\bf [\xxxmark #1]}}
}

\begin{document}
\maketitle

\begin{abstract}
	The Neural Tangent Kernel (NTK) characterizes the behavior of infinitely wide neural nets trained under least squares loss by gradient descent~\cite{jacot2018neural}. 
	However, despite its importance, the super-quadratic runtime of kernel methods limits the use of NTK in large-scale learning tasks.
	To accelerate kernel machines with NTK, we propose a near input sparsity time algorithm that maps the input data to a randomized low-dimensional feature space so that the inner product of the transformed data approximates their NTK evaluation. Our transformation works by sketching the polynomial expansions of arc-cosine kernels.
	Furthermore, we propose a feature map for approximating the convolutional counterpart of the NTK~\cite{arora2019exact}, which can transform any image using a runtime that is only linear in the number of pixels.
	We show that in standard large-scale regression and classification tasks a linear regressor trained on our features outperforms trained Neural Nets and Nystrom approximation of NTK kernel.

\end{abstract}

\section{Introduction}
\label{sec:intro}
The Neural Tangent Kernel (NTK), introduced by \cite{jacot2018neural}, characterizes the behavior of infinitely wide fully connected neural nets trained by gradient descent. In fact, for least-squares loss, an ultra-wide network trained by gradient descent is equivalent to kernel regression with respect to the NTK. This kernel is shown to be central in theoretical understanding of the generalization and optimization properties of neural networks~\cite{cao2019generalization,neyshabur2015search,allen2019convergence,arora2019fine,du2018gradient}.
More recently, \cite{arora2019exact} extended this result by proving an analogous equivalence between convolutional neural nets with an infinite number of channels and Convolutional NTK (CNTK).
Besides the aforementioned theoretical advances, several papers have explored the algorithmic use of this kernel. ~\cite{arora2019harnessing} and ~\cite{geifman2020similarity} showed that NTK-based kernel models can outperform trained nets with finite width. 
Furthermore, \cite{arora2019exact} showed that kernel regression using CNTK sets an impressive performance record on CIFAR-10 for kernel methods without trainable kernels.

While plugging these kernels into kernelized learning methods, such as kernel regression or kernel SVM, can effectively capture the power of a fully-trained deep net by a pure kernel-based method, this approach faces serious scalability challenges as it operate on the kernel matrix (Gram matrix of data), whose size scales quadratically in the number of training samples.
In particular, for a dataset of $n$ images $x_1,x_2, \ldots x_n \in \RR^{d_1 \times d_2}$, only writing down the kernel matrix corresponding to the CNTK requires $\Omega\left(d_1^2 d_2^2 \cdot n^2\right)$ operations \cite{arora2019exact}. Running regression or PCA on the resulting kernel matrix takes additional cubic time in $n$.
We address this issue by designing feature maps with low-dimensional feature spaces so that the inner product between a pair of transformed data points approximates their kernel evaluation for the important NTK and CNTK kernels.

\paragraph{Problem Definition.} For the NTK  $\Theta_{ntk}:\RR^d \times \RR^d \to \RR$, we propose an explicit randomized feature map $\Psi_{ntk} : \RR^{d} \to \RR^s$ with small target dimension $s$ such that, for any pair of points $y,z \in \RR^d$:
\[ \Pr\left[ \left< \Psi_{ntk}(y) , \Psi_{ntk}(z) \right> \in (1\pm \epsilon) \cdot \Theta_{ntk}(y,z) \right] \ge 1-\delta, \]
We construct feature spaces that uniformly approximate the NTK $\Theta_{ntk}(\cdot , \cdot)$ to within $(1\pm \epsilon)$ multiplicative factor with probability $1 - \delta$ with only $s = O\left(\epsilon^{-2} \log \frac{1}{\delta} \right)$ dimensions. Furthermore, our feature map can transform any data point $x\in \RR^d$ using a runtime which is nearly proportional to  the sparsity of that point, i.e., $\widetilde{O}(\text{nnz}(x))$ runtime.
Additionally, we present a mapping that achieves analogous approximation guarantee for the CNTK corresponding to a convolutional deep net that can transform any image $x\in \RR^{d_1\times d_2}$ using a runtime that scales only linearly in the number of pixels of the image, i.e., $\widetilde{O}(d_1d_2)$ runtime.

Consequently, we can simply transform the input dataset with our feature maps, and then apply fast linear learning methods to approximate the answer of the corresponding nonlinear kernel method with NTK kernels. In particular, the kernel regression can be solved approximately in time $O(s^2 \cdot n)$, which is significantly faster than the exact kernel regression when $n$ is large.

\subsection{Overview of Techniques and Contributions}
One of our main results is an efficient data oblivious feature map for uniformly approximating the NTK using a runtime that is linearly proportional to the sparsity of the input dataset.
The key insight behind the design of our sketching methods is the fact that the NTK of a fully connected deep net with ReLU activation is a normalized dot-product kernel which can be fully characterized by a univariate function $K_{relu}:[-1,1] \to \RR$, plotted in Figure~\ref{fig:relu-ntk}. As shown by \cite{bietti2019inductive}, this function is obtained by composition and product of arc-cosine kernels, therefore, the NTK feature map is the tensor product of the features of arc-cosine kernels.

Our methods rely on sketching the feature space of arc-cosine kernels defined by their Taylor expansion. By careful truncation of their Taylor series, we approximate the arc-cosine kernels with bounded-degree polynomial kernels. 
Because the feature space of a polynomial kernel is the tensor product of its input space, its dimensionality is exponential in the degree of the kernel. Fortunately,~\cite{ahle2020oblivious} have developed a linear sketch known as \PolyS that can reduce the dimensionality of high-degree tensor products very efficiently, therefore, we can sketch the resulting polynomial kernels using this technique. We then combine the transformed features from consecutive layers by further sketching their tensor products. Note that na\"ive construction of the features would incur an exponential loss in the depth of the NTK, due to the tensor product of features generated in consecutive layers. Thus, utilizing \PolyS for efficiently sketching the tensor products is crucial. 

Our method also applies to the CNTK and can transform images using a runtime that is linear in the number of pixels of the input images.
The main difference between our sketching methods for NTK and CNTK is that in the latter case we have an extra operation which sketches the direct sum of the features of neighbouring pixels at each layer that precisely corresponds to the convolution operation in CNNs. We carefully analyze the errors introduced by polynomial approximations and various sketching steps in our algorithms and also bound their runtimes in Theorems~\ref{mainthm-ntk}~and~\ref{maintheorem-cntk}.

This is the first data oblivious method that can quickly sketch the NTK and CNTK kernels. In particular, our CNTK sketch maps an entire dataset of $n$ images of size $d_1 \times d_2$ pixels to an $s=\widetilde{O}(\epsilon^{-2})$-dimensional feature space using a total runtime of $O\left( \text{poly}(L, \epsilon^{-1}) \cdot d_1d_2 \cdot n \right)$. Thus, we can approximately solve the CNTK kernel regression problem in total time $O\left( \text{poly}(L, \epsilon^{-1}) \cdot d_1d_2 \cdot n \right)$. This is extremely faster than exact kernel regression, which takes $\Omega(L \cdot (d_1d_2)^2 \cdot n^2 + n^{3})$.

Finally, we investigate our NTK feature maps on standard large-scale regression and classification datasets and show that our method competes favorably with trained MLPs and Nystrom approximation of the NTK. Furthermore, we classify CIFAR-10 dataset at least 150$\times$ faster than exact CNTK.

\subsection{Related Work}
There has been a long line of work on the correspondence between deep or convolutional neural nets and kernel machines \cite{lee2017deep,matthews2018gaussian,novak2018bayesian,garriga2018deep,yang2019scaling}. Furthermore, in recent years there has been a lot of effort in understanding the convergence and generalization of neural networks in the over-parameterized regime \cite{du2019gradient,du2018gradient,li2018learning,du2019width,allen2019learning,allen2019convergence,cao2019generalization,arora2019fine,yang2019scaling,arora2019harnessing, bietti2019inductive}, all of which are related to the notion of neural tangent kernel \cite{jacot2018neural}.
More recently, \cite{arora2019exact} gave an efficient procedure to compute the convolutional extension of neural tangent kernel, which they name CNTK. 
%We remark that the over-parametrized neural network theory is not our focus and our aim is merely to accelerate kernel-based learning with NTK and CNTK kernels using sketching techniques. 

So far, there is little known about how to speed up learning with these kernels. 
\cite{novak2018bayesian} tried accelerating CNTK computations via Monte Carlo methods by taking the gradient of a randomly initialized CNN with respect to its weights. Although they do not theoretically bound the number of required features, the fully-connected version of this method is analyzed in \cite{arora2019exact}. Particularly, for the gradient features to approximate the NTK up to $\varepsilon$, the network width needs to be $\Omega(\frac{L^{6}}{\varepsilon^4} \log \frac{L}{\delta})$, thus, transforming a single vector $x\in \mathbb{R}^{d}$ requires $\Omega(\frac{L^{13}}{\varepsilon^8}  \log^2 \frac{L}{\delta} + \frac{L^{6}}{\varepsilon^4}  \log \frac{L}{\delta} \cdot {\rm nnz}(x))$ operations, which is slower than our Theorem~\ref{mainthm-ntk} by a factor of $L^3/\varepsilon^2$.
Furthermore, \cite{arora2019exact} shows that the performance of these random gradients is worse than exact CNTK by a large margin, in practice. 
More recently, \cite{lee2020generalized} proposed leverage score sampling for the NTK, however, their work is primarily theoretical and suggests no practical way of sampling the features.
Another line of work on NTK approximation is an explicit feature map construction via tensor product proposed by \cite{bietti2019inductive}. These explicit features can have infinite dimension in general and even if one uses a finite-dimensional approximation to the features, the computational gain of random features will be lost due to expensive tensor product operations.

In the literature, much work has focused on scaling up kernel methods by producing low-rank approximations to kernel matrices \cite{alaoui2015fast,avron2017faster,pmlr-v108-zandieh20a}. A popular approach for accelerating kernel methods is based on Nystrom sampling. We refer the reader to the work of~\cite{musco2017recursive} and the references therein. This method can efficiently produce a spectral approximation to the kernel matrix, however, it is not data oblivious and needs multiple passes over the entire dataset in order to select the landmarks. In contrast, instead of approximating the kernel matrix, our work approximates the kernel function directly using a \emph{data oblivious} random mapping, providing nice advantages, such as one-round distributed protocols and single-pass streaming algorithms.

Another popular line of work on kernel approximation problem is based on the Fourier features method~\cite{rahimi2007random}, which works well for shift-invariant kernels and with some modifications can embed the Gaussian kernel near optimally~\cite{avron2017random}. However, it only works for shift-invariant kernels and does not apply to NTK or CNTK.

In linear sketching literature,~\cite{avron2014subspace} proposed a subspace embedding for the polynomial kernel. The runtime of this method, while nearly linear in sparsity of the input dataset, scales exponentially in kernel's degree. Recently,~\cite{ahle2020oblivious} improved this exponential dependence to polynomial which enabled them to sketch high-degree polynomial kernels and led to near-optimal embeddings for Gaussian kernel. In fact, this sketching technology constitutes one of the main ingredients of our proposed methods. 
Additionally, combining sketching with leverage score sampling can improve the runtime of the polynomial kernel embeddings~\cite{woodruff2020near}.

\section{Preliminaries and Notations}

\begin{definition}[Tensor product]
	Given $x \in \RR^m$ and $y \in \RR^n$ we define the twofold tensor product $x \otimes y$ as 

	\small
	\[	x\otimes y = \begin{bmatrix} x_1 y_1 & x_1 y_2 & \cdots & x_1 y_n\\ x_2 y_1 & x_2 y_2 & \cdots & x_2 y_n\\ \vdots & \vdots &  & \vdots \\ x_m y_1 & x_m y_2 & \cdots & x_m y_n \end{bmatrix} \in \RR^{m\times n}.\]
	\normalsize
	Tensor product can be naturally extended to matrices, which results in $4$-dimensional objects, i.e., for $X \in \RR^{m\times n}$ and $Y \in \RR^{m'\times n'}$, the tensor product $X \otimes Y$ is in $\RR^{m\times n \times m' \times n'}$.\\
	Although tensor products are multidimensional objects, it is often convenient to associate them with single-dimensional vectors. In particular, we often associate $x\otimes y$ with a single dimensional vector $(x_1y_1,x_2y_1, \ldots x_my_1, x_1y_2,x_2y_2, \ldots x_my_2, \ldots x_my_n)$.\\
	Given $v_1 \in \RR^{d_1}, v_2 \in \RR^{d_2}, \ldots v_p \in \RR^{d_p}$ we define the $p$-fold tensor product $v_1 \otimes v_2 \otimes \ldots v_p \in \RR^{d_1d_2\ldots d_p}$ in the same fashion. 
	For shorthand, we use the notation $x^{\otimes p}$ to denote $\underbrace{x\otimes x \otimes \ldots x}_{p \text{ terms}}$, the $p$-fold self-tensoring of $x$.
\end{definition}

Another related operation that we use is the \emph{direct sum} of vectors: $x\oplus y:= \begin{bmatrix}
	x\\
	y
\end{bmatrix}$. 
We also use the notation $\odot$ to denote the hadamard product of tensors. For instance, the hadamard product of $3$-dimensional tensors $X , Y \in \RR^{m \times n \times d}$ is a tensor in $\RR^{m \times n \times d}$ defined as $[X\odot Y]_{i,j,l}:=X_{i,j,l}\cdot Y_{i,j,l}$. 
We need notation for \emph{slices} of a tensor. For instance, for a $3$-dimensional tensor $Y \in \RR^{m \times n \times d}$ and every $l \in [d]$, we denote by $y_{(:,:,l)}$ the $m \times n$ matrix that is defined as $\left[Y_{(:,:,l)}\right]_{i,j} := Y_{i,j,l}$ for $i \in[m], j \in [n]$. 
For two univariate functions $f$ and $g$ we denote their twofold composition by $f\circ g$, defined as $f\circ g(\alpha):= f(g(\alpha))$.
Finally, we use $\mathcal{N}(\mu,\sigma^2)$ to denote the normal distribution with mean $\mu$ and variance $\sigma^2$.

\subsection{Sketching Background}
We use the Subsampled Randomized Hadamard Transform (\SRHT)~\cite{ailon2009fast} to reduce the dimensionality of the intermediate vectors that arise in our computations. The \SRHT is a norm-preserving dimensionality reduction that can be computed in near linear time using the FFT algorithm.

\begin{lemma}[\SRHT Sketch]\label{lem:srht}
	For every positive integer $d$ and every $\epsilon, \delta>0$, there exists a distribution on random matrices $S \in \RR^{m \times d}$ with $m = O\left(\frac{1}{\epsilon^2} \cdot \log^2 \frac{1}{\epsilon\delta} \right)$, called \SRHT, such that for any vector $x \in \RR^{d}$,
	$\Pr\left[ \|S x\|_2^2 \in (1\pm \epsilon)\|x\|_2^2 \right] \ge 1 - \delta$.
	Moreover, $S x$ can be computed in time $O\left(\frac{1}{\epsilon^2} \cdot \log^2 \frac{1}{\epsilon\delta} + d\log d \right)$.
\end{lemma}

Our method relies on linear sketches that can be applied to high-degree polynomial kernels. We use the \PolyS introduced in \cite{ahle2020oblivious} which preserves the norm of vectors in $\RR^{d^p}$ and can be applied to tensor product vectors of the form $v_1 \otimes v_2 \otimes \ldots v_p$ very quickly.
The main building block of this sketch is \TSRHT which is a generalization of the \SRHT that can be applied to the tensor product of two vectors in near linear time using FFT algorithm~\cite{ahle2020oblivious}.
The \PolyS extends the idea behind \TSRHT to high-degree tensor products by recursively sketching pairs of vectors in a binary tree structure.
The following Lemma, summarizes Theorem 1.2 of \cite{ahle2020oblivious} and is proved in Appendix~\ref{appendix-sketch-prelims}.

\begin{lemma}[\PolyS]\label{soda-result}
	For every integers $p,d\ge 1$, every $\epsilon, \delta>0$, there exists a distribution on random matrices $Q^p \in \RR^{m \times d^p}$ with $m = O\left(\frac{p}{\epsilon^2}\log^3 \frac{1}{\epsilon\delta} \right)$, called \emph{degree-$p$ \PolyS}, such that the following hold,
	
	\begin{enumerate}[wide, labelwidth=!, labelindent=0pt]
		\item For any $y \in \RR^{d^p}$, $\Pr\left[ \|Q^p y\|_2^2 \in (1\pm \epsilon)\|y\|_2^2 \right] \ge 1 - \delta$.
		\item For any $x \in \RR^d$, if $e_1\in\RR^d$ is the standard basis vector along the first coordinate, the total time to compute $Q^p \left(x^{\otimes p-j} \otimes {e}_1^{\otimes j}\right)$ for all $j=0,1, \cdots p$ is
		$$O\left( {\frac{p^2 \log^2\frac{p}{\epsilon}}{\epsilon^2}} \log^3\frac{1}{\epsilon \delta} + \min\left\{\frac{p^{3/2}}{\epsilon}  \log \frac{1}{\delta} \cdot \text{nnz}(x), pd\log d\right\} \right).$$
		\item For any collection of vectors $v_1, v_2, \ldots v_p \in \RR^d$, the time to compute $Q^p \left(v_1 \otimes v_2 \otimes \ldots \otimes v_p\right)$ is bounded by $O\left( {\frac{p^2 \log\frac{p}{\epsilon}}{\epsilon^2}} \log^3\frac{1}{\epsilon \delta} + \frac{p^{3/2}}{\epsilon} \cdot d \cdot \log \frac{1}{\delta} \right)$.
	\end{enumerate}
\end{lemma}

\section{Fully Connected Neural Tangent Kernel}
The main result of this section is an efficient oblivious sketch for the Neural Tangent Kernel (NTK) corresponding to a fully connected multi-layer network.
We start by considering the expression for the fully connected NTK.
\cite{arora2019exact} showed how to exactly compute the $L$-layered NTK with activation $\sigma:\RR \to \RR$ using the following dynamic program:
\begin{enumerate}[wide, labelwidth=!, labelindent=0pt]
	\item For every $y,z\in \RR^d$, let $\Sigma^{(0)}(y,z) := \langle y , z\rangle$ and for every layer $h = 1,2, \ldots L $, recursively define the covariance $\Sigma^{(h)}: \RR^d \times \RR^d \to \RR$ as:
	\small
	\begin{equation}\label{eq:dp-covar}
		\begin{split}
			&\Lambda^{(h)}(y,z) := \begin{pmatrix}
				\Sigma^{(h-1)}(y,y) & \Sigma^{(h-1)}(y,z)\\
				\Sigma^{(h-1)}(z,y) & \Sigma^{(h-1)}(z,z)
			\end{pmatrix},\\
			&\Sigma^{(h)}(y,z) := \frac{\EE_{(u,v) \sim \mathcal{N}\left( 0, \Lambda^{(h)}(y,z) \right)} \left[ \sigma(u) \cdot \sigma(v) \right]}{\EE_{x\sim \mathcal{N}(0,1)} \left[ |\sigma(x)|^2 \right]}.
		\end{split}
	\end{equation}
	\normalsize
	\item For $h = 1,2, \ldots L$, define the derivative covariance as,
	\small
	\begin{equation}\label{eq:dp-derivative-covar}
		\dot{\Sigma}^{(h)}(y,z) := \frac{\EE_{(u,v) \sim \mathcal{N}\left( 0, \Lambda^{(h)}(y,z) \right)} \left[ \dot{\sigma}(u) \cdot \dot{\sigma}(v) \right]}{\EE_{x\sim \mathcal{N}(0,1)} \left[ |\sigma(x)|^2 \right]}.
	\end{equation}
	\normalsize
	\item Let $\Theta_{ntk}^{(0)}(y,z) := \Sigma^{(0)}(y,z)$ and for every integer $L \ge 1$, the depth-$L$ NTK expression is defined recursively as:
	\begin{equation}\label{eq:dp-ntk}
		\Theta_{ntk}^{(L)}(y,z) := \Theta_{ntk}^{(L-1)}(y,z) \cdot \dot{\Sigma}^{(L)}(y,z) + \Sigma^{(L)}(y,z).
	\end{equation}
\end{enumerate}

While using this DP, one can compute the kernel value $\Theta_{ntk}^{(L)}(y,z)$ for any pair of vectors $y,z \in \RR^{d}$ in $O(d + L)$ operations, it is hard to gain insight into the structure of this kernel using the expression above. 
In particular, the NTK expression involves recursive applications of nontrivial expectations, so it is unclear whether there exist efficient sketching solutions for this kernel in its current form.
However, we show that for the important case of ReLU activation $\sigma(\alpha) = \max(\alpha,0)$, this kernel takes an extremely nice and highly structured form. We prove that the NTK in this case can be characterized by the composition of arc-cosine kernels, which can be approximated effectively. In fact, the NTK can be fully characterized by a \emph{univariate function} $K_{relu}^{(L)}:[-1,1]\to \RR$, and exploiting this special structure is the key to designing efficient sketching methods for this kernel.

\subsection{ReLU-NTK}
The $L$-layered NTK corresponding to ReLU activation $\sigma(\alpha) = \max(\alpha,0)$ can be fully characterized by a univariate function $K_{relu}^{(L)}:[-1,1] \to \RR$ that we refer to as \emph{ReLU-NTK}. This function, which is closely related to the arc-cosine kernel functions~\cite{cho2009kernel} and was recently derived in \cite{bietti2019inductive} is defined as follows: 
\begin{definition}[ReLU-NTK function]\label{def:relu-ntk}
	For every integer $L>0$, the $L$-layered {\bf ReLU-NTK} function $K_{relu}^{(L)}:[-1,1] \to \RR$ is defined via the following procedure, for every $\alpha \in [-1,1]$:
	\begin{enumerate}[wide, labelwidth=!, labelindent=0pt]
		\item Let functions $\kappa_0(\cdot)$ and $\kappa_1(\cdot)$ be $0^{th}$ and $1^{st}$ order arc-cosine kernels~\cite{cho2009kernel} defined as follows,
		\begin{equation}\label{relu-activ-cov}
			\begin{split}
				\kappa_0(\alpha) := \frac{1}{\pi} \cdot \left( \pi - \arccos\left( \alpha \right) \right),~~ \kappa_1(\alpha) := \frac{1}{\pi} \left(\sqrt{1 - \alpha^2} + \alpha \cdot \left( \pi - \arccos\left( \alpha \right) \right) \right) .
			\end{split}
		\end{equation}
		
		\item Let $\Sigma_{relu}^{(0)}(\alpha) := \alpha$ and for every $h = 1,2, \ldots L$, define $\Sigma_{relu}^{(h)}(\alpha)$ as follows:
		\begin{equation}\label{eq:dp-covar-relu}
			\Sigma_{relu}^{(h)}(\alpha) := \underbrace{\kappa_1 \circ \kappa_1 \circ \dots \circ \kappa_1}_{h \text{-fold self composition}}  (\alpha).
		\end{equation}
	
		\item For every $h = 1,2, \ldots L $, define $\dot{\Sigma}_{relu}^{(h)}(\alpha)$ as follows:
		\begin{equation}\label{eq:dp-derivative-covar-relu}
			\dot{\Sigma}_{relu}^{(h)}(\alpha) := \kappa_0 \left( \Sigma_{relu}^{(h-1)}(\alpha) \right).
		\end{equation}
	
		\item Let $K_{relu}^{(0)}(\alpha) := \Sigma_{relu}^{(0)}(\alpha) = \alpha$ and for $h=1,2, \ldots L$, define $K_{relu}^{(h)}(\alpha)$ recursively as follows:
		\begin{equation}\label{eq:dp-ntk-relu}
			K_{relu}^{(h)}(\alpha) := K_{relu}^{(h-1)}(\alpha)\cdot \dot{\Sigma}_{relu}^{(h)}(\alpha) + \Sigma_{relu}^{(h)}(\alpha).
		\end{equation}
	\end{enumerate}
\end{definition}

The connection between ReLU-NTK function $K_{relu}^{(L)}$ and the kernel function $\Theta_{ntk}^{(L)}$ is formalized in the following theorem, which is proved in Appendix~\ref{appendix-relu-ntk}.
\begin{theorem}\label{thm:ntk-relu}
	For every integer $L\ge 1$, if we let $K_{relu}^{(L)}: [-1,1] \to \RR$ be the ReLU-NTK function as in Definition~\ref{def:relu-ntk}, then
	the \emph{Neural Tangent Kernel} (NTK) of fully-connected $L$-layered network with ReLU activation, $\sigma(\alpha) = \max(\alpha,0)$, satisfies the following for any $y,z \in \RR^d$,
	\[	\Theta_{ntk}^{(L)}(y,z) \equiv \|y\|_2 \|z\|_2 \cdot K_{relu}^{(L)}\left( \frac{\langle y, z \rangle}{\|y\|_2 \|z\|_2} \right).\]
\end{theorem}
Theorem~\ref{thm:ntk-relu} shows that the NTK is a \emph{normalized dot-product kernel} which can be fully characterized by $K_{relu}^{(L)}:[-1,1] \to \RR$. This function is smooth and can be efficiently computed using $O(L)$ operations at any desired point $\alpha \in [-1,1]$. We plot the family of ReLU-NTK functions for a range of values of $L \in\{ 2, 4, 8, 16, 32\}$ in Figure~\ref{fig:relu-ntk}. 
It is evident that for relatively larger values of $L$, the function $K_{relu}^{(L)}(\cdot)$ converges to a \emph{knee shape}. More precisely, it has a nearly constant value of roughly $0.3\cdot (L+1)$ on the interval $\alpha \in [-1,1 - O(1/L)]$, and on the interval $\alpha \in [1 - O(1/L), 1]$ its value sharply increases to $L+1$ at point $\alpha=1$.

\begin{figure}[!t]
	\begin{center}
		\begin{subfigure}{0.48\textwidth}
		\includegraphics[width=\textwidth]{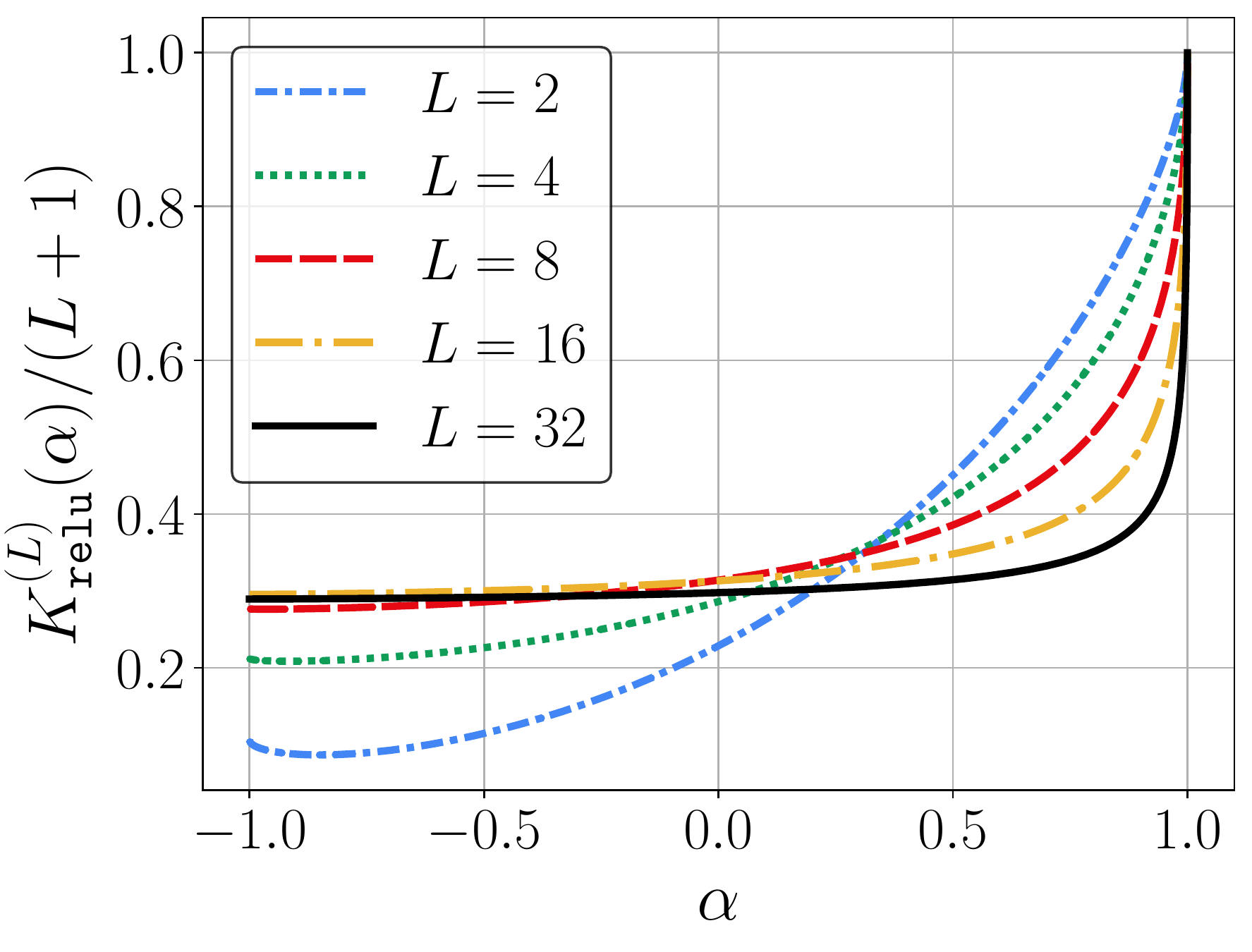}
		\end{subfigure}
\hskip 10pt
		\begin{subfigure}{0.48\textwidth}
		\includegraphics[width=\textwidth]{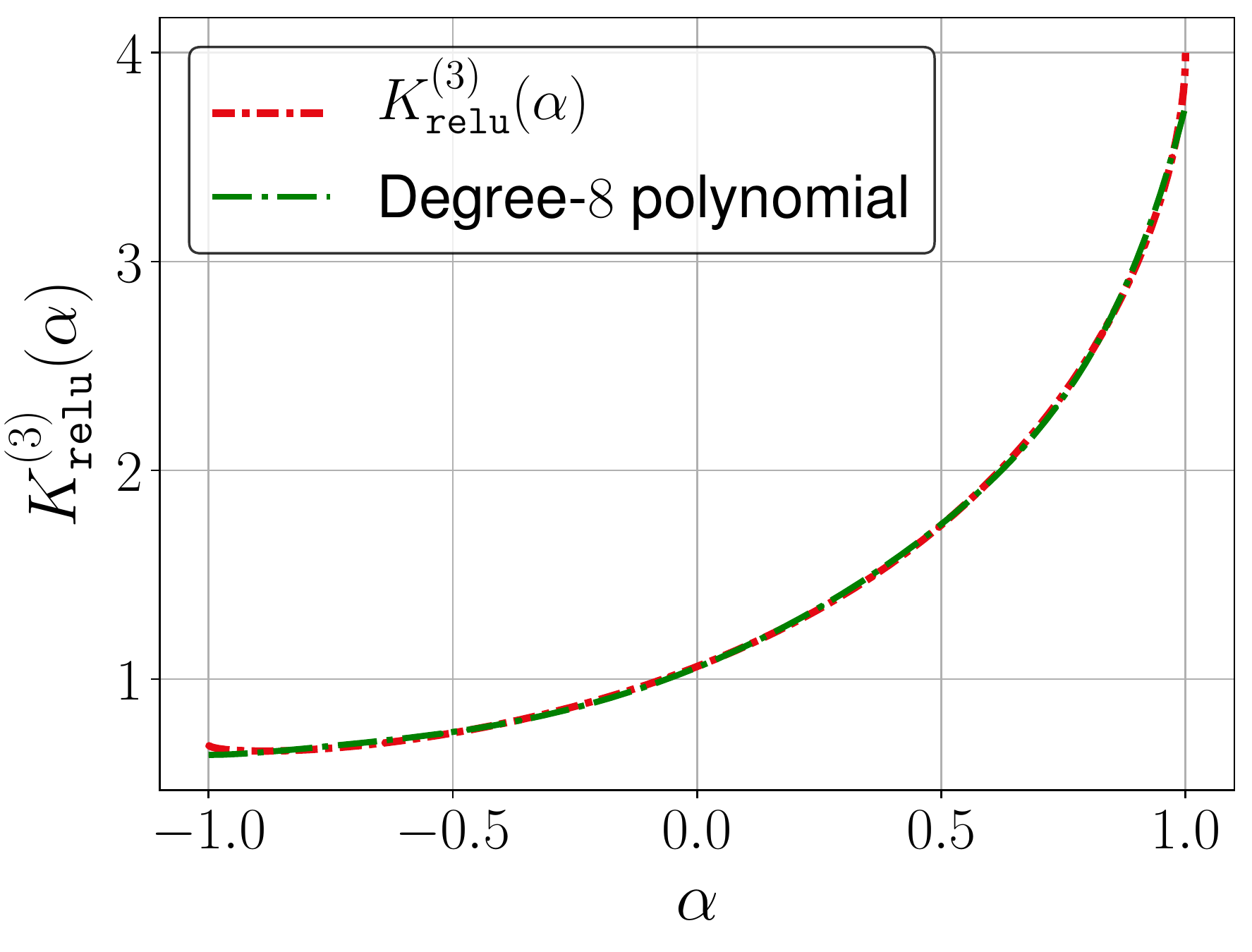}
		\end{subfigure}
		\vskip -0.1in
		\captionsetup{justification=centering,margin=0.1cm}
		\caption{(Left) Normalized ReLU-NTK function $\frac{1}{L+1} \cdot K_{relu}^{(L)}$ for various depths $L \in \{2,4,8,16,32\}$ and (Right) a degree-$8$ polynomial approximation of the depth-$3$ ReLU-NTK ($L=3$).} \label{fig:relu-ntk}
	\end{center}
\end{figure}

\subsection{NTK Sketch}
Given the tools we have introduced so far, we can design our oblivious sketch for the depth-$L$ NTK with ReLU activation in this section.
As shown in Definition~\ref{def:relu-ntk} and Theorem~\ref{thm:ntk-relu}, the NTK $\Theta_{ntk}^{(L)}$, is constructed by recursive composition of arc-cosine kernel functions $\kappa_1(\cdot)$ and $\kappa_0(\cdot)$. Thus, we crucially need efficient methods for approximating these functions in order to design efficient sketches for the NTK. 
We design our algorithms by applying fast sketching methods to the truncated Taylor series expansion of these functions.
Specifically, our main tool is approximating the arc-cosine kernels with low-degree polynomials, and then applying \PolyS to the resulting polynomial kernels. The features for multi-layered NTK are the recursive tensor product of arc-cosine sketches at consecutive layers, which in turn can be sketched efficiently using {\sc PolySketch}. We present our oblivious sketch in the following definition,

\begin{definition}[NTK Sketch Algorithm] \label{alg-def-ntk-sketch}
	For every input vector $x \in \RR^d$, every network depth $L$, and every error and failure parameters $\epsilon, \delta>0$, the {\bf NTK Sketch} $\Psi_{ntk}^{(L)}(x)$ is computed as follows,

	\begin{enumerate}[wide, labelwidth=!, labelindent=0pt]
		\item[$\bullet$] Choose integers $s = O\left(\frac{L^2}{\epsilon^2}  \log^2 \frac{L}{\epsilon\delta}\right)$, $n=O\left(\frac{L^6}{\epsilon^4} \log^3 \frac{L}{\epsilon\delta}\right)$, $n_1=O\left(\frac{L^4}{\epsilon^4} \log^3 \frac{L}{\epsilon\delta}\right)$, $r = O\left(\frac{L^6}{\epsilon^4} \log^2 \frac{L}{\epsilon\delta}\right)$, $m=O\left( \frac{L^8}{\epsilon^{16/3}} \log^3 \frac{L}{\epsilon\delta}\right)$, $m_2=O\left(\frac{L^2}{\epsilon^2} \log^3 \frac{L}{\epsilon\delta}\right)$, $s^*=O\left( \frac{1}{\epsilon^2} \log\frac{1}{\delta}\right)$ appropriately.

		\item[$\bullet$] For $p = \left\lceil 2L^2/\varepsilon^{{4}/{3}} \right\rceil$ and $p' = \left\lceil 9L^2/\varepsilon^{2} \right\rceil$, polynomials $P_{relu}^{(p)}(\cdot)$ and $\dot{P}_{relu}^{(p')}(\cdot)$ are defined as,
		\begin{equation}\label{eq:poly-approx-krelu}
			\begin{split}
				P_{relu}^{(p)}(\alpha) \equiv \sum_{j=0}^{2p+2} c_j \cdot \alpha^j &: = \frac{1}{\pi} + \frac{\alpha}{2}+ \frac{1}{\pi} \sum_{i=0}^p \frac{(2i)! \cdot \alpha^{2i+2}}{2^{2i}  (i!)^2  (2i+1) (2i+2)},\\
				\dot{P}_{relu}^{(p')}(\alpha) \equiv \sum_{j=0}^{2p'+1} b_j \cdot \alpha^j &: = \frac{1}{2} + \frac{1}{\pi} \sum_{i=0}^{p'} \frac{(2i)!}{2^{2i}  (i!)^2  (2i+1)} \cdot \alpha^{2i+1}.
			\end{split}
		\end{equation}

		\item Let $Q^{1} \in \RR^{n \times d}$ be an instance of the degree-$1$ \PolyS as per Lemma~\ref{soda-result}. Additionally, let $S \in \RR^{r \times n}$ be an instance of the SRHT as in Lemma~\ref{lem:srht}. Compute $\phi^{(0)}(x) \in \RR^r$ as follows,
		\begin{equation}\label{eq:map-covar-zero}
			\phi^{(0)}(x) \gets \frac{1}{\|x\|_2} \cdot S \cdot Q^{1} \cdot  x .
		\end{equation}
		\item 
		Let $Q^{2p+2} \in \RR^{m \times r^{2p+2}}$ be an instance of the degree-$(2p+2)$ \PolyS. 
		Also, let $T \in \RR^{r \times ((2p+3)\cdot m)}$ be an instance of the SRHT.
		For every $h = 1,2, \ldots L$ and $l=0,1,2, \ldots 2p+2$, compute the mappings $Z^{(h)}_{l}(x) \in \RR^{m}$ and $\phi^{(h)}(x) \in \RR^r$ as,
		\begin{equation}\label{eq:map-covar}
			\begin{split}
				&Z^{(h)}_{l}(x) \gets Q^{2p+2} \left(\left[ \phi^{(h-1)}(x) \right]^{\otimes l} \otimes  e_1^{\otimes 2p+2-l}\right) \\
				&\phi^{(h)}(x) \gets T \cdot \left( \bigoplus_{l=0}^{2p+2} \sqrt{c_l}  Z^{(h)}_l(x) \right).
			\end{split}
		\end{equation}

		\item 
		Let $Q^{2p'+1} \in \RR^{n_1 \times r^{2p'+1}}$ be an instance of the degree-$(2p'+1)$ \PolyS. 
		Also, let $W \in \RR^{s \times ((2p'+2)\cdot n_1)}$ be an instance of the SRHT.
		For every $h = 1,2, \ldots L $ and $l=0,1,2, \ldots 2p'+1$, compute the mappings $Y^{(h)}_{l}(x) \in \RR^{n_1}$ and $\dot{\phi}^{(h)}(x) \in \RR^s$ as,
		\begin{equation}\label{eq:map-derivative-covar}
			\begin{split}
				&Y^{(h)}_{l}(x) \gets Q^{2p'+1} \left(\left[ \phi^{(h-1)}(x) \right]^{\otimes l} \otimes  e_1^{\otimes 2p'+1-l}\right) \\
				&\dot{\phi}^{(h)}(x) \gets W \cdot \left( \bigoplus_{l=0}^{2p'+1} \sqrt{b_l}  Y^{(h)}_l(x) \right).
			\end{split}
		\end{equation}

		\item Let $Q^{2} \in \RR^{m_2 \times s^2}$ be an instance of the degree-$2$ \PolyS. Additionally, let $R \in \RR^{s \times (m_2+r)}$ and $V \in \RR^{s \times r}$ be independent instances of the SRHT sketch.
		For every integer $h = 1,2, \ldots L$, the mapping $\psi^{(h)}(x) \in \RR^s$ is computed recursively as:
		\begin{equation}\label{eq:map-relu}
			\begin{split}
				&\psi^{(0)}(x) \gets V \cdot \phi^{(0)}(x),\\
				&\psi^{(h)}(x) \gets R \cdot \left(Q^2 \left(\psi^{(h-1)}(x) \otimes \dot{\phi}^{(h)}(x)\right) \oplus \phi^{(h)}(x)\right).
			\end{split}
		\end{equation}

		\item Let $G \in \RR^{s^* \times s}$ be a matrix of i.i.d. normal entries with distribution $\mathcal{N}(0,1/s^*)$. Compute the mapping $\Psi_{ntk}^{(L)}(x) \in \RR^{s^*}$ as,
		\begin{equation}\label{Psi-ntk-def}
			\Psi_{ntk}^{(L)}(x) \gets {\|x\|_2} \cdot G \cdot \psi^{(L)}(x).
		\end{equation}
	\end{enumerate}
\end{definition}

Now we present our main theorem on NTK Sketch as follows,
\begin{theorem}\label{mainthm-ntk}
	For every integers $d\ge 1$ and $L\ge 2$, and any $\epsilon, \delta>0$, if we let $\Theta_{ntk}^{(L)}:\RR^d \times \RR^d \to \RR$ be the $L$-layered NTK with ReLU activation defined in \eqref{eq:dp-covar}, \eqref{eq:dp-derivative-covar}, and \eqref{eq:dp-ntk},
	then there exists a randomized map $\Psi_{ntk}^{(L)}: \RR^d \to \RR^{s^*}$ for some $s^* = O\left( \frac{1}{\epsilon^2}  \log \frac{1}{\delta} \right)$ such that the following invariants holds,
	\begin{enumerate}[wide, labelwidth=!, labelindent=0pt]
		\item For any vectors $y,z \in \RR^d$: 
		\[\Pr \left[ \left| \left< \Psi_{ntk}^{(L)}(y) , \Psi_{ntk}^{(L)}(z) \right> - \Theta_{ntk}^{(L)}(y,z) \right| > \epsilon\cdot \Theta_{ntk}^{(L)}(y,z) \right] \le \delta.\]

		\item For every vecor $x \in \RR^d$, the time to compute $\Psi_{ntk}^{(L)}(x)$ is $O\left( \frac{L^{11}}{\epsilon^{6.7}} \cdot \log^3 \frac{L}{\epsilon\delta} + \frac{L^3}{\epsilon^2} \cdot \log \frac{L}{\epsilon\delta} \cdot \text{nnz}(x) \right)$.
	\end{enumerate}
\end{theorem}
For a proof of the above theorem see Appendix~\ref{appendix-ntk-sketch}. 
One can observe that the runtime of our \NTKS is faster than the gradient features of an ultra-wide random DNN, studied by \cite{arora2019exact}, by a factor of $L^3/\varepsilon^2$.

\section{Convolutional Neural Tangent Kernel}
In this section, we design and analyze an efficient oblivious sketch for the Convolutional Neural Tangent Kernel (CNTK), which is the kernel function corresponding to a CNN with infinite number of channels.
\cite{arora2019exact} gave DP solutions for computing two variants of CNTK; one is the vanilla version which performs no pooling, and the other performs Global Average Pooling (GAP) on its top layer. For conciseness, we focus mainly on the CNTK with GAP, which also exhibits superior empirical performance~\cite{arora2019exact}. However, we remark that the vanilla CNTK has a very similar structure and hence our techniques can be applied to it, as well. 

Using the DP of \cite{arora2019exact}, the number of operations needed for exact computation of the depth-$L$ CNTK value $\Theta_{\tt cntk}^{(L)}(y,z)$ for images $y,z \in \RR^{d_1 \times d_1}$ is $\Omega\left( (d_1d_2)^2 \cdot L \right)$, which is extremely slow particularly due to its quadratic dependence on the number of pixels of input images $d_1d_2$.
Fortunately, we are able to show that the CNTK for the important case of ReLU activation is a highly structured object that can be fully characterized in terms of tensoring and composition of arc-cosine kernels, and exploiting this special structure is key in designing efficient sketching methods for the CNTK.

\subsection{ReLU-CNTK}
Here we present our derivation of the ReLU CNTK function.
Unlike the fully-connected NTK, the CNTK is not a simple dot-product kernel function. The key reason being that CNTK works by partitioning its input images into patches and locally transforming the patches at each layer, as opposed to the NTK which operates on the entire input vectors.
We prove that the depth-$L$ CNTK corresponding to ReLU activation is highly structured and can be fully characterized in terms of tensoring and composition of arc-cosine kernels $\kappa_1(\cdot)$ and $\kappa_0(\cdot)$. We show how to recursively compute the ReLU CNTK as follows,

\begin{definition}[ReLU-CNTK] \label{relu-cntk-def}
	For every positive integers $q,L$, the $L$-layered CNTK for ReLU activation function and convolutional filter size of $q \times q$ is defined as follows
	\begin{enumerate}[wide, labelwidth=!, labelindent=0pt]
		\item For $x\in \RR^{d_1\times d_2 \times c}$, every $i \in [d_1]$ and $j \in [d_2]$ let $N_{i,j}^{(0)} (x) :=q^2 \cdot \sum_{l=1}^c \left| x_{i+a,j+b,l} \right|^2$, and for every $h \ge 1$, recursively define,
		\begin{equation}\label{eq:dp-cntk-norm-simplified}
			N^{(h)}_{i,j}(x):=\frac{1}{q^2} \cdot \sum_{a=-\frac{q-1}{2}}^{\frac{q-1}{2}} \sum_{b=-\frac{q-1}{2}}^{\frac{q-1}{2}} N^{(h-1)}_{i+a,j+b}(x).
		\end{equation}
		\item Define $\Gamma^{(0)}(y,z)  := \sum_{l=1}^c y_{(:,:,l)} \otimes z_{(:,:,l)}$. Let $\kappa_1:[-1,1]\to \RR$ be the function defined in \eqref{relu-activ-cov} of Definition~\ref{def:relu-ntk}. For every layer $h = 1,2, \ldots , L$, every $i,i' \in [d_1]$ and $j,j' \in [d_2]$, define $\Gamma^{(h)}: \RR^{d_1\times d_2 \times c} \times \RR^{d_1 \times d_2 \times c} \to \RR^{d_1 \times d_2 \times d_1 \times d_2}$ recursively as:
		\begin{equation}\label{eq:dp-cntk-covar-simplified}
			\Gamma^{(h)}_{i,j,i',j'}(y,z) := \frac{\sqrt{N^{(h)}_{i,j}(y) \cdot N^{(h)}_{i',j'}(z)}}{q^2} \cdot \kappa_1\left( A \right),
		\end{equation}
		where $A:=\frac{\sum_{a=-\frac{q-1}{2}}^{\frac{q-1}{2}} \sum_{b=-\frac{q-1}{2}}^{\frac{q-1}{2}}  \Gamma^{(h-1)}_{i+a,j+b,i'+a,j'+b}(y,z)}{\sqrt{N^{(h)}_{i,j}(y) \cdot N^{(h)}_{i',j'}(z)}}$.
		\item Let $\kappa_0:[-1,1]\to \RR$ be the function defined in \eqref{relu-activ-cov} of Definition~\ref{def:relu-ntk}. For every $h = 1,2, \ldots L$, every $i,i' \in [d_1]$ and $j,j' \in [d_2]$, define $\dot{\Gamma}^{(h)}(y,z) \in \RR^{d_1 \times d_2 \times d_1 \times d_2}$ as:
		\begin{equation}\label{eq:dp-cntk-derivative-covar-simplified}
			\dot{\Gamma}^{(h)}_{i,j,i',j'}(y,z) := \frac{1}{q^2} \cdot \kappa_0\left(A\right),
		\end{equation}
		where $A$ is defined in \eqref{eq:dp-cntk-covar-simplified}.
		\item Let $\Pi^{(0)}(y,z) := 0$ and for every $h = 1,2, \ldots, L-1$, every $i,i' \in [d_1]$ and $j,j' \in [d_2]$, define $\Pi^{(h)}: \RR^{d_1\times d_2 \times c} \times \RR^{d_1\times d_2 \times c} \to \RR^{d_1\times d_2\times d_1 \times d_2}$ recursively as:
		\small
		\begin{equation}\label{eq:dp-cntk}
			\Pi^{(h)}_{i,j,i',j'}(y,z) := \sum_{a=-\frac{q-1}{2}}^{\frac{q-1}{2}} \sum_{b=-\frac{q-1}{2}}^{\frac{q-1}{2}}  B_{i+a,j+b,i'+a,j'+b},
		\end{equation}
		\normalsize
		where $B :=\Pi^{(h-1)}(y,z) \odot \dot{\Gamma}^{(h)}(y,z) + \Gamma^{(h)}(y,z)$.\\
		Furthermore, for $h=L$ define: 
		\begin{equation}\label{eq:dp-cntk-last-layer}
			\Pi^{(L)}(y,z) := \Pi^{(L-1)}(y,z) \odot \dot{\Gamma}^{(L)}(y,z).
		\end{equation}
		\item The final CNTK expressions for ReLU activation is:
		\begin{equation}\label{eq:dp-cntk-finalkernel}
			\Theta_{cntk}^{(L)}(y,z) := \frac{1}{d_1^2d_2^2} \cdot \sum_{i , i' \in [d_1]} \sum_{j , j' \in [d_2]} \Pi_{i,j,i',j'}^{(L)}(y,z).
		\end{equation}
	\end{enumerate}
\end{definition}

We present the main properties of the ReLU-CNTK function in Appendix~\ref{appendix-relu-cntk-expr}.

\subsection{CNTK Sketch}
Similar to NTK Sketch, our method relies on approximating the arc-cosine kernels with low-degree polynomials via Taylor expansion, and then applying \PolyS to the resulting polynomial kernels. Our sketch computes the features for each pixel of the input image, by tensor product of arc-cosine sketches at consecutive layers, which in turn can be sketched efficiently by \textsc{PolySketch}. Additionally, the features of pixels that lie in the same patch get \emph{locally combined} at each layer via direct sum operation. This precisely corresponds to the convolution operation in neural networks. We start by presenting our CNTK Sketch algorithm in the following definition,

\begin{definition}[CNTK Sketch Algorithm] \label{alg-def-cntk-sketch}
	For every image $x \in \RR^{d_1 \times d_2 \times c}$, compute the \emph{CNTK Sketch}, $\Psi_{cntk}^{(L)}(x)$, recursively as follows,
	\begin{enumerate}[wide, labelwidth=!, labelindent=0pt]
		\item[$\bullet$] Let $s = O\left(\frac{L^2}{\epsilon^2}  \log^2 \frac{d_1d_2L}{\epsilon\delta}\right)$, $r = O\left(\frac{L^6}{\epsilon^4}  \log^2 \frac{d_1d_2L}{\epsilon\delta}\right)$, $m_2=O\left(\frac{L^2}{\epsilon^2}  \log^3 \frac{d_1d_2L}{\epsilon\delta}\right)$, $n=O\left(\frac{L^4}{\epsilon^4} \log^3 \frac{d_1d_2L}{\epsilon\delta}\right)$, $m=O\left( \frac{L^8}{\epsilon^{16/3}} \log^3 \frac{d_1d_2L}{\epsilon\delta}\right)$, and $s^* = O(\frac{1}{\epsilon^2} \log\frac{1}{\delta})$ be appropriate integers and $P^{(p)}_{relu}(\alpha) = \sum_{l=0}^{2p+2} c_l \cdot \alpha^l$ and $\dot{P}^{(p')}_{relu}(\alpha) = \sum_{l=0}^{2p'+1} b_l \cdot \alpha^l$ be the polynomials defined in \eqref{eq:poly-approx-krelu}.
		
		\item For every $i \in [d_1]$, $j \in [d_2]$, and $h = 0, 1, 2, \ldots L$ compute $N_{i,j}^{(h)}(x)$ as per \eqref{eq:dp-cntk-norm-simplified} of Definition~\ref{relu-cntk-def}.
		
		\item Let $S \in \RR^{r \times c}$ be an SRHT. For every $i \in [d_1]$ and $j \in [d_2]$, compute $\phi_{i,j}^{(0)}(x) \in \RR^r$ as,
		\begin{equation}\label{cntk-sketch-covar-zero}
			\phi_{i,j}^{(0)}(x) \gets S \cdot x_{(i,j,:)}.
		\end{equation}
		
		\item 
		Let $Q^{2p+2} \in \RR^{m \times \left(q^2r\right)^{2p+2}}$ be a degree-$(2p+2)$ \PolyS, and let $T \in \RR^{r \times ((2p+3)\cdot m)}$ be an SRHT.
		For every layer $h = 1,2, \ldots L$, every $i \in [d_1]$ and $j \in [d_2]$, and $l=0,1,2, \ldots 2p+2$ compute,
		\begin{equation}\label{eq:maping-cntk-covar}
			\begin{split}
				&\mu^{(h)}_{i,j}(x) \gets \frac{1}{\sqrt{N^{(h)}_{i,j}(x)}} \cdot  \bigoplus_{a=-\frac{q-1}{2}}^{\frac{q-1}{2}} \bigoplus_{b=-\frac{q-1}{2}}^{\frac{q-1}{2}}  \phi_{i+a,j+b}^{(h-1)}(x)\\
				&\left[Z^{(h)}_{i,j}(x)\right]_l \gets Q^{2p+2} \cdot \left(\left[ \mu^{(h)}_{i,j}(x) \right]^{\otimes l} \otimes e_1^{\otimes 2p+2-l}\right), \\
				&\phi_{i,j}^{(h)}(x) \gets \frac{\sqrt{N^{(h)}_{i,j}(x)}}{q} \cdot T \cdot \left( \bigoplus_{l=0}^{2p+2} \sqrt{c_l}  \left[Z^{(h)}_{i,j}(x)\right]_l \right).
			\end{split}
		\end{equation}

		\item 
		Let $Q^{2p'+1} \in \RR^{n \times \left(q^2r\right)^{2p'+1}}$ be a degree-$(2p'+1)$ \PolyS, and let $W \in \RR^{s \times ((2p'+2)\cdot n)}$ be an SRHT.
		For every $h = 1,2, \ldots L$, every $i \in [d_1]$ and $j \in [d_2]$, and $l=0,1,2, \ldots 2p'+1$ compute,
		\begin{equation}\label{eq:cntk-map-phidot}
			\begin{split}
				&\left[Y^{(h)}_{i,j}(x)\right]_l \gets Q^{2p'+1} \cdot \left(\left[ \mu^{(h)}_{i,j}(x) \right]^{\otimes l} \otimes e_1^{\otimes 2p'+1-l}\right), \\
				&\dot{\phi}_{i,j}^{(h)}(x) \gets \frac{1}{q} \cdot W \cdot \left( \bigoplus_{l=0}^{2p'+1} \sqrt{b_l}  \left[Y^{(h)}_{i,j}(x)\right]_l \right),
			\end{split}
		\end{equation}
		where $\mu^{(h)}_{i,j}(x)$ is computed in \eqref{eq:maping-cntk-covar}.
		
		\item Let $Q^{2} \in \RR^{m_2 \times s^2}$ be a degree-$2$ \PolyS, and let $R \in \RR^{s \times \left(q^2(m_2+r)\right)}$ be an SRHT. 
		Let $\psi^{(0)}_{i,j}(x) \gets 0$ and for every $h \in [L-1]$, and $i \in [d_1], j \in [d_2]$, compute the mapping $\psi^{(h)}_{i,j}(x) \in \RR^s$ as:
		\begin{equation}\label{psi-cntk}
			\begin{split}
				&\eta^{(h)}_{i,j}(x) \gets Q^2 \cdot \left(\psi_{i,j}^{(h-1)}(x) \otimes \dot{\phi}_{i,j}^{(h)}(x)\right) \oplus \phi_{i,j}^{(h)}(x),\\
				&\psi^{(h)}_{i,j}(x) \gets R \cdot \left(\bigoplus_{a=-\frac{q-1}{2}}^{\frac{q-1}{2}} \bigoplus_{b=-\frac{q-1}{2}}^{\frac{q-1}{2}}  \eta_{i+a,j+b}^{(h)}(x)\right).
			\end{split}
		\end{equation}
		\begin{equation}\label{psi-cntk-last}
			\text{(for $h=L$:) }~~~~~~~~~~~~~~~~\psi^{(L)}_{i,j}(x) \gets Q^2 \cdot \left(\psi^{(L-1)}_{i,j}(x) \otimes \dot{\phi}_{i,j}^{(L)}(x)\right).~~~~~~~~~~~~~
		\end{equation}
		
		\item Let $G \in \RR^{s^* \times m_2}$ be a random matrix of i.i.d. normal entries with distribution $\mathcal{N}(0,1/s^*)$. The \emph{CNTK Sketch} is the following:
		\begin{equation}\label{Psi-cntk-def}
			\Psi_{cntk}^{(L)}(y,z) := \frac{1}{d_1d_2} \cdot G \cdot \left(\sum_{i \in [d_1]} \sum_{j \in [d_2]} \psi^{(L)}_{i,j}(x)\right).
		\end{equation}
	\end{enumerate}
\end{definition}

Now we present our main theorem on the performance guarantee of our \emph{CNTK Sketch}:

\begin{theorem}\label{maintheorem-cntk}
	For every positive integers $d_1,d_2,c$ and $L \ge 2$, and every $\epsilon, \delta>0$, if we let $\Theta_{cntk}^{(L)}:\RR^{d_1\times d_2 \times c} \times \RR^{d_1\times d_2\times c} \to \RR$ be the $L$-layered CNTK with ReLU activation and GAP given in \cite{arora2019exact},
	then there exists a randomized map $\Psi_{cntk}^{(L)}: \RR^{d_1 \times d_2 \times c} \to \RR^{s^*}$ for some $s^* = O\left( \frac{1}{\epsilon^2} \cdot \log \frac{1}{\delta} \right)$ such that:
	\begin{enumerate}[wide, labelwidth=!, labelindent=0pt]
		\item For any images $y,z \in \RR^{d_1 \times d_2 \times c}$: 
		\[ \Pr \left[ \left< \Psi_{cntk}^{(L)}(y) , \Psi_{cntk}^{(L)}(z) \right> \notin (1\pm \epsilon) \cdot \Theta_{cntk}^{(L)}(y,z) \right] \le \delta.
		\]

		\item For every image $x \in \RR^{d_1 \times d_2 \times c}$, $\Psi_{cntk}^{(L)}(x)$ can be computed in time $O\left( \frac{L^{11}}{\epsilon^{6.7}} \cdot (d_1d_2) \cdot \log^3 \frac{d_1d_2L}{\epsilon\delta} \right)$.
	\end{enumerate}
\end{theorem}
This theorem is proved in Appendix~\ref{app-cntk-sketch}.
Runtime of our \CNTKS is only linear in the number of image pixels $d_1d_2$, which is in stark contrast to quadratic scaling of the exact CNTK computation \cite{arora2019exact}.

\section{Experiments}
In this section, we empirically show that running least squares regression on the features generated by our methods is an extremely fast and effective way of approximately learning with NTK and CNTK kernel machines and can compete with various baseline methods. One of the baselines that we consider is an MLP with ReLU activation that is fully-trained by stochastic gradient descent. Another baselin is the Nystrom kernel approximation method of \cite{musco2017recursive} applied to the NTK kernel function with ReLU activation. Finally, we also compare our CNTK Sketch agains the gradient features of an ultra-wide random CNN~\cite{novak2018bayesian} as the third baseline. 

\paragraph{Reproducibility.} All the codes used to produce our experimental results are publicly available at this link:
\url{https://github.com/amirzandieh/NTK-CNTK-Sketch.git}.

\subsection{NTK Sketch: Classification and Regression}
We first show that our proposed NTK Sketch algorithm achieves better runtime and accuracy trade-off compared to the Nystrom method of \cite{musco2017recursive} as well as fully-trained MLPs.
We benchmark all algorithms on a variety of large-scale classification and regression datasets. 
To apply the algorithms to classification problem, we encode class labels into one-hot vectors with zero-mean and solve a ridge regression problem. We accelerate our NTK Sketch for deep networks using the following trick,

\paragraph{Optimizing NTK Sketch for Deeper Nets}
As shown in Theorem\ref{thm:ntk-relu}, the NTK is a normalized dot-product kernel characterized by the function $K^{(L)}_{relu} (\alpha)$. This function can be easily computed using $O(L)$ operations at any desired $\alpha \in [-1,1]$, therefore, we can efficiently fit a polynomial to this function using numerical methods (for instance, it is shown in Fig.~\ref{fig:relu-ntk} that a degree-$8$ polynomial can tightly approximate the depth-$3$ ReLU-NTK function $K_{relu}^{(3)}$). Then, we can efficiently sketch the resulting polynomial kernel using \PolyS, as was previously done for Gaussian and general dot-product kernels \cite{ahle2020oblivious,woodruff2020near}. Therefore, we can accelerate our NTK Sketch for deeper networks ($L>2$), using this heuristic.

\begin{figure}[!t]
	\centering
	\begin{subfigure}{\textwidth}
		\captionsetup{justification=centering,margin=0.5cm}
		\includegraphics[width=0.32\textwidth]{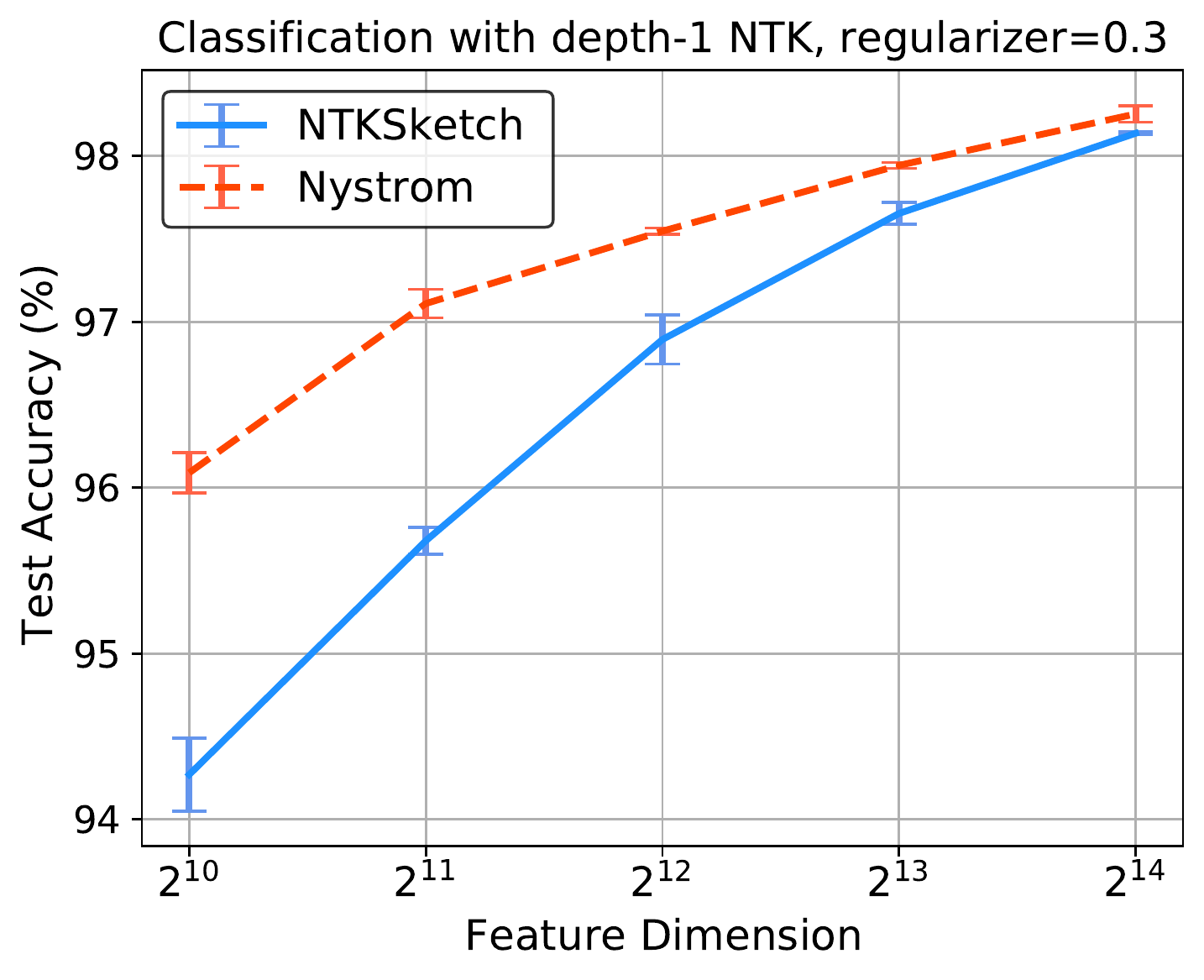}
		\includegraphics[width=0.32\textwidth]{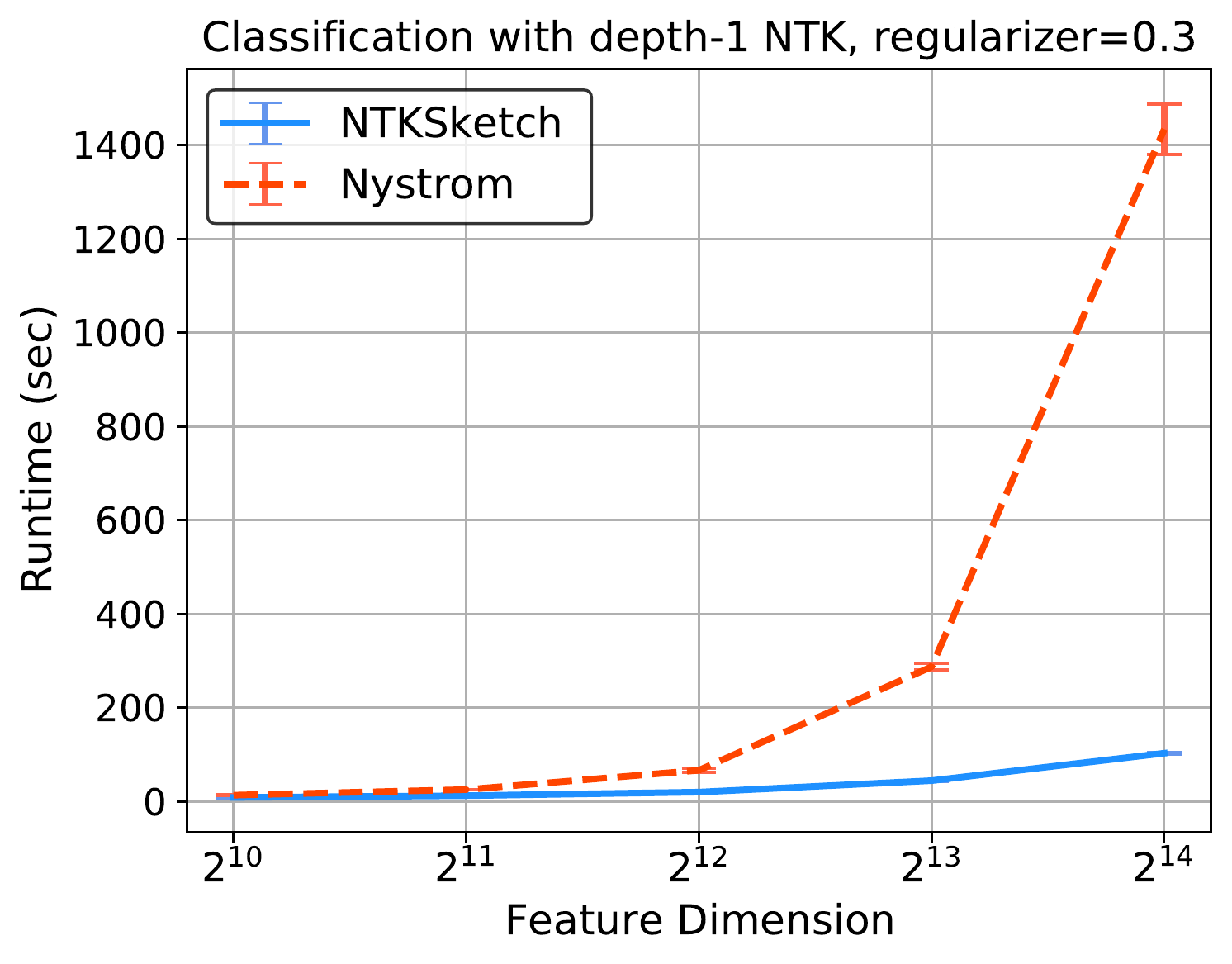}
		\includegraphics[width=0.32\textwidth]{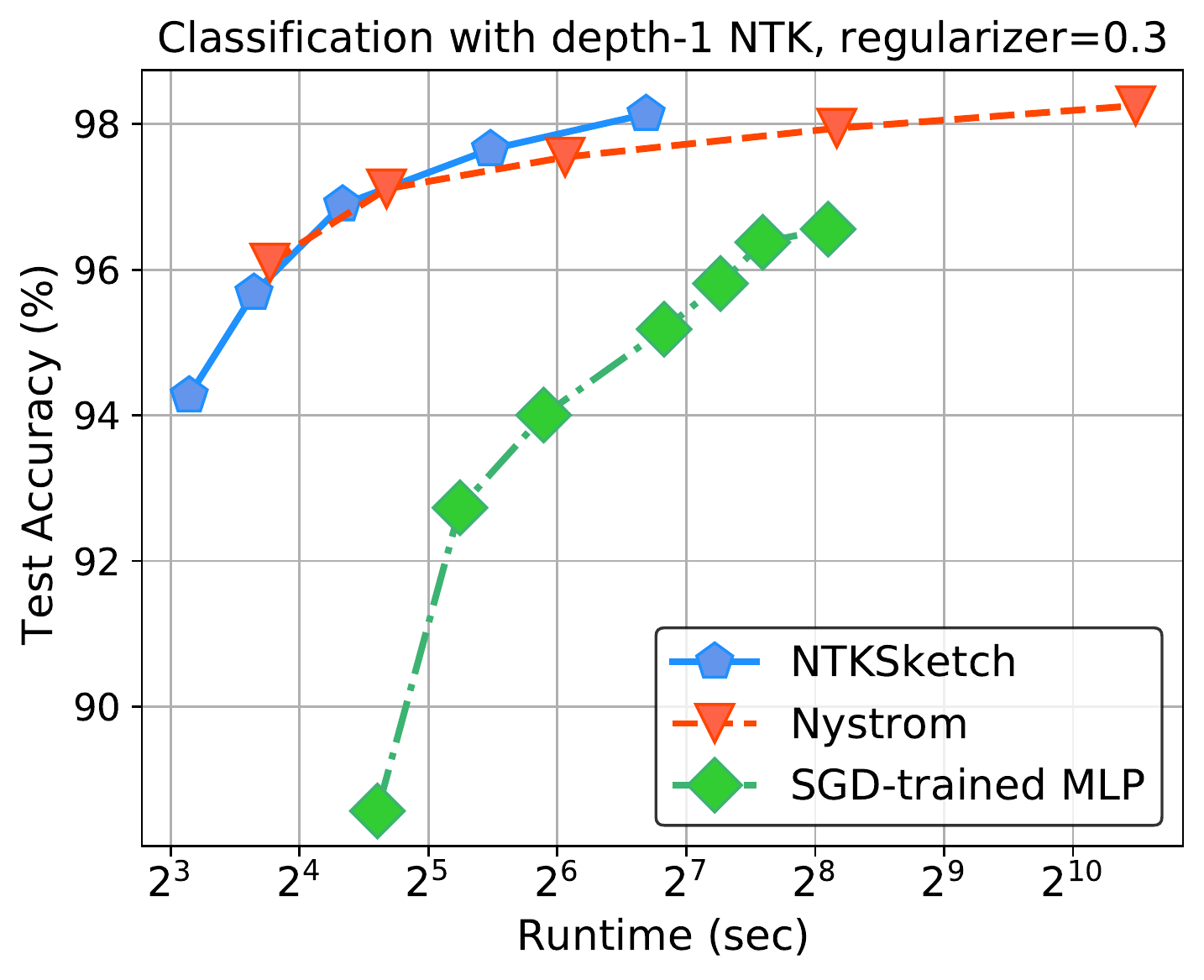}
		\caption{Classification on MNIST dataset using apprximate NTKs of depth $1$ as well as MLP with $1$ hidden layer. The regularizer is $\lambda = 0.3$.}\label{fig:mnist-ntk}
	\end{subfigure}
	\begin{subfigure}{\textwidth}
		\captionsetup{justification=centering,margin=0.5cm}
		\includegraphics[width=0.32\textwidth]{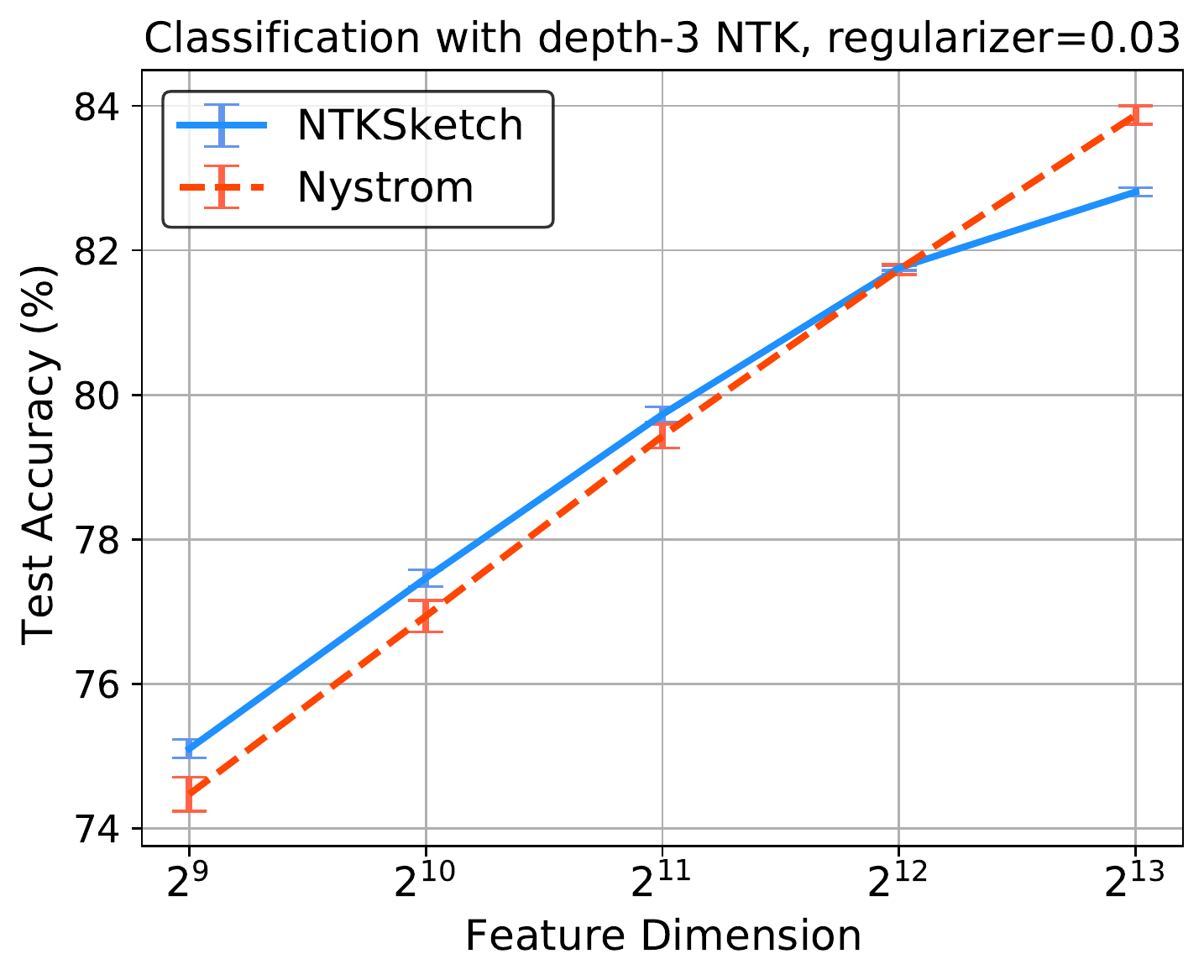}
		\includegraphics[width=0.32\textwidth]{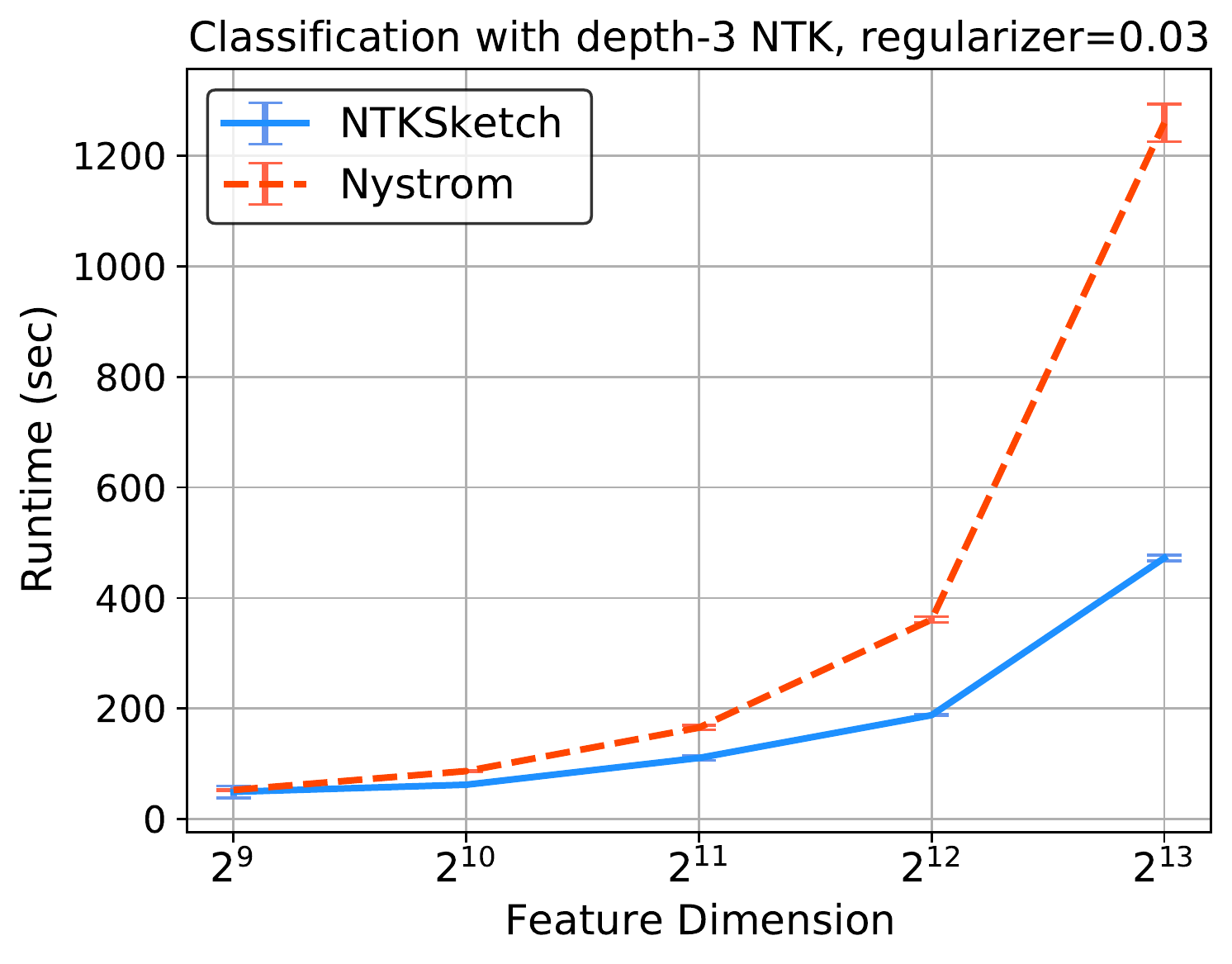}
		\includegraphics[width=0.32\textwidth]{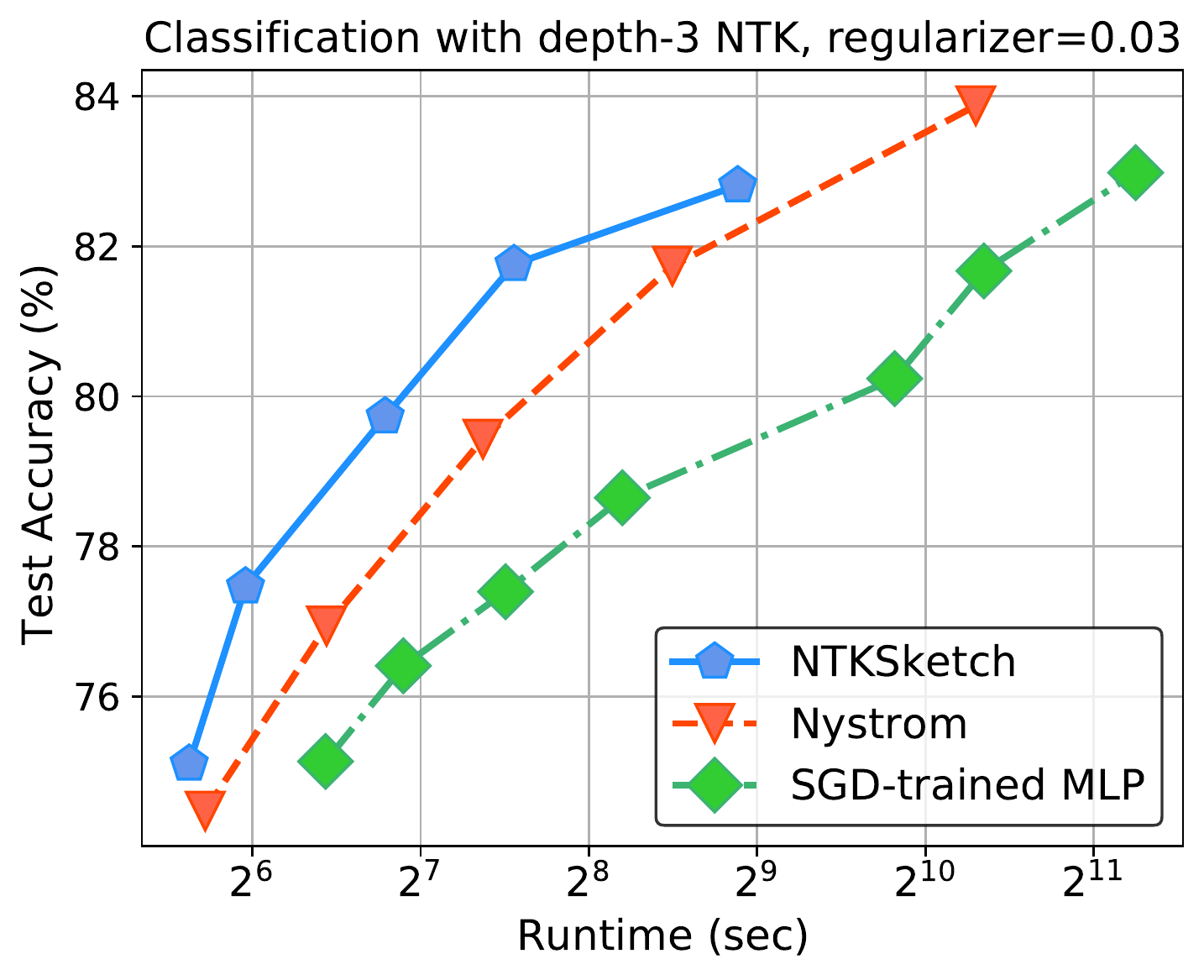}
		\caption{Classification on Forest CoverType dataset using apprximate NTKs of depth $3$ as well as MLP with $3$ hidden layer. The regularizer is $\lambda = 0.03$.}\label{fig:forest-ntk}
	\end{subfigure}
	\begin{subfigure}{\textwidth}
	\captionsetup{justification=centering,margin=0.5cm}
	\includegraphics[width=0.32\textwidth]{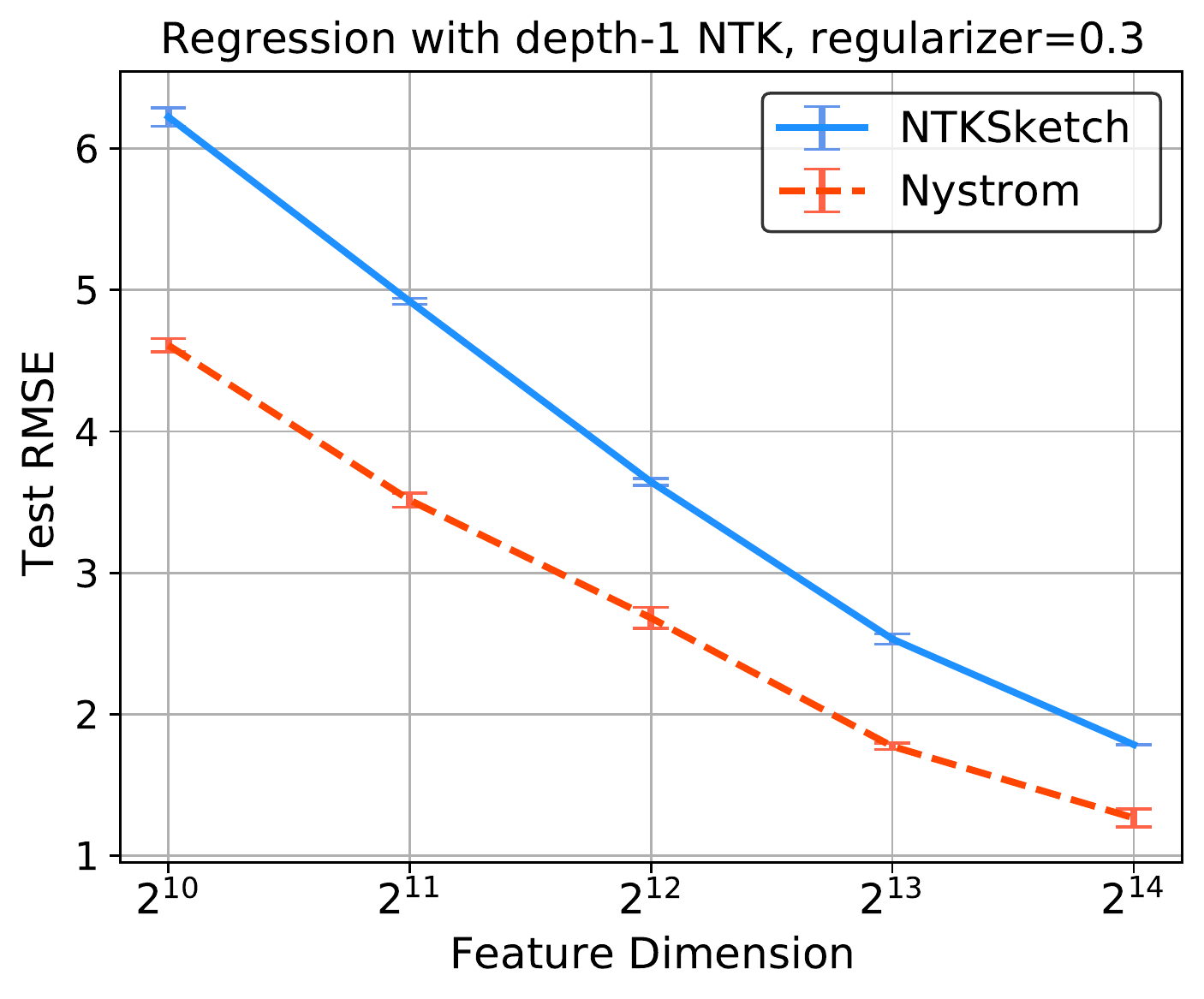}
	\includegraphics[width=0.32\textwidth]{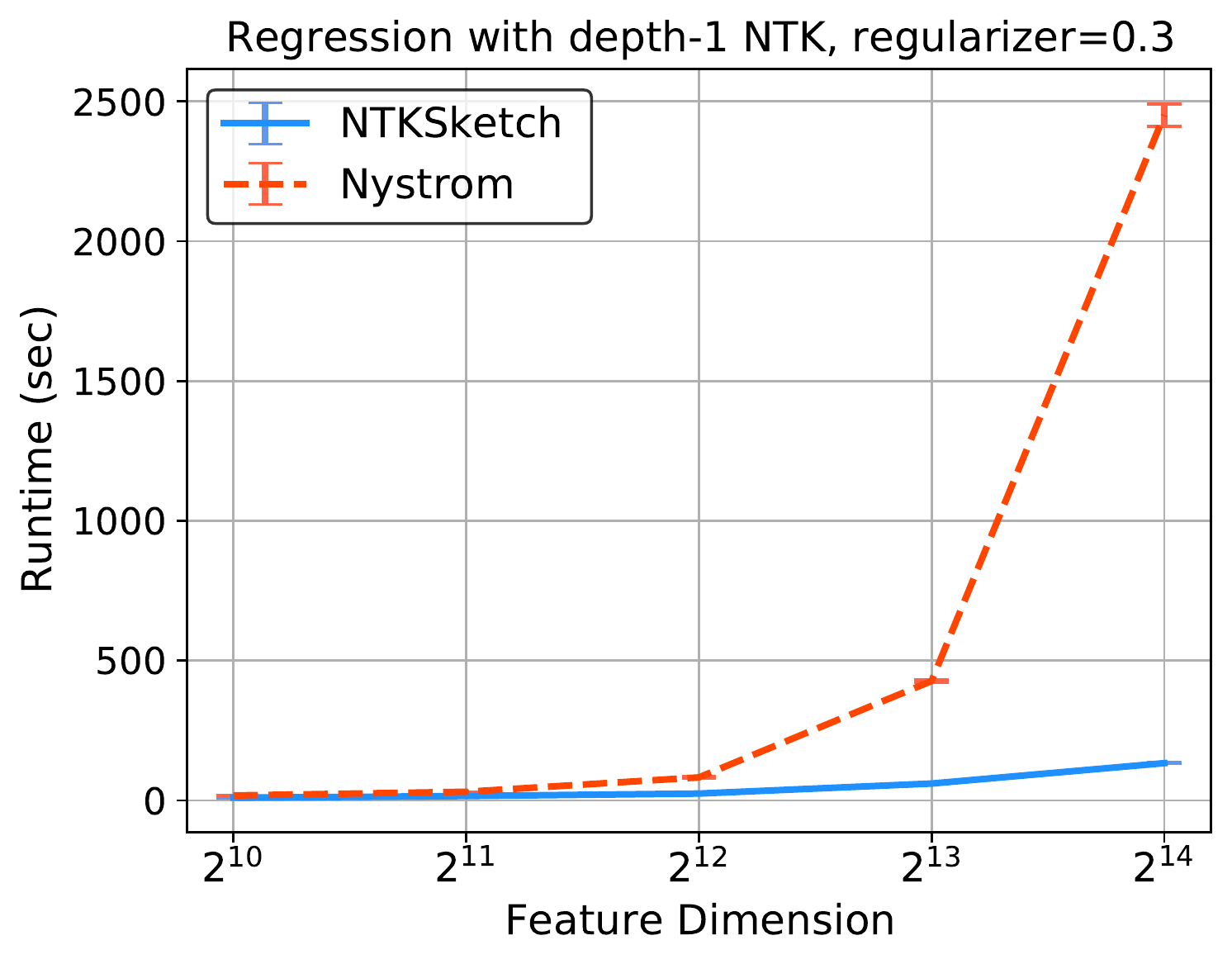}
	\includegraphics[width=0.32\textwidth]{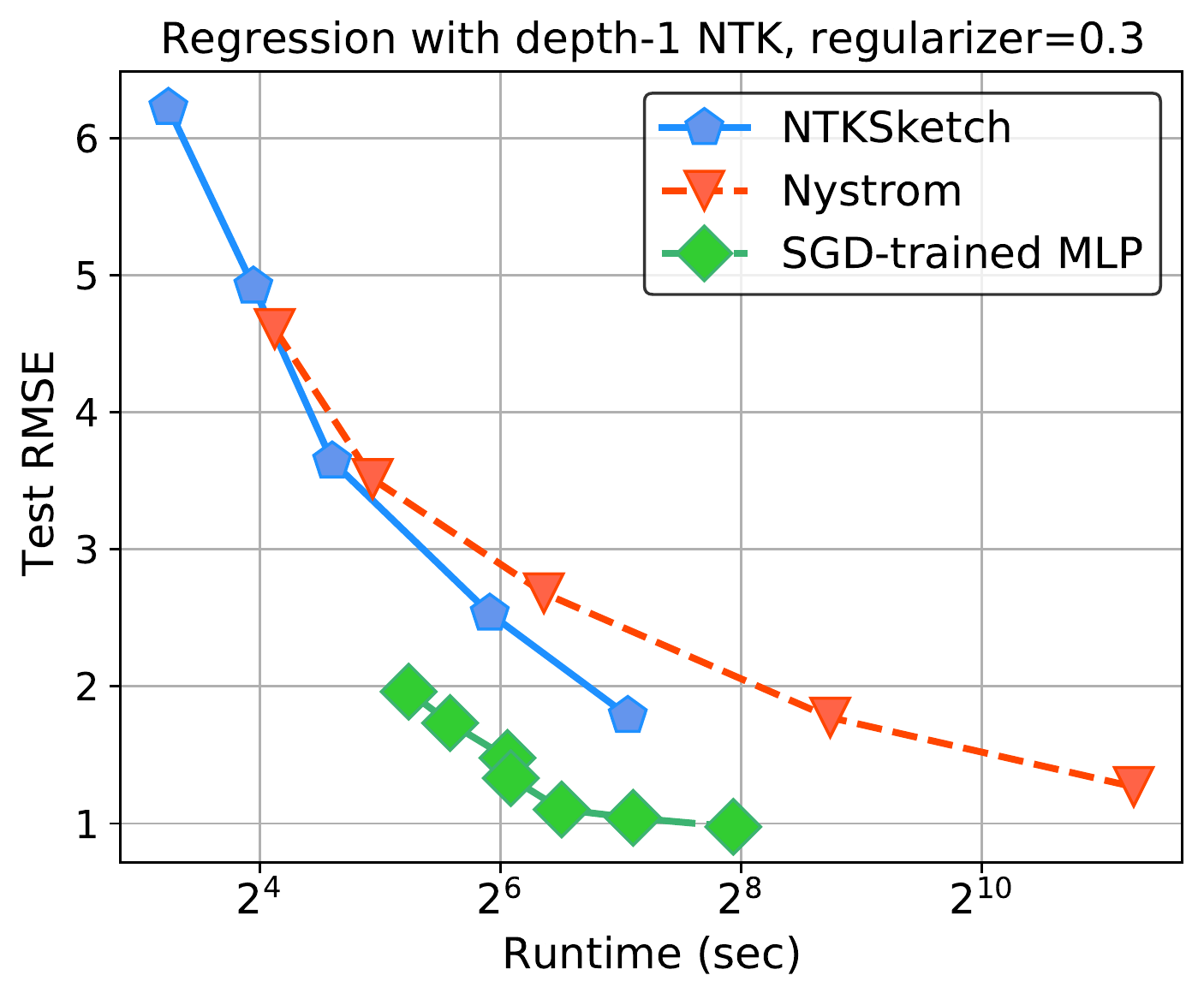}
	\caption{Regression on CT Location dataset using apprximate NTKs of depth $1$ as well as MLP with $1$ hidden layer. The regularizer is $\lambda = 0.3$.}\label{fig:ct-ntk}
	\end{subfigure}
	\captionsetup{justification=centering,margin=0.1cm}
	\caption{The runtime, test accuracy, and test RMSE of various classification and regression methods. Comparison between our NTK Sketch, the Nystrom method of \cite{musco2017recursive} applied to the NTK kernel, and SGD-trained MLP with ReLU activation.} \label{fig:ntksketch}
	\vspace{-7pt}
\end{figure}

In Figure~\ref{fig:ntksketch}, we compare the runtime and accuracy of all algorithms on various datasets and using various network depths $L$. We also plot the trade-off between the accuracy and runtime of different methods. The accuracy vs runtime graph for the MLP method is obtained by varying the width of the hidden layers and training the model for different widths.

Fig.~\ref{fig:mnist-ntk} shows the test set accuracy and the total training and estimation time for classifying MNIST dataset using networks of depth $L=1$ ($1$-hidden layer MLP or depth-$1$ NTK $\Theta_{ntk}^{(1)}$). One can see that while the Nystrom method achieves better accuracy, its runtime is extremely slower than our sketching method and in fact, our method achieves the best trade-off between runtime and accuracy, even better than fully trained MLP.
 
 Fig.~\ref{fig:forest-ntk} shows the results for classifying Forest CoverType dataset using networks of depth $L=3$ ($3$-hidden layer MLP or depth-$3$ NTK $\Theta_{ntk}^{(3)}$). In terms of accuracy, the Nystrom method and our method are quite similar, but the runtime of our sketching method is extremely faster and thus, our method achieves the best trade-off between runtime and accuracy with a large margin.

 Fig.~\ref{fig:forest-ntk} shows the RMSE of the regression on CT Location dataset using networks of depth $L=1$. While the Nystrom method achieves better RMSE, its runtime is extremely slower than our sketching method and in fact, our method achieves better trade-off between runtime and accuracy than the Nystrom method. However, in this experiment, the trained MLP has the best performance.

\paragraph{CNTK Sketch: Classification of CIFAR-10}
We base this set of experiments on the CIFAR-10 dataset. 
We compare the performance of our CNTK Sketch against the Monte Carlo method of \cite{novak2018bayesian} which generates features by taking the gradient of a randomly initialized CNN with finite width. We choose a convolutional network of depth $L=3$ with GAP (CNN with 3 convolutional layers or depth-$3$ CNTK $\Theta_{cntk}^{(3)}$) and compare the runtime and test accuracy of both methods for various feature dimensions.

 \begin{figure}[!t]
 	\centering
 		\captionsetup{justification=centering,margin=0.5cm}
 		\includegraphics[width=0.32\textwidth]{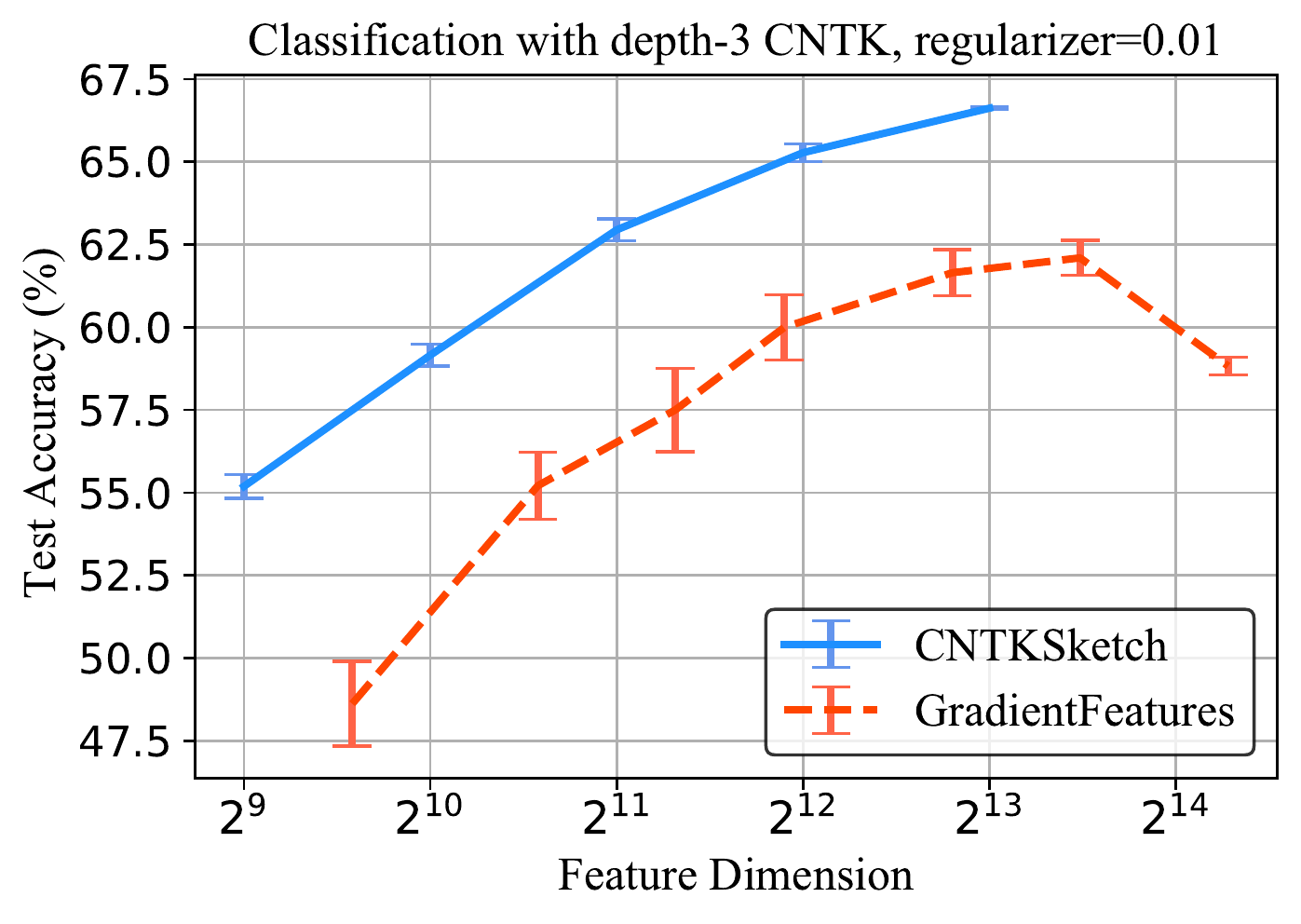}
 		\includegraphics[width=0.32\textwidth]{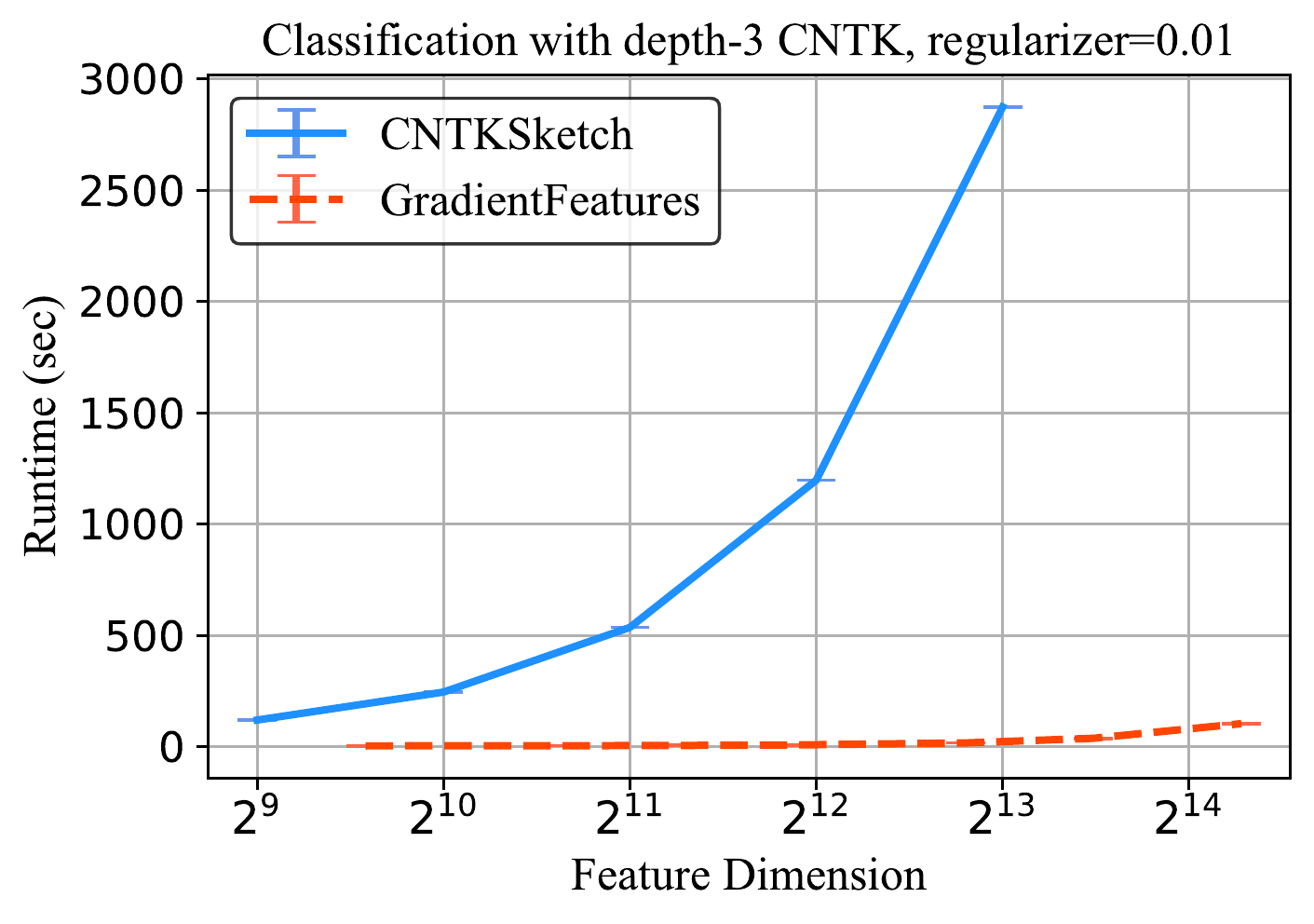}
 		\includegraphics[width=0.32\textwidth]{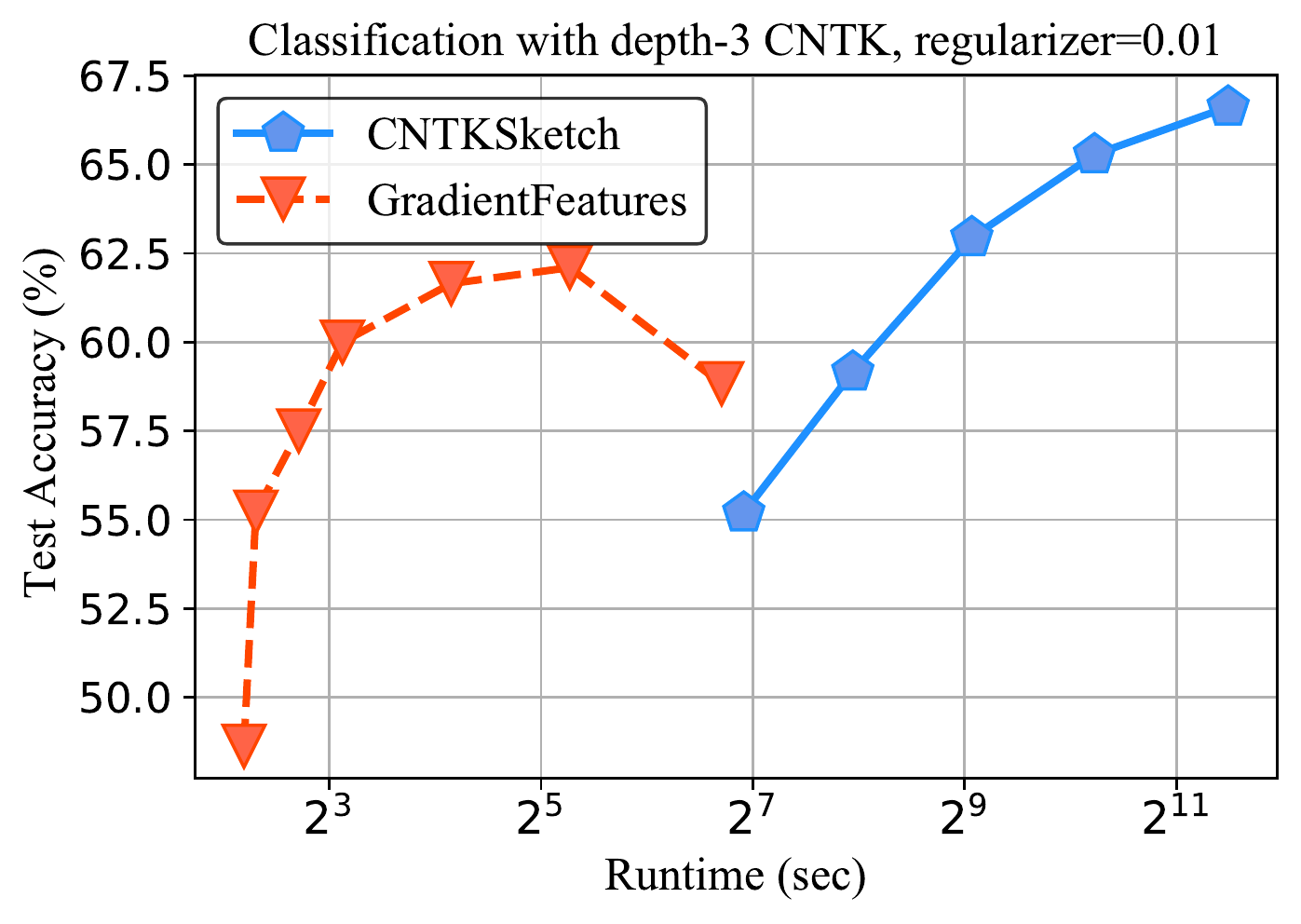}
 	\captionsetup{justification=centering,margin=0.1cm}
 	\caption{The runtime and test accuracy of classifying the CIFAR-10 dataset using CNTK Sketch and Gradient Features for depth $3$ convolutional networks. The regularizer is $\lambda = 0.01$} \label{fig:cntk-cifar}
 	\vspace{-7pt}
 \end{figure}

The results are provided in Figure~\ref{fig:cntk-cifar}. Our CNTK Sketch achieves the best test accuracy for any fixed features dimension while the Gradient Features is much faster. An interesting but expected observation is that the accuracy of the Gradient Features has a high variance which resulted in large error bars in Fig.~\ref{fig:cntk-cifar}. This is caused by approximating an infinitely wide CNN by a finite one. 
Another interesting observation is that increasing the number of features of the Gradient Features does not always improve the test accuracy and in fact the test accuracy saturates at some point and even declines.
As a result, the plot on the right of Fig.~\ref{fig:cntk-cifar} shows that Gradient features performs better when the runtime budget is small but its accuracy peaks at about $62\%$ while our method can achieves higher accuracies by consuming more computational resources.


\bibliography{example_paper}
\bibliographystyle{alpha}

\clearpage

\onecolumn
\appendix
\section{Sketching Background: \PolyS}\label{appendix-sketch-prelims}

Our sketching algorithms use the \PolyS to reduce the dimensionality of the tensor products that arise in our computations. Here we present the proof of Lemma~\ref{soda-result},

\begin{proofof}{Lemma~\ref{soda-result}}
	By invoking Theorem 1.2 of \cite{ahle2020oblivious}, we find that there exists a random sketch $Q^p \in\RR^{m \times d^p}$ such that $m = C \cdot \frac{p}{\varepsilon^2} \log^3 \frac{1}{\varepsilon\delta}$, for some absolute constant $C$, and for any $y\in \RR^{d^p}$, 
	\begin{align*}
		\Pr\left[ \|Q^p y \|_2^2 \in (1\pm\varepsilon) \|y\|_2^2\right] \ge 1- \delta.
	\end{align*}
	This immediately proves the first statement of the lemma.
	
	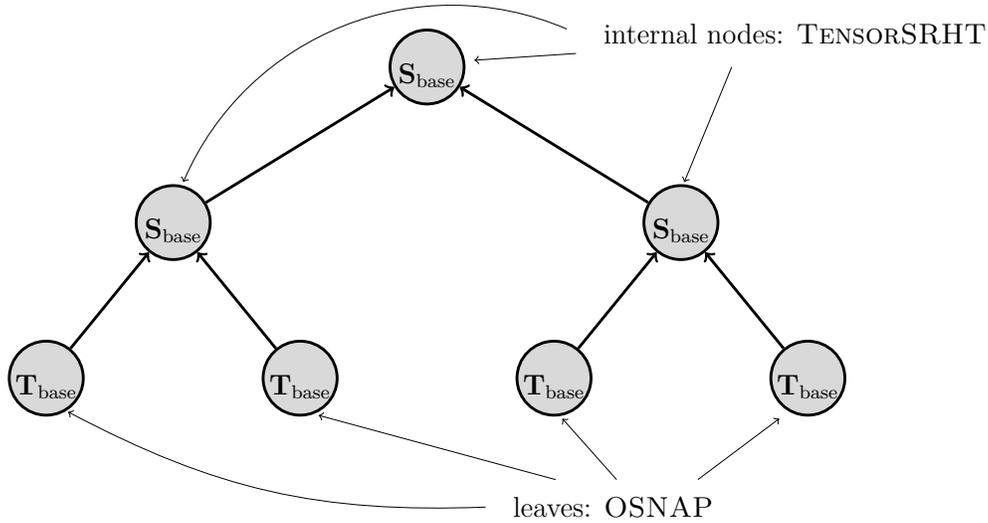
\begin{figure*}
		\centering
		
		\scalebox{0.9}{
			\begin{tikzpicture}[<-, level/.style={sibling distance=75mm/#1,level distance = 2.3cm}]
				\node [arn_t] (z){}
				child {node [arn_t] (a){}edge from parent [electron]
					child {node [arn_t] (b){}edge from parent [electron]
					}
					child {node [arn_t] (e){}edge from parent [electron]
					}
				}
				child { node [arn_t] (h){}edge from parent [electron]
					child {node [arn_t] (i){}edge from parent [electron]
					}
					child {node [arn_t] (l){}edge from parent [electron]
					}
				};
				
				\node []	at (z.south)	[label=\large{${\bf S_{\text{base}}}$}]	{};

				\node []	at (a.south)	[label=\large{${\bf S_{\text{base}}}$}]	{};

				\node []	at (b.south)	[label=\large${\bf T_{\text{base}}}$]	{};
				\node []	at (e.south)	[label=\large${\bf T_{\text{base}}}$] {};
				\node []	at (h.south)	[label=\large{${\bf S_{\text{base}}}$}] {};

				\node []	at (i.south)	[label=\large${\bf T_{\text{base}}}$] {};
				\node []	at (l.south)	[label=\large${\bf T_{\text{base}}}$] {};

				\draw[draw=black, ->] (2.2,0.2) -- (0.7,0.1);
				\draw[draw=black, ->] (4.5,0) -- (3.8,-1.7);
				
				\node [] at (2.2,0.5) [label=right:\large{ internal nodes: {\sc TensorSRHT}}]	{}
				edge[->, bend right=45] (-3.6,-1.7);
				
				\draw[draw=black, ->] (4.0,-6.1) -- (5.2,-5.2);
				\draw[draw=black, ->] (2.8,-6.1) -- (2,-5.2);
				\draw[draw=black, ->] (1.9,-6.1) -- (-1.6,-5.15);
				
				\node [] at (1,-6.5) [label=right:\large{leaves: {\sc OSNAP}}]	{}
				edge[->, bend left=15] (-5.3,-5.1);
				
			\end{tikzpicture}
		}
		\par

		\caption{The structure of sketch $Q^p$ proposed in Theorem 1.2 of \cite{ahle2020oblivious}: the sketch matrices in nodes of the tree labeled with $S_{\text{base}}$ and $T_{\text{base}}$ are independent instances of {\sc TensorSRHT} and {\sc OSNAP}, respectively.} \label{sketchingtree}
	\end{figure*}
	
	As shown in \cite{ahle2020oblivious}, the sketch $Q^p$ can be applied to tensor product vectors of the form $v_1 \otimes v_2 \otimes \ldots v_p$ by recursive application of $O(p)$ independent instances of OSNAP transform~\cite{nelson2013osnap} and a novel variant of the \SRHT, proposed in \cite{ahle2020oblivious} called {\sc TensorSRHT}, on vectors $v_i$ and their sketched versions. The sketch $Q^p$, as shown in Figure~\ref{sketchingtree}, can be represented by a binary tree with $p$ leaves where the leaves are {\sc OSNAP} sketches and the internal nodes are the {\sc TensorSRHT}. The use of {\sc OSNAP} in the leaves of this sketch structure ensures excellent runtime for sketching sparse input vectors. However, note that if the input vectors are not sparse, i.e., ${\rm nnz}(v_i) = \widetilde{\Omega}(d)$ for input vectors $v_i$, then we can simply remove the {\sc OSNAP} transforms from the leaves of this structure and achieve improved runtime, without hurting the approximation guarantee. Therefore, the sketch $Q^p$ that satisfies the statement of the lemma is exactly the one introduced in \cite{ahle2020oblivious} for sparse input vectors and for non-sparse inputs is obtained by removing the {\sc OSNAP} transforms from the leaves of the sketch structure given in Figure~\ref{sketchingtree}.
	
	{\bf Runtime analysis:} By Theorem 1.2 of \cite{ahle2020oblivious}, for any vector $x \in \RR^d$, $Q^p x^{\otimes p}$ can be computed in time $O\left( p m \log m + \frac{p^{3/2}}{\varepsilon} \log \frac{1}{\delta} \cdot {\rm nnz}(x) \right)$. 
	From the binary tree structure of the sketch, shown in Figure~\ref{sketchingtree}, it follows that once we compute $Q^p x^{\otimes p}$, then $Q^p \left(x^{\otimes p-1} \otimes {e}_1\right)$ can be computed by updating the path from one of the leaves to the root of the binary tree which amounts to applying an instance of {\sc OSNAP} transform on the ${e}_1$ vector and then applying $O(\log p)$ instances of \TSRHT on the intermediate nodes of the tree. This can be computed in a total additional runtime of $O( m\log m \log p )$. By this argument, it follows that $Q^p \left( x^{\otimes p-j} \otimes {e}_1^{j} \right)$ can be computed sequentially for all $j=0,1,2, \cdots p$ in total time $O\left( p m \log p \log m + \frac{p^{3/2}}{\varepsilon} \log\frac{1}{\delta} \cdot {\rm nnz}(x) \right)$. By plugging in the value $m = O\left( \frac{p}{\varepsilon^2} \log^3 \frac{1}{\varepsilon\delta} \right)$, this runtime will be $O\left( \frac{p^2 \log^2 \frac{p}{\varepsilon}}{\varepsilon^2} \log^3 \frac{1}{\varepsilon\delta} + \frac{p^{3/2}}{\varepsilon} \log\frac{1}{\delta} \cdot {\rm nnz}(x) \right)$, which gives the second statement of the lemma for sparse input vectors $x$. If $x$ is non-sparse, as we discussed in the above paragraph, we just need to omit the {\sc OSNAP} transforms from our sketch construction which translates into a runtime of $O\left( \frac{p^2 \log^2 \frac{p}{\varepsilon}}{\varepsilon^2} \log^3 \frac{1}{\varepsilon\delta} + pd \log d \right)$. Therefore, the final runtime bound is $O\left( \frac{p^2 \log^2 \frac{p}{\varepsilon}}{\varepsilon^2} \log^3 \frac{1}{\varepsilon\delta} + \min \left\{\frac{p^{3/2}}{\varepsilon} \log\frac{1}{\delta} \cdot {\rm nnz}(x), pd\log d \right\} \right)$, which proves the second statement of the lemma.
	
	Furthermore, the sketch $Q^p$ can be applied to tensor product of any collection of $p$ vectors. The time to apply $Q^p$ to the tensor product $v_1\otimes v_2\otimes \ldots  v_p$ consists of time of applying OSNAP to each of the vectors $v_1, v_2 \ldots v_p$ and time of applying $O(p)$ instances of \TSRHT to intermediate vectors which are of size $m$. This runtime can be upper bounded by $O\left( \frac{p^2 \log \frac{p}{\varepsilon}}{\varepsilon^2} \log^3 \frac{1}{\varepsilon\delta} + \frac{p^{3/2}}{\varepsilon} d \cdot \log\frac{1}{\delta} \right)$, which proves the third statement of Lemma~\ref{soda-result}.
\end{proofof}

\section{ReLU-NTK Expression}\label{appendix-relu-ntk}
In this section we prove Theorem~\ref{thm:ntk-relu}.
First note that the main component of the DP given in \eqref{eq:dp-covar}, \eqref{eq:dp-derivative-covar}, and \eqref{eq:dp-ntk} for computing the NTK of fully connected network is what we call the \emph{Activation Covariance}. The activation covariance is defined with respect to the activation $\sigma(\cdot)$ as follows,
\begin{definition}[Activation Covariance]\label{def:activation-covariance}
	For any continuous activation function $\sigma:\RR \to \RR$, we define the \emph{Activation Covariance} function $k_\sigma:\RR^d \times \RR^d \to \RR$ as follows:
	\[ k_\sigma(y,z) := \EE_{w \sim \mathcal{N}(0, I_d)} \left[ \sigma(w^\top y) \cdot \sigma(w^\top z) \right] \text{~~ for every }y,z \in \RR^d,\]
	Additionally, we define the \emph{Derivative Activation Covariance} function $\dot{k}_\sigma:\RR^d \times \RR^d \to \RR$ as follows:
	\[ \dot{k}_\sigma(y,z) := \EE_{w \sim \mathcal{N}(0, I_d)} \left[ \dot{\sigma}(w^\top y) \cdot \dot{\sigma}(w^\top z) \right] \text{~~ for every }y,z \in \RR^d. \]
\end{definition}
It is worth noting that one can prove the activation covariance function $k_\sigma$ as well as the derivative activation covariance function $\dot{k}_\sigma$ are positive definite and hence they both define valid kernel functions in $\RR^d \times \RR^d$.
The connection between the ReLU activation covariance functions and arc-cosine kernel functions defined in \eqref{relu-activ-cov} of Definition~\ref{def:relu-ntk} is proved in \cite{cho2009kernel}. We restate this result in the following claim,
\begin{claim}[Covariance of ReLU Activation]\label{relu-covariance}
	For ReLU activation function $\sigma(\alpha) = \max(\alpha,0)$, let $k_\sigma$ and $\dot{k}_\sigma$ denote the activation covariance function and the derivative activation covariance function defined in Definition~\ref{def:activation-covariance}. If we let $\kappa_0(\cdot)$ and $\kappa_1(\cdot)$ be the zeroth and first order arc-cosine kernels defined in \eqref{relu-activ-cov} of Definition~\ref{def:relu-ntk}, then for every positive integer $d$ and every $y,z \in \RR^d$ we have:
	\[ k_\sigma(y,z) = \frac{\|y\|_2\|z\|_2}{2} \cdot \kappa_1 \left(\frac{\langle y, z \rangle}{\|y\|_2\|z\|_2}\right)\]
	and
	\[\dot{k}_\sigma(y,z) = \frac{1}{2} \cdot \kappa_0\left(\frac{\langle y, z \rangle}{\|y\|_2\|z\|_2}\right).\]
\end{claim}

Now we are ready to prove theorem~\ref{thm:ntk-relu}:

\begin{proofof}{Theorem~\ref{thm:ntk-relu}}
	Consider the NTK expression given in \eqref{eq:dp-covar}, \eqref{eq:dp-derivative-covar}, and \eqref{eq:dp-ntk}.
	We first prove by induction on $h =0,1,2, \ldots L$ that the covariance function $\Sigma^{(h)}(y,z)$ defined in \eqref{eq:dp-covar} satisfies: 
	\[\Sigma^{(h)}(y,z) = \|y\|_2 \|z\|_2 \cdot \Sigma_{relu}^{(h)} \left( \frac{\langle y,z \rangle}{\|y\|_2 \|z\|_2} \right).\]
	The {\bf base of induction} is trivial for $h=0$ because $\Sigma^{(0)}(y,z) = \langle y,z \rangle = \|y\|_2 \|z\|_2 \cdot \Sigma_{relu}^{(0)} \left( \frac{\langle y,z \rangle}{\|y\|_2 \|z\|_2} \right)$.
	
	To prove the {\bf inductive step}, suppose that the inductive hypothesis holds for $h-1$, i.e.,
	\[\Sigma^{(h-1)}(y,z) = \|y\|_2 \|z\|_2 \cdot \Sigma_{relu}^{(h-1)} \left( \frac{\langle y,z \rangle}{\|y\|_2 \|z\|_2} \right)\]
	The $2 \times 2$ covariance matrix $\Lambda^{(h)}(y,z)$, defined in \eqref{eq:dp-covar}, can be decomposed as $\Lambda^{(h)}(y,z) = \begin{pmatrix} f^\top \\ g^\top \end{pmatrix} \cdot \begin{pmatrix}
		f & g \end{pmatrix}$, where $f,g \in \RR^2$. Now note that $\|f\|_2^2 = \Sigma^{(h-1)}(y,y)$, hence, by inductive hypothesis, we have,
	\[ \|f\|_2^2 = \|y\|_2^2 \cdot \Sigma_{relu}^{(h-1)} \left( \frac{\langle y,y \rangle}{\|y\|_2^2} \right) = \|y\|_2^2 \cdot \Sigma_{relu}^{(h-1)}(1) = \|y\|_2^2. \]
	Using a similar argument we have $\|g\|_2^2 = \|z\|_2^2$. Therefore, by Claim~\ref{relu-covariance}, we can write
	\begin{align*}
		\Sigma^{(h)}(y,z) &= \frac{1}{\EE_{x\sim \mathcal{N}(0,1)} \left[ |\sigma(x)|^2 \right]} \cdot \EE_{w \sim \mathcal{N}\left( 0, I_2 \right)} \left[ \sigma(w^\top f) \cdot \sigma(w^\top g) \right]\\
		&= \frac{2}{\kappa_1(1)} \cdot \EE_{w \sim \mathcal{N}\left( 0, I_2 \right)} \left[ \sigma(w^\top f) \cdot \sigma(w^\top g) \right]\\
		&= \|f\|_2 \|g\|_2 \cdot \kappa_1\left( \frac{\langle f, g \rangle}{\|f\|_2 \|g\|_2} \right) = \|y\|_2 \|z\|_2 \cdot \kappa_1\left( \frac{\langle f, g \rangle}{\|y\|_2 \|z\|_2} \right).
	\end{align*}
	Since we assumed that $\Lambda^{(h)}(y,z) = \begin{pmatrix} f^\top \\ g^\top \end{pmatrix} \cdot \begin{pmatrix}
		f & g \end{pmatrix}$, we have $\langle f, g \rangle = \Sigma^{(h-1)}(y,z)$. By inductive hypothesis along with \eqref{eq:dp-covar-relu}, we find that
	\[ \Sigma^{(h)}(y,z) = \|y\|_2 \|z\|_2 \cdot \kappa_1\left( \Sigma_{relu}^{(h-1)} \left( \frac{\langle y,z \rangle}{\|y\|_2 \|z\|_2} \right) \right) = \|y\|_2 \|z\|_2 \cdot \Sigma_{relu}^{(h)}\left( \frac{\langle y, z\rangle}{\|y\|_2 \|z\|_2}  \right), \]
	which completes the induction and proves that for every $h \in \{0,1,2, \ldots L\}$,
	\begin{equation}\label{eq:covariance-dotproduct}
		\Sigma^{(h)}(y,z) = \|y\|_2 \|z\|_2 \cdot \Sigma_{relu}^{(h)} \left( \frac{\langle y,z \rangle}{\|y\|_2 \|z\|_2} \right).
	\end{equation}
	
	For obtaining the final NTK expression we also need to simplify the derivative covariance function defined in \eqref{eq:dp-derivative-covar}.
	Recall that we proved before, the covariance matrix $\Lambda^{(h)}(y,z)$ can be decomposed as $\Lambda^{(h)}(y,z) = \begin{pmatrix} f^\top \\ g^\top \end{pmatrix} \cdot \begin{pmatrix}
		f & g \end{pmatrix}$, where $f,g \in \RR^2$ with $\|f\|_2 = \|y\|_2$ and $\|g\|_2=  \|z\|_2$. Therefore, by Claim~\ref{relu-covariance}, we can write
	\begin{align*}
		\dot{\Sigma}^{(h)}( y, z ) &= \frac{1}{\EE_{x\sim \mathcal{N}(0,1)} \left[ |\sigma(x)|^2 \right]} \cdot \EE_{w \sim \mathcal{N}\left( 0, I_2 \right)} \left[ \dot{\sigma}(w^\top f) \cdot \dot{\sigma}(w^\top g) \right] \\
		&=\frac{2}{\kappa_1(1)} \cdot \EE_{w \sim \mathcal{N}\left( 0, I_2 \right)} \left[ \dot{\sigma}(w^\top f) \cdot \dot{\sigma}(w^\top g) \right]\\
		&= \kappa_0\left( \frac{\langle f, g \rangle}{\|y\|_2 \|z\|_2} \right).
	\end{align*}
	Since we assumed that $\Lambda^{(h)}(y,z) = \begin{pmatrix} f^\top \\ g^\top \end{pmatrix} \cdot \begin{pmatrix}
		f & g \end{pmatrix}$, $\langle f, g \rangle = \Sigma^{(h-1)}(y,z) = \|y\|_2 \|z\|_2 \cdot \Sigma_{relu}^{(h-1)} \left( \frac{\langle y,z \rangle}{\|y\|_2 \|z\|_2} \right)$. Therefore, by \eqref{eq:dp-derivative-covar-relu}, for every $h \in \{1,2, \dots L\}$,
	\begin{equation}\label{eq:der-cov-dotproduct}
		\dot{\Sigma}^{(h)}(y,z) = \kappa_0\left( \Sigma_{relu}^{(h-1)}\left( \frac{\langle y, z \rangle}{\|y\|_2 \|z\|_2} \right) \right) = \dot{\Sigma}_{relu}^{(h)}\left( \frac{\langle y, z \rangle}{\|y\|_2 \|z\|_2} \right).
	\end{equation}
	
	Now we prove by induction on integer $L$ that the NTK with $L$ layers and ReLU activation given in \eqref{eq:dp-ntk} satisfies 
	\[\Theta_{ntk}^{(L)}(y,z) = \|y\|_2 \|z\|_2 \cdot K_{relu}^{(L)}\left( \frac{\langle y,z\rangle}{\|y\|_2 \|z\|_2} \right).\]
	The {\bf base of induction}, trivially holds because, by \eqref{eq:covariance-dotproduct}:
	\[\Theta_{ntk}^{(0)}(y,z)=\Sigma^{(0)}(y,z)=\|y\|_2 \|z\|_2 \cdot K^{(0)}_{relu}\left( \frac{\langle y,z\rangle}{\|y\|_2 \|z\|_2} \right).\]
	
	To prove the {\bf inductive step}, suppose that the inducive hypothesis holds for $L-1$, that is $\Theta_{ntk}^{(L-1)}(y,z) = \|y\|_2 \|z\|_2 \cdot K_{relu}^{(L-1)}\left( \frac{\langle y,z\rangle}{\|y\|_2 \|z\|_2} \right)$. Now using the recursive definition of $\Theta_{ntk}^{(L)}(y,z)$ given in \eqref{eq:dp-ntk} along with \eqref{eq:covariance-dotproduct} and \eqref{eq:der-cov-dotproduct} we can write,
	\begin{align*}
		\Theta_{ntk}^{(L)}(y,z)
		&= \Theta_{ntk}^{(L-1)}(y,z) \cdot \dot{\Sigma}^{(L)}(y,z) + \Sigma^{(L)}(y,z)\\
		&= \|y\|_2 \|z\|_2 \cdot K_{relu}^{(L-1)}\left( \frac{\langle y,z\rangle}{\|y\|_2 \|z\|_2} \right) \cdot \dot{\Sigma}_{relu}^{(h)}\left( \frac{\langle y, z \rangle}{\|y\|_2 \|z\|_2} \right) +  \|y\|_2 \|z\|_2 \cdot \Sigma_{relu}^{(h)} \left( \frac{\langle y,z \rangle}{\|y\|_2 \|z\|_2} \right)\\
		&\equiv \|y\|_2 \|z\|_2 \cdot K_{relu}^{(L)}\left( \frac{\langle y,z\rangle}{\|y\|_2 \|z\|_2} \right),
	\end{align*}
	which gives the theorem.
\end{proofof}

\section{NTK Sketch: Claims and Invariants}\label{appendix-ntk-sketch}

We start by proving that the polynomials $P_{relu}^{(p)}(\cdot)$ and $\dot{P}_{relu}^{(p)}(\cdot)$ defined in \eqref{eq:poly-approx-krelu} of Definition~\ref{alg-def-ntk-sketch} closely approximate the arc-cosine functions $\kappa_1(\cdot)$ and $\kappa_0(\cdot)$ on the interval $[-1,1]$.

\paragraph{Remark.} Observe that $\kappa_0(\alpha) = \frac{d}{d\alpha}\left( \kappa_1(\alpha) \right)$.

\begin{lemma}[Polynomial Approximations to $\kappa_1$ and $\kappa_0$]\label{lem:polynomi-approx-krelu}
	If we let ${k}_{relu}(\cdot)$ and $\kappa_0(\cdot)$ be defined as in \eqref{relu-activ-cov} of Definition~\ref{def:relu-ntk}, then for any integer $p \ge \frac{1}{9 \epsilon^{2/3}}$, the polynomial $P_{relu}^{(p)}(\cdot)$ defined in \eqref{eq:poly-approx-krelu} of Definition~\ref{alg-def-ntk-sketch} satisfies,
	\[ \max_{\alpha \in [-1,1]} \left| P_{relu}^{(p)}(\alpha) - \kappa_1(\alpha) \right| \le \epsilon.\]
	Moreover, for any integer $p' \ge \frac{1}{26 \epsilon^2}$, the polynomial $\dot{P}_{relu}^{(p')}(\cdot)$ defined as in \eqref{eq:poly-approx-krelu} of Definition~\ref{alg-def-ntk-sketch}, satisfies,
	\[ \max_{\alpha \in [-1,1]} \left| \dot{P}_{relu}^{(p')}(\alpha) - \kappa_0(\alpha) \right| \le \epsilon.\]
\end{lemma}
\begin{proof}
	We start by Taylor series expansion of $\kappa_0(\cdot)$ at $\alpha=0$, $\kappa_0(\alpha) = \frac{1}{2} + \frac{1}{\pi} \cdot \sum_{n=0}^\infty \frac{(2n)!}{2^{2n} \cdot (n!)^2 \cdot (2n+1)} \cdot \alpha^{2n+1}$.
	Therefore, we have
	\begin{align*}
		\max_{\alpha \in [-1,1]} \left| \dot{P}_{relu}^{(p')}(\alpha) - \kappa_0(\alpha) \right| &= \frac{1}{\pi} \cdot \sum_{n=p'+1}^\infty \frac{(2n)!}{2^{2n} \cdot (n!)^2 \cdot (2n+1)}\\
		&\le \frac{1}{\pi} \cdot \sum_{n = p'+1}^\infty \frac{e \cdot e^{-2n} \cdot (2n)^{2n + 1/2}}{ 2\pi \cdot 2^{2n} \cdot e^{-2n} \cdot n^{2n + 1} \cdot (2n+1)}\\
		&= \frac{e}{\sqrt{2} \pi^2} \cdot \sum_{n = p'+1}^\infty \frac{1}{\sqrt{n} \cdot (2n+1)}\\
		&\le \frac{e}{\sqrt{2} \pi^2} \cdot \int_{p'}^\infty \frac{1}{\sqrt{x} \cdot (2x+1)} dx\\
		&\le \frac{e}{\sqrt{2} \pi^2} \cdot \frac{1}{\sqrt{p'}} \le \epsilon.
	\end{align*}
	
	To prove the second part of the lemma, we consider the Taylor expansion of $\kappa_1(\cdot)$ at $\alpha=0$. Observe that $\kappa_0(\alpha) = \frac{d}{d\alpha}\left(\kappa_1(\alpha)\right)$, therefore, the Taylor series of $\kappa_1(\alpha)$ can be obtained from the Taylor series of $\kappa_0(\alpha)$ as follows,
	\[ \kappa_1(\alpha) = \frac{1}{\pi} + \frac{\alpha}{2} + \frac{1}{\pi} \cdot \sum_{n=0}^\infty \frac{(2n)!}{2^{2n} \cdot (n!)^2 \cdot (2n+1)\cdot (2n+2)} \cdot \alpha^{2n+2}. \]
	Hence, we have
	\begin{align*}
		\max_{\alpha \in [-1,1]} \left| P_{relu}^{(p)}(\alpha) - \kappa_1(\alpha) \right| &= \frac{1}{\pi} \cdot \sum_{n=p+1}^\infty \frac{(2n)!}{2^{2n} \cdot (n!)^2 \cdot (2n+1) \cdot (2n+2)}\\
		&\le \frac{1}{\pi} \cdot \sum_{n = p+1}^\infty \frac{e \cdot e^{-2n} \cdot (2n)^{2n + 1/2}}{ 2\pi \cdot 2^{2n} \cdot e^{-2n} \cdot n^{2n + 1} \cdot (2n+1) \cdot (2n+2)}\\
		&= \frac{e}{\sqrt{2} \pi^2} \cdot \sum_{n = p+1}^\infty \frac{1}{\sqrt{n} \cdot (2n+1) \cdot (2n+2) }\\
		&\le \frac{e}{\sqrt{2} \pi^2} \cdot \int_{p}^\infty \frac{1}{\sqrt{x} \cdot (2x+1) \cdot (2x+2)} dx\\
		&\le \frac{e}{\sqrt{2} \pi^2} \cdot \frac{1}{6 \cdot p^{3/2}} \le \epsilon.
	\end{align*}
\end{proof}

Therefore, it is possible to approximate the function $\kappa_0(\cdot)$ up to error $\epsilon$ using a polynomial of degree $O\left(\frac{1}{ \epsilon^2}\right)$. Also if we want to approximate $\kappa_1(\cdot)$ using a polynomial up to error $\epsilon$ on the interval $[-1,1]$, it suffices to use a polynomial of degree $O\left(\frac{1}{\epsilon^{2/3} }\right)$.
One can see that since the Taylor expansion of $\kappa_1$ and $\kappa_0$ contain non-negative coefficients only, both of these functions are positive definite. Additionally, the Polynomial approximations $P_{relu}^{(p)}$ and $\dot{P}_{relu}^{(p')}$ given in \eqref{eq:poly-approx-krelu} of Definition~\ref{alg-def-ntk-sketch} are positive definite functions.

In order to prove Theorem~\ref{mainthm-ntk}, we also need the following lemma on the error sensitivity of polynomials $P_{relu}^{(p)}$ and $\dot{P}_{relu}^{(p')}$,
\begin{lemma}[Sensitivity of $P_{relu}^{(p)}$ and $\dot{P}_{relu}^{(p)}$]\label{lema:sensitivity-polynomial}
	For any integer $p\ge 3$, any $\alpha \in [-1,1]$, and any $\alpha'$ such that $|\alpha - \alpha'| \le \frac{1}{6p}$, if we let the polynomials $P_{relu}^{(p)}(\alpha)$ and $\dot{P}_{relu}^{(p)}(\alpha)$ be defined as in \eqref{eq:poly-approx-krelu} of Definition~\ref{alg-def-ntk-sketch}, then 
	\[ \left| P_{relu}^{(p)}(\alpha) - P_{relu}^{(p)}(\alpha') \right| \le |\alpha - \alpha'|, \]
	and
	\[ \left| \dot{P}_{relu}^{(p)}(\alpha) - \dot{P}_{relu}^{(p)}(\alpha') \right| \le \sqrt{p} \cdot |\alpha - \alpha'|. \]
\end{lemma}
\begin{proof}
	Note that an $\alpha' $ that satisfies the preconditions of the lemm, is in the range $\left[-1 - \frac{1}{6p}, 1 + \frac{1}{6p} \right]$. Now we bound the derivative of the polynomial $\dot{P}_{relu}^{(p)}$ on the interval $\left[-1 - \frac{1}{6p}, 1 + \frac{1}{6p} \right]$,
	\begin{align*}
		\max_{\alpha \in \left[-1 - \frac{1}{6p}, 1 + \frac{1}{6p} \right]} \left| \frac{d}{d\alpha} \left(\dot{P}_{relu}^{(p)}(\alpha)\right) \right| &= \frac{1}{\pi} \cdot \sum_{n=0}^p \frac{(2n)!}{2^{2n} \cdot (n!)^2} \cdot \left( 1 + \frac{1}{6p} \right)^{2n}\\
		&\le \frac{1}{\pi} + \frac{e^{4/3}}{\sqrt{2} \pi^2} \cdot \sum_{n=1}^p \frac{1}{\sqrt{n}}\\
		&\le \frac{1}{\pi} + \frac{e^{4/3}}{\sqrt{2} \pi^2} \cdot \int_{0}^p \frac{1}{\sqrt{x}} dx\\
		&\le \sqrt{p},
	\end{align*}
	therefore, the second statement of lemma holds.
	
	To prove the first statement of lemma, we bound the derivative of the polynomial $P_{relu}^{(p)}$ on the interval $\left[-1 - \frac{1}{6p}, 1 + \frac{1}{6p} \right]$ as follows,
	\begin{align*}
		\max_{\alpha \in \left[-1 - \frac{1}{6p}, 1 + \frac{1}{6p} \right]} \left| \frac{d}{d\alpha} \left(P_{relu}^{(p)}(\alpha)\right) \right| &= \frac{1}{\pi} \cdot \sum_{n=0}^p \frac{(2n)!}{2^{2n} \cdot (n!)^2 \cdot (2n+1)} \cdot \left( 1 + \frac{1}{6p} \right)^{2n+1}\\
		&\le \frac{19}{18\pi} + \frac{e^{25/18}}{\sqrt{2} \pi^2} \cdot \sum_{n=1}^p \frac{1}{\sqrt{n} \cdot (2n+1)}\\
		&\le \frac{19}{18\pi} + \frac{e^{25/18}}{\sqrt{2} \pi^2} \cdot \int_{0}^p \frac{1}{\sqrt{x} \cdot (2x+1)} dx\\
		&\le 1,
	\end{align*}
	therefore, the second statement of the lemma follows.
\end{proof}

For the rest of this section, we need two basic properties of tensor products and direct sums. 
\begin{align}
	\langle x \otimes y , z \otimes w\rangle = \langle x  , z\rangle \cdot \langle y  , w \rangle,\quad
	\langle x \oplus y , z \oplus w\rangle = \langle x  , z\rangle + \langle y  , w \rangle
\end{align}
for vectors $x,y,z,w$ with conforming sizes.

Now we are in a position to analyze the invariants that are maintained throughout the execution of NTK Sketch algorithm (Definition~\ref{alg-def-ntk-sketch}):
\begin{lemma}[Invariants of the NTK Sketch algorithm]
	\label{thm:ntk-sketch-corr}
	For every positive integers $d$ and $L$, every $\epsilon, \delta>0$, every vectors $y,z \in \RR^d$, if we let $\Sigma_{relu}^{(h)}:[-1,1] \to \RR$ and $K_{relu}^{(h)}:[-1,1] \to \RR$ be the functions defined in \eqref{eq:dp-covar-relu} and \eqref{eq:dp-ntk-relu} of Definition~\ref{def:relu-ntk}, then with probability at least $1-\delta$ the following invariants hold for every $h =0, 1, 2, \ldots L$:
	\begin{enumerate}
		\item The mapping $\phi^{(h)}(\cdot)$ computed by the NTK Sketch algorithm in \eqref{eq:map-covar-zero} and \eqref{eq:map-covar} of Definition~\ref{alg-def-ntk-sketch} satisfy
		\[ \left| \left< \phi^{(h)}(y), \phi^{(h)}(z) \right> - \Sigma_{relu}^{(h)}\left( \frac{\langle y , z \rangle}{\|y\|_2 \|z\|_2} \right) \right| \le ({h+1}) \cdot \frac{\epsilon^2}{60L^3}. \]
		\item The mapping $\psi^{(h)}(\cdot)$ computed by the NTK Sketch algorithm in \eqref{eq:map-relu} of Definition~\ref{alg-def-ntk-sketch} satisfy
		\[ \left| \left< \psi^{(h)}(y), \psi^{(h)}(z) \right> - K_{relu}^{(h)}\left( \frac{\langle y , z \rangle}{\|y\|_2 \|z\|_2} \right) \right| \le \epsilon \cdot \frac{h^2+1}{10L}. \]
	\end{enumerate}
\end{lemma}

\begin{proof}
	The proof is by induction on the value of $h=0,1,2, \ldots L$. 
	More formally, consider the following statements for every $h=0,1,2, \ldots L$:
	\begin{enumerate}[leftmargin=1.5cm]
		\item[${\bf P_1(h):}$]
		\[ \begin{split}
			&\left| \left< \phi^{(h)}(y), \phi^{(h)}(z) \right> - \Sigma_{relu}^{(h)}\left( \frac{\langle y , z \rangle}{\|y\|_2 \|z\|_2} \right) \right| \le  ({h+1}) \cdot \frac{\epsilon^2}{60L^3},\\
			&\left| \left\| \phi^{(h)}(y) \right\|_2^2 - 1 \right| \le ({h+1}) \cdot \frac{\epsilon^2}{60L^3}, \text{ and } \left| \left\| \phi^{(h)}(z) \right\|_2^2 - 1 \right| \le ({h+1}) \cdot \frac{\epsilon^2}{60L^3}.
		\end{split} \]
		\item[${\bf P_2(h):}$]
		\[ \begin{split}
			&\left| \left< \psi^{(h)}(y), \psi^{(h)}(z) \right> - K_{relu}^{(h)}\left( \frac{\langle y , z \rangle}{\|y\|_2 \|z\|_2} \right) \right| \le \epsilon \cdot \frac{h^2+1}{10L},\\
			&\left| \left\| \psi^{(h)}(y) \right\|_2^2 - K_{relu}^{(h)}(1) \right| \le \epsilon \cdot \frac{h^2+1}{10L}, \text{ and } \left| \left\| \psi^{(h)}(z) \right\|_2^2 - K_{relu}^{(h)}(1) \right| \le \epsilon \cdot \frac{h^2+1}{10L}.
		\end{split} \]
	\end{enumerate}
	We prove by induction that the following holds,
	\[ \Pr[ P_1(0)] \ge 1 - O(\delta/L), \text{ and } \Pr[ P_2(0)|P_1(0)] \ge 1 - O(\delta/L), \]
	Additionally, we prove that for every $h = 1,2, \ldots L$,
	\[ \Pr \left[ P_1(h) | P_1(h-1) \right] \ge 1 - O(\delta/L), \text{ and }\Pr\left[ P_2(h) | P_2(h-1), P_1(h), P_1(h-1) \right] \ge 1 - O(\delta/L).\]

	\paragraph{(1) Base of induction ($h=0$):} By \eqref{eq:map-covar-zero}, $\phi^{(0)}(y) = \frac{1}{\|y\|_2} \cdot S \cdot Q^{1} \cdot y$ and $\phi^{(0)}(z) = \frac{1}{\|z\|_2} \cdot S \cdot Q^{1} \cdot z$, thus, Lemma~\ref{lem:srht} implies the following
	\[ \Pr\left[ \left| \left< \phi^{(0)}(y), \phi^{(0)}(z)\right> - \frac{1}{\|y\|_2\|z\|_2} \cdot \left< Q^{1} y, Q^{1} z \right> \right| \le O\left(\frac{\epsilon^2}{L^3}\right)\cdot \frac{\left\| Q^{1} y \right\|_2 \| Q^{1} z \|_2}{\|y\|_2 \|z\|_2} \right] \ge 1 - O(\delta/L). \]
	By using the above together with Lemma~\ref{soda-result} and union bound as well as triangle inequality we find that the following holds,
	\[ \Pr\left[ \left| \left< \phi^{(0)}(y), \phi^{(0)}(z)\right> - \frac{\left<  y,  z \right>}{\|y\|_2 \|z\|_2} \right| \le O\left(\frac{\epsilon^2}{L^3}\right) \right] \ge 1 - O(\delta/L). \]
	Similarly, we can prove
	\begin{align*}
	\Pr\left[ \left| \left\| \phi^{(0)}(y)\right\|_2^2 - 1 \right| \le O\left(\epsilon^2/L^3\right) \right] &\ge 1 - O(\delta/L), \\ 
	\Pr\left[ \left| \left\| \phi^{(0)}(z)\right\|_2^2 - 1 \right| \le O\left(\epsilon^2/L^3\right) \right] &\ge 1 - O(\delta/L).
	\end{align*}
	Using union bound, this proves the base of induction for statement $P_1(h)$, i.e., $\Pr[ P_1(0) ] \ge 1 - O(\delta/L)$.
	
	Moreover, by \eqref{eq:map-relu}, $\psi^{(0)}(y) = V \cdot \phi^{(0)}(y)$ and $\psi^{(0)}(z) = V \cdot \phi^{(0)}(z)$, thus, Lemma~\ref{lem:srht} implies that,
	\[ \Pr\left[ \left| \left< \psi^{(0)}(y), \psi^{(0)}(z)\right> - \left< \phi^{(0)}(y), \phi^{(0)}(z) \right> \right| \le O\left(\frac{\epsilon}{L}\right)\cdot \left\| \phi^{(0)}(y) \right\|_2 \left\| \phi^{(0)}(z)\right\|_2 \right] \ge 1 - O(\delta/L). \]
	By conditioninig on $P_1(0)$ and using the above together with triangle inequality it follows that,
	\[ \Pr\left[ \left| \left< \psi^{(0)}(y), \psi^{(0)}(z)\right> - K_{relu}^{(0)}\left( \frac{\langle y , z \rangle}{\|y\|_2 \|z\|_2} \right) \right| \le \frac{\epsilon}{10L} \right] \ge 1 - O(\delta/L). \]
	Similarly we can prove that with probability $1 - O(\delta/L)$ we have $\left| \left\| \psi^{(0)}(y)\right\|_2^2 - K_{relu}^{(0)}( 1 ) \right| \le \frac{\epsilon}{10L}$ and $\left| \left\| \psi^{(0)}(z)\right\|_2^2 - K_{relu}^{(0)}( 1 ) \right| \le \frac{\epsilon}{10L}$,	
	which proves the base of induction for the second statement, i.e., $\Pr[P_2(0)|P_1(0)] \ge 1 - O(\delta/L)$. This completes the base of induction.
	
	\paragraph{(2) Inductive step:} Assuming that the inductive hypothesis holds for $h-1$.
	First, note that by Lemma~\ref{lem:srht} and using \eqref{eq:map-covar} we have the following,
	\begin{equation}\label{eq:phi-h}
		\Pr\left[ \left| \left< \phi^{(h)}(y), \phi^{(h)}(z)\right> - \sum_{j=0}^{2p+2} c_j \left<Z^{(h)}_{j}(y), Z^{(h)}_{j}(z)\right> \right| \le O\left(\frac{\epsilon^2}{L^3}\right) \cdot A \right] \ge 1 - O(\delta/L),
	\end{equation}
	where $A := \sqrt{\sum_{j=0}^{2p+2} c_j \|Z^{(h)}_{j}(y)\|_2^2} \cdot \sqrt{\sum_{j=0}^{2p+2} c_j \|Z^{(h)}_{j}(z)\|_2^2}$ and the collection of vectors $\left\{Z^{(h)}_{j}(y)\right\}_{j=0}^{2p+2}$ and $\left\{Z^{(h)}_{j}(z)\right\}_{j=0}^{2p+2}$ and coefficients $c_0,c_1, c_2, \ldots c_{2p+2}$ are defined as per \eqref{eq:map-covar} and \eqref{eq:poly-approx-krelu}, respectively. By Lemma~\ref{soda-result} together with union bound, the following inequalities hold true simultaneously for all $j \in \{0,1,2, \ldots 2p+2\}$, with probability at least $1 - O\left( \frac{\delta}{L} \right)$:
	\begin{align}
		&\left|\left<Z^{(h)}_{j}(y), Z^{(h)}_{j}(z)\right> - \left<\phi^{(h-1)}(y), \phi^{(h-1)}(z)\right>^j \right| \le O\left( \frac{\epsilon^2}{L^3} \right) \cdot \left\| \phi^{(h-1)}(y) \right\|_2^j \left\| \phi^{(h-1)}(z)\right\|_2^j \nonumber\\
		&\left\| Z^{(h)}_{j}(y)\right\|_2^2 \le \frac{11}{10} \cdot \left\| \phi^{(h-1)}(y) \right\|_2^{2j} \label{eq:Zj-sketch}\\
		&\left\| Z^{(h)}_{j}(z)\right\|_2^2 \le \frac{11}{10} \cdot \left\| \phi^{(h-1)}(z) \right\|_2^{2j} \nonumber
	\end{align}
	Therefore, by plugging \eqref{eq:Zj-sketch} back to \eqref{eq:phi-h} and using union bound, triangle inequality and Cauchy–Schwarz inequality we find that,
	\begin{equation} \label{eq:phi-inner-prod-bound}
		\Pr\left[ \left| \left< \phi^{(h)}(y), \phi^{(h)}(z)\right> - P^{(p)}_{relu}\left( \left<\phi^{(h-1)}(y), \phi^{(h-1)}(z)\right> \right) \right| \le O\left(\frac{\epsilon^2}{L^3}\right) \cdot B \right] \ge 1 - O(\delta/L),
	\end{equation}
	where $B:= \sqrt{P^{(p)}_{relu}\left(\|\phi^{(h-1)}(y)\|_2^2\right) \cdot P^{(p)}_{relu}\left(\|\phi^{(h-1)}(z)\|_2^2\right)}$ and $P^{(p)}_{relu}(\alpha) = \sum_{j=0}^{2p+2} c_j \cdot \alpha^j$ is the polynomial defined in \eqref{eq:poly-approx-krelu}. Using the inductive hypothesis $P_1(h-1)$, we have that
	\begin{align*}
		\left| \left\| \phi^{(h-1)}(y) \right\|_2^2 - 1 \right| \le h \cdot \frac{\varepsilon^2}{60L^3},~\text{ and }~\left| \left\| \phi^{(h-1)}(z) \right\|_2^2 - 1 \right| \le h \cdot \frac{\varepsilon^2}{60L^3}.
	\end{align*}
	Therefore, by Lemma~\ref{lema:sensitivity-polynomial} we have 
	\begin{align*}
	\left| P_{relu}^{(p)}\left(\|\phi^{(h-1)}(y)\|_2^2\right) - P_{relu}^{(p)}(1) \right| \le h \cdot \frac{\epsilon^2}{60L^3}, ~\text{ and }~ \left| P_{relu}^{(p)}\left(\|\phi^{(h-1)}(z)\|_2^2\right) - P_{relu}^{(p)}(1) \right| \le h \cdot \frac{\epsilon^2}{60L^3}.
	\end{align*}
	Consequently, because $P_{relu}^{(p)}(1) \le P_{relu}^{(+\infty)}(1) = 1$, we find that $B \le \frac{11}{10}$.
	By plugging this into \eqref{eq:phi-inner-prod-bound} we have,
	\begin{equation} \label{eq:phi-inner-prod-final}
		\Pr\left[ \left| \left< \phi^{(h)}(y), \phi^{(h)}(z)\right> - P^{(p)}_{relu}\left( \left<\phi^{(h-1)}(y), \phi^{(h-1)}(z)\right> \right) \right| \le O\left(\frac{\epsilon^2}{L^3}\right) \right] \ge 1 - O(\delta/L).
	\end{equation}

	Furthermore, the inductive hypothesis $P_1(h-1)$ implies that 
	\[\left| \left< \phi^{(h-1)}(y), \phi^{(h-1)}(z) \right> - \Sigma_{relu}^{(h-1)}\left( \frac{\langle y , z \rangle}{\|y\|_2 \|z\|_2} \right) \right| \le h \cdot \frac{\epsilon^2}{60L^3}. \]
	Hence, by invoking Lemma~\ref{lema:sensitivity-polynomial} we find that,
	\[ \left|P^{(p)}_{relu}\left( \left<\phi^{(h-1)}(y), \phi^{(h-1)}(z)\right> \right) - P^{(p)}_{relu}\left( \Sigma_{relu}^{(h-1)}\left( \frac{\langle y , z \rangle}{\|y\|_2 \|z\|_2} \right) \right) \right| \le h \cdot \frac{\epsilon^2}{60L^3}. \]
	By incorporating the above inequality into \eqref{eq:phi-inner-prod-final} using triangle inequality we find that,
	\begin{equation} \label{eq:phi-inner-prod-}
		\Pr\left[ \left| \left< \phi^{(h)}(y), \phi^{(h)}(z)\right> - P^{(p)}_{relu}\left(\Sigma_{relu}^{(h-1)}\left( \frac{\langle y , z \rangle}{\|y\|_2 \|z\|_2} \right)\right) \right| \le \frac{h \cdot \epsilon^2}{60L^3} +  O\left(\frac{\epsilon^2}{L^3}\right) \right] \ge 1 - O(\delta/L).
	\end{equation}
	Now, by invoking Lemma~\ref{lem:polynomi-approx-krelu} and using the fact that $p = \left\lceil 2 L^2 /{\epsilon}^{4/3} \right\rceil$ we have,
	\[ \left| P_{relu}^{(p)}\left(\Sigma_{relu}^{(h-1)}\left( \frac{\langle y , z \rangle}{\|y\|_2 \|z\|_2} \right)\right) - \kappa_1\left(\Sigma_{relu}^{(h-1)}\left( \frac{\langle y , z \rangle}{\|y\|_2 \|z\|_2} \right)\right) \right| \le \frac{\epsilon^2}{76 L^3}. \]
	By combining the above inequality with \eqref{eq:phi-inner-prod-} using triangle inequality and using the fact that $\kappa_1\left(\Sigma_{relu}^{(h-1)}\left( \frac{\langle y , z \rangle}{\|y\|_2 \|z\|_2} \right)\right) = \Sigma_{relu}^{(h)}\left( \frac{\langle y , z \rangle}{\|y\|_2 \|z\|_2} \right)$ (by \eqref{eq:dp-covar-relu}), we get the following bound,
	\[ \Pr\left[ \left| \left< \phi^{(h)}(y), \phi^{(h)}(z)\right> - \Sigma_{relu}^{(h)}\left( \frac{\langle y , z \rangle}{\|y\|_2 \|z\|_2} \right) \right| \le (h+1) \cdot \frac{\epsilon^2}{60L^3} \right] \ge 1 - O(\delta/L). \]
	Similarly, we can prove that
	\begin{align*}
		&\Pr\left[ \left| \left\| \phi^{(h)}(y) \right\|_2^2 - 1 \right| >  \frac{(h+1) \cdot\varepsilon^2}{60L^3} \right] \le O\left( \frac{\delta}{L} \right), \text{ and } \\
		&\Pr\left[ \left| \left\| \phi^{(h)}(z) \right\|_2^2 - 1 \right| > \frac{(h+1) \cdot\varepsilon^2}{60L^3} \right] \le O\left( \frac{\delta}{L} \right).
	\end{align*}
	This is sufficient to prove the inductive step by union bound, i.e., $\Pr[P_1(h)|P_1(h-1)] \ge 1 - O(\delta/L)$.
	
	Now we prove the inductive step for statement $P_2(h)$, that is, we prove that conditioned on $P_2(h-1), P_1(h), P_1(h-1)$, statement $P_2(h)$ holds with probability at least $1-O(\delta/L)$.
	First, note that by Lemma~\ref{lem:srht} and using \eqref{eq:map-derivative-covar} we have,
	\begin{equation}\label{eq:phi-dot-h}
		\Pr\left[ \left| \left< \dot{\phi}^{(h)}(y), \dot{\phi}^{(h)}(z)\right> - \sum_{j=0}^{2p'+1} b_j \left<Y^{(h)}_{j}(y), Y^{(h)}_{j}(z)\right> \right| \le O\left(\frac{\epsilon}{L}\right) \cdot \widehat{A} \right] \ge 1 - O(\delta/L),
	\end{equation}
	where $\widehat{A} := \sqrt{\sum_{j=0}^{2p'+1} b_j \|Y^{(h)}_{j}(y)\|_2^2} \cdot \sqrt{\sum_{j=0}^{2p'+1} b_j \|Y^{(h)}_{j}(z)\|_2^2}$ and the collection of vectors $\left\{Y^{(h)}_{j}(y)\right\}_{j=0}^{2p'+1}$ and $\left\{Y^{(h)}_{j}(z)\right\}_{j=0}^{2p'+1}$ and coefficients $b_0,b_1, b_2, \ldots b_{2p'+1}$ are defined as per \eqref{eq:map-derivative-covar} and \eqref{eq:poly-approx-krelu}, respectively. 
	By invoking Lemma~\ref{soda-result} along with union bound, with probability at least $1 - O\left( \frac{\delta}{L} \right)$, the following inequalities hold true simultaneously for all $j \in \{0,1,2, \ldots 2p'+1\}$
	\begin{align}
		& \left|\left<Y^{(h)}_{j}(y), Y^{(h)}_{j}(z)\right> - \left<\phi^{(h-1)}(y), \phi^{(h-1)}(z)\right>^j \right| \le O\left( \frac{\epsilon}{L} \right) \cdot \left\| \phi^{(h-1)}(y) \right\|_2^j \left\| \phi^{(h-1)}(z)\right\|_2^j \nonumber\\
		& \left\| Y^{(h)}_{j}(y)\right\|_2^2 \le \frac{11}{10} \cdot \left\| \phi^{(h-1)}(y) \right\|_2^{2j} \label{eq:Yj-sketch}\\
		& \left\| Y^{(h)}_{j}(z)\right\|_2^2 \le \frac{11}{10} \cdot \left\| \phi^{(h-1)}(z) \right\|_2^{2j} \nonumber
	\end{align}
	Therefore, by plugging \eqref{eq:Yj-sketch} into \eqref{eq:phi-dot-h} and using union bound, triangle inequality and Cauchy–Schwarz inequality we find that,
	\begin{equation} \label{eq:phi-dot-inner-prod-bound}
		\Pr\left[ \left| \left< \dot{\phi}^{(h)}(y), \dot{\phi}^{(h)}(z)\right> - \dot{P}^{(p')}_{relu}\left( \left<\phi^{(h-1)}(y), \phi^{(h-1)}(z)\right> \right) \right| \le O\left(\frac{\epsilon}{L}\right) \cdot \widehat{B} \right] \ge 1 - O(\delta/L),
	\end{equation}
	where $\widehat{B}:= \sqrt{\dot{P}^{(p')}_{relu}\left(\|\phi^{(h-1)}(y)\|_2^2\right) \cdot \dot{P}^{(p')}_{relu}\left(\|\phi^{(h-1)}(z)\|_2^2\right)}$ and $\dot{P}^{(p')}_{relu}(\alpha) = \sum_{j=0}^{2p'+1} b_j \cdot \alpha^j$ is the polynomial defined in \eqref{eq:poly-approx-krelu}.
	By inductive hypothesis $P_1(h-1)$ we have $\left| \left\| \phi^{(h-1)}(y) \right\|_2^2 - 1 \right| \le h \cdot \frac{\epsilon^2}{60L^3}$ and $\left| \left\| \phi^{(h-1)}(z) \right\|_2^2 - 1 \right| \le h \cdot \frac{\epsilon^2}{60L^3}$. 
	Therefore, using the fact that $p' = \left\lceil 9L^2 /\epsilon^{2} \right\rceil$ and Lemma~\ref{lema:sensitivity-polynomial}, $\left| \dot{P}_{relu}^{(p')}\left(\|\phi^{(h-1)}(y)\|_2^2\right) - \dot{P}_{relu}^{(p')}(1) \right| \le  \frac{h \cdot\epsilon}{20L^2}$ and $\left| \dot{P}_{relu}^{(p')}\left(\|\phi^{(h-1)}(z)\|_2^2\right) - \dot{P}_{relu}^{(p')}(1) \right| \le  \frac{h \cdot\epsilon}{20L^2}$. Consequently, because $\dot{P}_{relu}^{(p')}(1) \le \dot{P}_{relu}^{(+\infty)}(1) = 1$, we find that
	\[ \widehat{B} \le \frac{11}{10}.\]
	By plugging this into \eqref{eq:phi-dot-inner-prod-bound} we have,
	\begin{equation} \label{eq:phi-dot-inner-prod-final}
		\Pr\left[ \left| \left< \dot{\phi}^{(h)}(y), \dot{\phi}^{(h)}(z)\right> - \dot{P}^{(p')}_{relu}\left( \left<\phi^{(h-1)}(y), \phi^{(h-1)}(z)\right> \right) \right| \le O\left(\frac{\epsilon}{L}\right) \right] \ge 1 - O(\delta/L).
	\end{equation}

	Furthermore, inductive hypothesis $P_1(h-1)$ implies $\left| \left< \phi^{(h-1)}(y), \phi^{(h-1)}(z) \right> - \Sigma_{relu}^{(h-1)}\left( \frac{\langle y , z \rangle}{\|y\|_2 \|z\|_2} \right) \right| \le \frac{h \cdot \epsilon^2}{60L^3}$,
	hence, by invoking Lemma~\ref{lema:sensitivity-polynomial} we find that,
	\[ \left|\dot{P}^{(p')}_{relu}\left( \left<\phi^{(h-1)}(y), \phi^{(h-1)}(z)\right> \right) - \dot{P}^{(p')}_{relu}\left( \Sigma_{relu}^{(h-1)}\left( \frac{\langle y , z \rangle}{\|y\|_2 \|z\|_2} \right) \right) \right| \le  \frac{h \cdot \epsilon}{20L^2}. \]
	By plugging the above inequality into \eqref{eq:phi-dot-inner-prod-final} using triangle inequality we find that,
	\begin{equation} \label{eq:phi-dot-inner-prod-}
		\Pr\left[ \left| \left< \dot{\phi}^{(h)}(y), \dot{\phi}^{(h)}(z)\right> - \dot{P}^{(p')}_{relu}\left(\Sigma_{relu}^{(h-1)}\left( \frac{\langle y , z \rangle}{\|y\|_2 \|z\|_2} \right)\right) \right| \le \frac{h \cdot \epsilon}{20L^2} +  O\left(\frac{\epsilon}{L}\right) \right] \ge 1 - O(\delta/L).
	\end{equation}
	Now, by invoking Lemma~\ref{lem:polynomi-approx-krelu} and using the fact that $p' = \left\lceil 9L^2 / {\epsilon}^2 \right\rceil$ we have,
	\[ \left| \dot{P}_{relu}^{(p')}\left(\Sigma_{relu}^{(h-1)}\left( \frac{\langle y , z \rangle}{\|y\|_2 \|z\|_2} \right)\right) - \kappa_0\left(\Sigma_{relu}^{(h-1)}\left( \frac{\langle y , z \rangle}{\|y\|_2 \|z\|_2} \right)\right) \right| \le \frac{\epsilon}{15 L}. \]
	By combining the above inequality with \eqref{eq:phi-dot-inner-prod-} using triangle inequality and using the fact that $\kappa_0\left(\Sigma_{relu}^{(h-1)}\left( \frac{\langle y , z \rangle}{\|y\|_2 \|z\|_2} \right)\right) = \dot{\Sigma}_{relu}^{(h)}\left( \frac{\langle y , z \rangle}{\|y\|_2 \|z\|_2} \right)$ (by \eqref{eq:dp-derivative-covar-relu}) we get the following bound,
	\begin{equation}\label{eq:phi-dot-error-bound}
		\Pr\left[ \left| \left< \dot{\phi}^{(h)}(y), \dot{\phi}^{(h)}(z)\right> - \dot{\Sigma}_{relu}^{(h)}\left( \frac{\langle y , z \rangle}{\|y\|_2 \|z\|_2} \right) \right| \le  \frac{\epsilon}{8L} \right] \ge 1 - O(\delta/L).
	\end{equation}
	Similarly we can show that,
	\begin{equation}\label{eq:phi-dot-norm-bound}
		\begin{split}
			&\Pr\left[ \left| \left\| \dot{\phi}^{(h)}(y)\right\| - 1 \right| \le \frac{\epsilon}{8L} \right] \ge 1 - O(\delta/L),\\
			&\Pr\left[ \left| \left\| \dot{\phi}^{(h)}(z)\right\| - 1 \right| \le  \frac{\epsilon}{8L} \right] \ge 1 - O(\delta/L).
		\end{split}
	\end{equation}

	Now if we let $f := \psi^{(h-1)}(y) \otimes \dot{\phi}^{(h)}(y)$ and $g := \psi^{(h-1)}(z) \otimes \dot{\phi}^{(h)}(z)$, then by Lemma~\ref{lem:srht} and using \eqref{eq:map-relu} we have the following,
	\begin{equation}\label{eq:psi-h-inner-prod}
		\Pr\left[ \left| \left< \psi^{(h)}(y) , \psi^{(h)}(z) \right> - \left< Q^2  f \oplus \phi^{(h)}(y), Q^2  g \oplus \phi^{(h)}(z) \right> \right| \le O\left( {\epsilon}/{L} \right) \cdot D \right] \ge 1 - O(\delta/L), 
	\end{equation}
	where $D := \left\| Q^2  f \oplus \phi^{(h)}(y) \right\|_2 \left\| Q^2  g \oplus \phi^{(h)}(z) \right\|_2$. By the fact that we conditioned on $P_1(h)$, we have,
	\[ D \le  \sqrt{\left\| Q^2  f \right\|_2^2 + \frac{11}{10}} \cdot \sqrt{\left\| Q^2  g \right\|_2^2 + \frac{11}{10}}. \]
	By Lemma~\ref{soda-result}, we can further upper bound the above as follows,
	\[ D \le \frac{11}{10} \cdot \sqrt{\left\| f \right\|_2^2 + 1} \cdot \sqrt{\left\| g \right\|_2^2 + 1}. \]
	Now note that because we conditioned on $P_2(h-1)$ and using \eqref{eq:phi-dot-norm-bound}, with probability at least $1 - O(\delta/L)$ the following holds:
	\[ \left\| f \right\|_2^2 = \left\| \psi^{(h-1)}(y) \right\|_2^2 \left\| \dot{\phi}^{(h)}(y) \right\|_2^2 \le \frac{11}{10} \cdot K^{(h-1)}_{relu}(1) = 11h/10. \]
	Similarly, $\left\| g \right\|_2^2 \le 11h/10$ with probability at least $1-O(\delta/L)$, thus, by union bound:
	\[ \Pr[D \le 2 (h+1) | P_2(h-1), P_1(h), P_1(h-1)] \ge 1 - O(\delta/L). \]
	Therefore, by combining the above with \eqref{eq:psi-h-inner-prod} via union bound we find that,
	\begin{equation}\label{eq:psi-h-inner-prod-bound}
		\Pr\left[ \left| \left< \psi^{(h)}(y) , \psi^{(h)}(z) \right> - \left< Q^2  f \oplus \phi^{(h)}(y), Q^2  g \oplus \phi^{(h)}(z) \right> \right| \le O\left({\epsilon h}/{L}\right) \right] \ge 1 - O(\delta/L), 
	\end{equation}
	Now note that $\left< Q^2  f \oplus \phi^{(h)}(y), Q^2  g \oplus \phi^{(h)}(z) \right> = \left< Q^2  f, Q^2  g \right> + \left< \phi^{(h)}(y), \phi^{(h)}(z) \right>$. We proceed by bounding the term $\left| \left< Q^2  f, Q^2  g \right> - \left< f, g \right> \right|$ using Lemma~\ref{soda-result}, as follows,
	\[ \Pr\left[ \left| \left< Q^2  f, Q^2  g \right> - \left< f, g \right> \right| \le O\left( \epsilon/L \right) \cdot \|f\|_2 \|g\|_2 \right] \ge 1 - O(\delta/L). \]
	We proved that conditioned on $P_2(h-1)$ and $P_1(h-1)$, $\left\| f \right\|_2^2 \le 11h/10$ and $\left\| g \right\|_2^2 \le 11h/10$ with probability at least $1-O(\delta/L)$. Hence, by union bound we find that,
	\begin{equation}\label{eq:first-term-kernel-bound}
		\Pr\left[ \left.\left| \left< Q^2  f, Q^2  g \right> - \left< f, g \right> \right| \le O\left( \epsilon h/L \right) \right| P_2(h-1), P_1(h-1) \right] \ge 1 - O(\delta/L).
	\end{equation}
	Note that $\left< f, g \right> = \left< \psi^{(h-1)}(y), \psi^{(h-1)}(z) \right> \cdot \left< \dot{\phi}^{(h)}(y), \dot{\phi}^{(h)}(z) \right>$, thus by conditioning on inductive hypothesis $P_2(h-1)$ and \eqref{eq:phi-dot-error-bound} we have,
	\[ \begin{split}
		\left| \left< f, g \right> - K_{relu}^{(h-1)}\left( \frac{\langle y , z \rangle}{\|y\|_2 \|z\|_2} \right) \cdot \dot{\Sigma}_{relu}^{(h)}\left( \frac{\langle y , z \rangle}{\|y\|_2 \|z\|_2} \right) \right| &\le \frac{\epsilon}{8L} \cdot \left(h+\epsilon \cdot \frac{(h-1)^2+1}{10L} \right) + \epsilon \cdot \frac{(h-1)^2+1}{10L}
	\end{split}
	\]
	By combining the above inequality, \eqref{eq:first-term-kernel-bound}, $P_1(h)$, and \eqref{eq:psi-h-inner-prod-bound} using triangle inequality and union bound we get the following inequality,
	\small
	\[ 	\Pr\left[ \left| \left< \psi^{(h)}(y) , \psi^{(h)}(z) \right> - K_{relu}^{(h-1)}\left( \frac{\langle y , z \rangle}{\|y\|_2 \|z\|_2} \right) \dot{\Sigma}_{relu}^{(h)}\left( \frac{\langle y , z \rangle}{\|y\|_2 \|z\|_2} \right) - \Sigma_{relu}^{(h)}\left( \frac{\langle y , z \rangle}{\|y\|_2 \|z\|_2} \right) \right| > \epsilon \cdot \frac{h^2+1}{10L} \right] \le O(\delta/L). \]
	\normalsize
	By noting that $K_{relu}^{(h-1)}\left( \cdot \right) \cdot \dot{\Sigma}_{relu}^{(h)}\left( \cdot \right) + \Sigma_{relu}^{(h)}\left( \cdot \right) = K_{relu}^{(h)}\left( \cdot \right)$ (see \eqref{eq:dp-ntk-relu}) we have indeed proved that
	\[ 	\Pr\left[ \left| \left< \psi^{(h)}(y) , \psi^{(h)}(z) \right> - K_{relu}^{(h)}\left( \frac{\langle y , z \rangle}{\|y\|_2 \|z\|_2} \right) \right| \le \epsilon \cdot \frac{h^2+1}{10L} \right] \ge 1 - O(\delta/L). \]
	Similarly we can prove the following inequalities hold with probability at least $1 - O(\delta/L)$,
	\[ \left| \left\| \psi^{(h)}(y) \right\|_2^2 - K_{relu}^{(h)}( 1) \right| \le \epsilon \cdot \frac{h^2+1}{10L}, \text{ and } \left| \left\| \psi^{(h)}(z) \right\|_2^2 - K_{relu}^{(h)}( 1) \right| \le \epsilon \cdot \frac{h^2+1}{10L}. \]
	This proves the inductive step for the statement $P_2(h)$ follows, i.e.,
	\[\Pr[ P_2(h) | P_2(h-1), P_1(h), P_1(h-1) ] \ge 1 - O(\delta/L).\]
	Therefore, by union bounding over all $h=0,1,2, \ldots L$, it follows that the statements of the lemma hold simultaneously for all $h$ with probability at least $1 - \delta$, which completes the proof.
	
\end{proof}

We now analyze the runtime of the NTK Sketch algorithm:
\begin{lemma}[Runtime of the NTK Sketch]
	\label{thm:ntk-sketch-runtime}
	For every positive integers $d$ and $L$, every $\epsilon, \delta>0$, every vector $x \in \RR^d$, the time to compute the NTK sketch $\Psi_{ntk}^{(L)}(x) \in \RR^{s^*}$, for $s^*=O\left( \frac{1}{\epsilon^2} \cdot \log \frac{1}{\delta} \right)$, using the procedure given in Definition~\ref{alg-def-ntk-sketch} is bounded by,
	\[ O\left( \frac{L^{11}}{\epsilon^{6.7}} \cdot \log^3 \frac{L}{\epsilon\delta} + \frac{L^3}{\epsilon^2} \cdot \log \frac{L}{\epsilon\delta} \cdot \text{nnz}(x) \right). \]
\end{lemma}
\begin{proof}
	There are three main components to the runtime of this procedure that we have to account for. The first is the time to apply the sketch $Q^1$ to $x$ in \eqref{eq:map-covar-zero}. By Lemma~\ref{soda-result}, the runtime of computing $Q^1\cdot x$ is $O\left( \frac{L^6}{\epsilon^4} \cdot \log^3 \frac{L}{\epsilon\delta} + \frac{L^3}{\epsilon^2} \cdot \log \frac{L}{\epsilon\delta} \cdot \text{nnz}(x) \right)$. 
	The second heavy operation corresponds to computing vectors $Z_j^{(h)}(x) = Q^{2p+2} \cdot \left(\left[ \phi^{(h-1)}(x) \right]^{\otimes j} \otimes  e_1^{\otimes 2p+2-j}\right)$ for $j=0,1,2, \ldots 2p+2$ and $h=1,2, \ldots L$ in \eqref{eq:map-covar}. By Lemma~\ref{soda-result}, the time to compute $Z_j^{(h)}(x)$ for a fixed $h$ and all $j=0,1,2, \ldots 2p+2$ is bounded by,
	\[ O\left( L^{10}/\epsilon^{20/3} \cdot \log^2({L}/{\epsilon}) \log^3 \frac{L}{\epsilon\delta} + L^{8}/\epsilon^{16/3} \cdot \log^3 \frac{L}{\epsilon\delta} \right) = O\left( \frac{L^{10}}{\epsilon^{6.7}} \cdot \log^3 \frac{L}{\epsilon\delta} \right). \]
	The total time to compute vectors $Z_j^{(h)}(x)$ for all $h=1,2, \ldots L$ and all $j=0,1,2, \ldots 2p+2$ is thus $O\left( \frac{L^{11}}{\epsilon^{6.7}} \cdot \log^3 \frac{L}{\epsilon\delta} \right)$. 
	Finally, the last computationally expensive operation is computing vectors $Y_j^{(h)}(x) = Q^{2p'+1} \cdot \left(\left[ \phi^{(h-1)}(x) \right]^{\otimes j} \otimes  e_1^{\otimes 2p'+1-j}\right)$ for $j=0,1,2, \ldots 2p'+1$ and $h=1,2, \ldots L$ in \eqref{eq:map-derivative-covar}.
	By Lemma~\ref{soda-result}, the runtime of computing $Y_j^{(h)}(x)$ for a fixed $h$ and all $j=0,1,2, \ldots 2p'+1$ is bounded by,
	\[ O\left( \frac{L^{6}}{\epsilon^6} \cdot \log^2({L}/{\epsilon}) \log^3 \frac{L}{\epsilon\delta} + \frac{L^{8}}{\epsilon^6} \cdot \log^3 \frac{L}{\epsilon\delta} \right) = O\left( \frac{L^{8}}{\epsilon^6}  \cdot \log^3 \frac{L}{\epsilon\delta} \right). \]
	Hence, the total time to compute vectors $Y_j^{(h)}(x)$ for all $h=1,2, \ldots L$ and all $j=0,1,2, \ldots 2p'+1$ is $O\left( \frac{L^{9}}{\epsilon^6} \cdot \log^3 \frac{L}{\epsilon\delta} \right)$. The total runtime of the NTK Sketch is obtained by summing up these three contributions.
\end{proof}

Now we are ready to prove the main theorem on the NTK Sketch, 

\begin{proofof}{Theorem~\ref{mainthm-ntk}}
	Let $\psi^{(L)}:\RR^d \to \RR^s$ for $s=O\left( \frac{L^2}{\epsilon^2} \cdot \log^2 \frac{L}{\epsilon\delta} \right)$ be the mapping defined in \eqref{eq:map-relu} of Definition~\ref{alg-def-ntk-sketch}. By \eqref{Psi-ntk-def}, the NTK Sketch $\Psi_{ntk}^{(L)}(x)$ is defined as 
	\[\Psi_{ntk}^{(L)}(x):= {\|x\|_2} \cdot G \cdot \psi^{(L)}( x).\]
	Because $G$ is a matrix of i.i.d normal entries with $s^{*} = C \cdot \frac{1}{\epsilon^2} \cdot \log\frac{1}{\delta}$ rows for large enough constant $C$, by \cite{dasgupta2003elementary}, $G$ is a JL transform and hence $\Psi_{ntk}^{(L)}$ satisfies the following,
	\[ \Pr \left[ \left| \left< \Psi_{ntk}^{(L)}(y) , \Psi_{ntk}^{(L)}(z) \right> - {\|y\|_2}{\|z\|_2}  \cdot \left< \psi^{(L)}( y),  \psi^{(L)}(z) \right> \right| \le O(\epsilon) \cdot A \right] \ge 1 - O(\delta), \]
	where $A := {\|y\|_2}{\|z\|_2} \left\| \psi^{(L)}({y}) \right\|_2  \left\| \psi^{(L)}({z}) \right\|_2$. By Lemma~\ref{thm:ntk-sketch-corr} and using the fact that $K_{relu}^{(L)}(1) = L+1$, the following bounds hold with probability at least $1 - O(\delta)$:
	\[ \left\| \psi^{(L)}( {y}) \right\|_2^2 \le \frac{11}{10} \cdot (L+1), \text{ and } \left\| \psi^{(L)}( {z}) \right\|_2^2 \le \frac{11}{10} \cdot (L+1). \]
	Therefore, by union bound we find that,
	\[ \Pr \left[ \left| \left< \Psi_{ntk}^{(L)}(y) , \Psi_{ntk}^{(L)}(z) \right> - {\|y\|_2}{\|z\|_2}  \cdot \left< \psi^{(L)}({y}),  \psi^{(L)}( {z}) \right> \right| \le O(\epsilon L)\cdot {\|y\|_2}{\|z\|_2} \right] \ge 1 - O(\delta). \]
	Additionally, by Lemma~\ref{thm:ntk-sketch-corr}, the following holds with probability at least $1 - O(\delta)$:
	\[ \left| \left< \psi^{(L)}( {y}),  \psi^{(L)}( {z}) \right> - K_{relu}^{(L)}\left( \frac{\langle y, z \rangle }{\|y\|_2\|z\|_2} \right) \right| \le \frac{\epsilon (L+1)}{10}. \]
	Hence by union bound and triangle inequality we have,
	\[ \Pr \left[ \left| \left< \Psi_{ntk}^{(L)}(y) , \Psi_{ntk}^{(L)}(z) \right> - {\|y\|_2}{\|z\|_2}  \cdot K_{relu}^{(L)}\left( \frac{\langle y, z \rangle }{\|y\|_2\|z\|_2} \right) \right| \le \frac{\epsilon (L+1)}{9} \cdot {\|y\|_2}{\|z\|_2} \right] \ge 1 - O(\delta). \]
	Now note that by Theorem~\ref{thm:ntk-relu}, ${\|y\|_2}{\|z\|_2}  \cdot K_{relu}^{(L)}\left( \frac{\langle y, z \rangle }{\|y\|_2\|z\|_2} \right) = \Theta_{ntk}^{(L)}(y,z)$, and also note that for every $L\ge 2$ and any $\alpha \in [-1,1]$, $K_{relu}^{(L)}\left( \alpha \right) \ge (L+1)/9$, therefore,
	\[ \Pr \left[ \left| \left< \Psi_{ntk}^{(L)}(y) , \Psi_{ntk}^{(L)}(z) \right> - \Theta_{ntk}^{(L)}(y,z) \right| \le \epsilon \cdot \Theta_{ntk}^{(L)}(y,z) \right] \ge 1 - \delta. \] 
	
	{\bf Remark on the fact that $K_{relu}^{(L)}\left( \alpha \right) \ge (L+1)/9$ for every $L\ge 2$ and any $\alpha \in [-1,1]$.}  Note that from the definition of $\Sigma_{relu}^{(h)}$ in \eqref{eq:dp-covar-relu}, we have that for any $\alpha \in [-1,1]$: $\Sigma_{relu}^{(0)}(\alpha) \ge -1$, $\Sigma_{relu}^{(1)}(\alpha) \ge 0$, $\Sigma_{relu}^{(2)}(\alpha) \ge \frac{1}{\pi}$, and $\Sigma_{relu}^{(h)}(\alpha) \ge \frac{1}{2}$ for every $h \ge 3$ because $\kappa_1(\cdot)$ is a monotonically increasing function on the interval $[-1,1]$. Moreover, using the definition of $\dot{\Sigma}_{relu}^{(h)}$ in \eqref{eq:dp-derivative-covar-relu}, we have that for any $\alpha \in [-1,1]$: $\dot{\Sigma}_{relu}^{(1)}(\alpha) \ge 0$, $\dot{\Sigma}_{relu}^{(2)}(\alpha) \ge \frac{1}{2}$, and $\dot{\Sigma}_{relu}^{(h)}(\alpha) \ge \frac{3}{5}$ for every $h \ge 3$ because $\kappa_1(\cdot)$ is a monotonically increasing function on the interval $[-1,1]$. By an inducive proof and using the definition of $K_{relu}^{(L)}$ in \eqref{eq:dp-ntk-relu}, we can show that $K_{relu}^{(L)}\left( \alpha \right) \ge (L+1)/9$ for every $L\ge 2$ and any $\alpha \in [-1,1]$
	
	{\bf Runtime:} By Lemma~\ref{thm:ntk-sketch-runtime}, time to compute the NTK Sketch is $O\left( \frac{L^{11}}{\epsilon^{6.7}} \log^3 \frac{L}{\epsilon\delta} + \frac{L^3}{\epsilon^2} \log \frac{L}{\epsilon\delta} \cdot \text{nnz}(x) \right)$.
\end{proofof}

\section{ReLU-CNTK: Expression and Main Properties}\label{appendix-relu-cntk-expr}
In this section we prove that Definition~\ref{relu-cntk-def} computes the CNTK kernel function corresponding to ReLU activaion and additionally, we present useful corollaries and consequences of this fact.
We start by restating the DP proposed in \cite{arora2019exact} for computing the $L$-layered CNTK kernel corresponding to an arbitrary activation function $\sigma:\RR \to \RR$ and convolutional filters of size $q \times q$, with global average pooling (GAP):

\begin{enumerate}
	\item Let $y,z \in \RR^{d_1\times d_2 \times c}$ be two input images, where $c$ is the number of channels ($c=3$ for standard color images). 
	Define $\Gamma^{(0)}: \RR^{d_1\times d_2 \times c} \times \RR^{d_1 \times d_2 \times c} \to \RR^{d_1 \times d_2 \times d_1 \times d_2}$ and $\Sigma^{(0)}: \RR^{d_1\times d_2 \times c} \times \RR^{d_1 \times d_2 \times c} \to \RR^{d_1 \times d_2 \times d_1 \times d_2}$ as follows for every $i,i' \in [d_1]$ and $j,j' \in [d_2]$:
	\small
	\begin{equation}\label{eq:dp-cntk-zero}
		\Gamma^{(0)}(y,z)  := \sum_{l=1}^c y_{(:,:,l)} \otimes z_{(:,:,l)}, \text{ and } \Sigma^{(0)}_{i,j,i',j'}(y,z) := \sum_{a=-\frac{q-1}{2}}^{\frac{q-1}{2}} \sum_{b=-\frac{q-1}{2}}^{\frac{q-1}{2}}  \Gamma^{(0)}_{i+a,j+b,i'+a,j'+b}(y,z).
	\end{equation}
	\normalsize
	\item For every layer $h = 1,2, \ldots , L$ of the network and every $i,i' \in [d_1]$ and $j,j' \in [d_2]$, define $\Gamma^{(h)}: \RR^{d_1\times d_2 \times c} \times \RR^{d_1 \times d_2 \times c} \to \RR^{d_1 \times d_2 \times d_1 \times d_2}$ recursively as:
	\begin{equation}\label{eq:dp-cntk-covar}
		\begin{split}
			&\Lambda^{(h)}_{i,j,i',j'}(y,z) := \begin{pmatrix}
				\Sigma^{(h-1)}_{i,j,i,j}(y,y) & \Sigma^{(h-1)}_{i,j,i',j'}(y,z)\\
				&\\
				\Sigma^{(h-1)}_{i',j',i,j}(z,y) & \Sigma^{(h-1)}_{i',j',i',j'}(z,z)
			\end{pmatrix},\\
			&\Gamma^{(h)}_{i,j,i',j'}(y,z) := \frac{1}{q^2 \cdot \EE_{w\sim \mathcal{N}(0,1)} \left[ |\sigma(w)|^2 \right]} \cdot \EE_{(u,v) \sim \mathcal{N}\left( 0, \Lambda^{(h)}_{i,j,i',j'}(y,z) \right)} \left[ \sigma(u) \cdot \sigma(v) \right],\\
			&\Sigma^{(h)}_{i,j,i',j'}(y,z) := \sum_{a=-\frac{q-1}{2}}^{\frac{q-1}{2}} \sum_{b=-\frac{q-1}{2}}^{\frac{q-1}{2}}  \Gamma^{(h)}_{i+a,j+b,i'+a,j'+b}(y,z),
		\end{split}
	\end{equation}
	\item For every $h = 1,2, \ldots L $, every $i,i' \in [d_1]$ and $j,j' \in [d_2]$, define $\dot{\Gamma}^{(h)}(y,z) \in \RR^{d_1 \times d_2 \times d_1 \times d_2}$ as:
	\begin{equation}\label{eq:dp-cntk-derivative-covar}
		\dot{\Gamma}^{(h)}_{i,j,i',j'}(y,z) := \frac{1}{q^2 \cdot \EE_{w\sim \mathcal{N}(0,1)} \left[ |\sigma(w)|^2 \right]} \cdot \EE_{(u,v) \sim \mathcal{N}\left( 0, \Lambda^{(h)}_{i,j,i',j'}(y,z) \right)} \left[ \dot{\sigma}(u) \cdot \dot{\sigma}(v) \right].
	\end{equation}
	\item Let $\Pi^{(0)}(y,z) := 0$ and for every $h = 1,2, \ldots, L-1$, every $i,i' \in [d_1]$ and $j,j' \in [d_2]$, define $\Pi^{(h)}: \RR^{d_1\times d_2 \times c} \times \RR^{d_1\times d_2 \times c} \to \RR^{d_1\times d_2\times d_1 \times d_2}$ recursively as:
	\begin{equation}\label{eq:dp-cntk-pi}
		\begin{split}
			&\Pi^{(h)}_{i,j,i',j'}(y,z) := \sum_{a=-\frac{q-1}{2}}^{\frac{q-1}{2}} \sum_{b=-\frac{q-1}{2}}^{\frac{q-1}{2}}  \left[\Pi^{(h-1)}(y,z) \odot \dot{\Gamma}^{(h)}(y,z) + \Gamma^{(h)}(y,z)\right]_{i+a,j+b,i'+a,j'+b},\\
			&\text{also let } \Pi^{(L)}(y,z) := \Pi^{(L-1)}(y,z) \odot \dot{\Gamma}^{(L)}(y,z).
		\end{split}
	\end{equation}
	\item The final CNTK expressions is defined as:
	\begin{equation}\label{eq:dp-cntk-gap-def}
		\Theta_{cntk}^{(L)}(y,z) := \frac{1}{d_1^2d_2^2} \cdot \sum_{i , i' \in [d_1]} \sum_{j , j' \in [d_2]} \Pi_{i,j,i',j'}^{(L)}(y,z).
	\end{equation}
\end{enumerate}

In what follows we prove that the procedure in Definition~\ref{relu-cntk-def} precisely computes the CNTK kernel function corresponding to ReLU activation and additionally, we present useful corollaries and consequences of this fact.

\begin{lemma}\label{thm:cntk-relu}
	For every positive integers $d_1,d_2, c$, odd integer $q$, and every integer $h \ge 0$, if the activation function is ReLU $\sigma(\alpha) = \max(\alpha,0)$, then the tensor covariances $\Gamma^{(h)} , \dot{\Gamma}^{(h)}(y,z): \RR^{d_1\times d_2 \times c} \times \RR^{d_1 \times d_2 \times c} \to \RR^{d_1 \times d_2 \times d_1 \times d_2}$ defined in \eqref{eq:dp-cntk-covar} and \eqref{eq:dp-cntk-derivative-covar}, are precisely equal to the tensor covariances defined in \eqref{eq:dp-cntk-covar-simplified} and \eqref{eq:dp-cntk-derivative-covar-simplified} of Definition~\ref{relu-cntk-def}, respectively.
\end{lemma}

\begin{proof}
	To prove the lemma, we first show by induction on $h=1,2, \ldots$ that $N^{(h)}_{i,j}(x) \equiv \Sigma^{(h-1)}_{i,j,i,j}(x,x)$ for every $x\in \RR^{d_1\times d_2 \times c}$ and every $i \in [d_1]$ and $j \in [d_2]$, where $\Sigma^{(h-1)}(x,x)$ is defined as per \eqref{eq:dp-cntk-zero} and \eqref{eq:dp-cntk-covar}. The {\bf base of induction} trivially holds for $h=1$ because by definition of $N^{(1)}(x)$ and \eqref{eq:dp-cntk-zero} we have,
	\[ N^{(1)}_{i,j}(x) = \sum_{a=-\frac{q-1}{2}}^{\frac{q-1}{2}} \sum_{b=-\frac{q-1}{2}}^{\frac{q-1}{2}} \sum_{l=1}^c \left| x_{i+a,j+b,l} \right|^2 \equiv \Sigma^{(0)}_{i,j,i,j}(x,x). \]
	To prove the {\bf inductive step}, suppose that the inductive hypothesis $N^{(h-1)}_{i,j}(x) = \Sigma^{(h-2)}_{i,j,i,j}(x,x)$ holds for some $h\ge2$. Now we show that conditioned on the inductive hypothesis, the inductive claim holds. By \eqref{eq:dp-cntk-covar}, we have,
	\begin{align*}
		\Sigma^{(h-1)}_{i,j,i,j}(x,x) &= \sum_{a=-\frac{q-1}{2}}^{\frac{q-1}{2}} \sum_{b=-\frac{q-1}{2}}^{\frac{q-1}{2}}  \Gamma^{(h-1)}_{i+a,j+b,i+a,j+b}(x,x)\\
		&= \sum_{a=-\frac{q-1}{2}}^{\frac{q-1}{2}} \sum_{b=-\frac{q-1}{2}}^{\frac{q-1}{2}} \frac{\EE_{(u,v) \sim \mathcal{N}\left( 0, \Lambda^{(h-1)}_{i+a,j+b,i+a,j+b}(x,x) \right)} \left[ \sigma(u) \cdot \sigma(v) \right]}{q^2 \cdot \EE_{w\sim \mathcal{N}(0,1)} \left[ |\sigma(w)|^2 \right]} \\
		&= \sum_{a=-\frac{q-1}{2}}^{\frac{q-1}{2}} \sum_{b=-\frac{q-1}{2}}^{\frac{q-1}{2}} \frac{\EE_{u \sim \mathcal{N}\left( 0, \Sigma_{i+a,j+b,i+a,j+b}^{(h-2)}(x,x) \right)} \left[ |\max(0,u)|^2 \right]}{q^2 \cdot \EE_{w\sim \mathcal{N}(0,1)} \left[ |\max(0,w)|^2 \right]} \\
		&= \sum_{a=-\frac{q-1}{2}}^{\frac{q-1}{2}} \sum_{b=-\frac{q-1}{2}}^{\frac{q-1}{2}} \frac{1}{q^2} \cdot \Sigma^{(h-2)}_{i+a,j+b,i+a,j+b}(x,x)\\
		&= \sum_{a=-\frac{q-1}{2}}^{\frac{q-1}{2}} \sum_{b=-\frac{q-1}{2}}^{\frac{q-1}{2}} \frac{1}{q^2} \cdot N^{(h-1)}_{i+a,j+b}(x) \equiv N^{(h)}_{i,j}(x). \text{~~~~~~~~~~~~~~~~~~~~~~~~~~~~~~~~~~~~~~~~(by \eqref{eq:dp-cntk-norm-simplified})}
	\end{align*}
	Therefore, this proves that $N^{(h)}_{i,j}(x) \equiv \Sigma^{(h-1)}_{i,j,i,j}(x,x)$ for every $x$ and every integer $h\ge 1$.
	
	Now, note that the $2 \times 2$ covariance matrix $\Lambda^{(h)}_{i,j,i',j'}(y,z)$, defined in \eqref{eq:dp-cntk-covar}, can be decomposed as $\Lambda^{(h)}_{i,j,i',j'}(y,z) = \begin{pmatrix} f^\top \\ g^\top \end{pmatrix} \cdot \begin{pmatrix}
		f & g \end{pmatrix}$, where $f,g \in \RR^2$.
	Also note that $\|f\|_2^2 = \Sigma^{(h-1)}_{i,j,i,j}(y,y)$ and $\|g\|_2^2 = \Sigma^{(h-1)}_{i',j',i',j'}(z,z)$, hence, by what we proved above, we have,
	\[ \|f\|_2^2 = N^{(h)}_{i,j}(y), \text{ and } \|g\|_2^2 = N^{(h)}_{i',j'}(z). \]
	Therefore, by Claim~\ref{relu-covariance}, we can write:
	\begin{align*}
		\Gamma^{(h)}_{i,j,i',j'}(y,z) &= \frac{1}{q^2 \cdot \EE_{w\sim \mathcal{N}(0,1)} \left[ |\sigma(w)|^2 \right]} \cdot \EE_{(u,v) \sim \mathcal{N}\left( 0, \Lambda^{(h)}_{i,j,i',j'}(y,z) \right)} \left[ \sigma(u) \cdot \sigma(v) \right]\\
		&= \frac{1}{q^2 \cdot \EE_{w\sim \mathcal{N}(0,1)} \left[ |\sigma(w)|^2 \right]} \cdot \EE_{u \sim \mathcal{N}\left( 0, I_d \right)} \left[ \sigma(u^\top f) \cdot \sigma(u^\top g) \right]\\
		&= \frac{2 \cdot \|f\|_2 \cdot \|g\|_2}{q^2 \cdot \kappa_1(1)} \cdot \frac{1}{2} \cdot \kappa_1\left(\frac{\langle f, g \rangle}{\|f\|_2 \cdot \|g\|_2}\right)\\
		&= \frac{\sqrt{N^{(h)}_{i,j}(y) \cdot N^{(h)}_{i',j'}(z)}}{q^2} \cdot \kappa_1\left(\frac{\Sigma^{(h-1)}_{i,j,i',j'}(y,z)}{\sqrt{N^{(h)}_{i,j}(y) \cdot N^{(h)}_{i',j'}(z)}}\right)\\
		&= \frac{\sqrt{N^{(h)}_{i,j}(y) \cdot N^{(h)}_{i',j'}(z)}}{q^2} \cdot \kappa_1\left( \frac{\sum_{a=-\frac{q-1}{2}}^{\frac{q-1}{2}} \sum_{b=-\frac{q-1}{2}}^{\frac{q-1}{2}}  \Gamma^{(h-1)}_{i+a,j+b,i'+a,j'+b}(y,z)}{\sqrt{N^{(h)}_{i,j}(y) \cdot N^{(h)}_{i',j'}(z)}} \right),
	\end{align*}
	where the third line follows from Claim~\ref{relu-covariance} and fourth line follows because we have $\langle f, g \rangle = \Sigma^{(h-1)}_{i,j,i',j'}(y,z)$. The fifth line above follows from \eqref{eq:dp-cntk-covar}. This proves the equivalence between the tensor covariance defined in \eqref{eq:dp-cntk-covar} and the one defined in \eqref{eq:dp-cntk-covar-simplified} of Definition~\ref{relu-cntk-def}. Similarly, by using Claim~\ref{relu-covariance}, we can prove the statement of the lemma about $\dot{\Gamma}^{(h)}_{i,j,i',j'}(y,z)$ as follows,
	\begin{align*}
		\dot{\Gamma}^{(h)}_{i,j,i',j'}(y,z) &= \frac{1}{q^2 \cdot \EE_{w\sim \mathcal{N}(0,1)} \left[ |\sigma(w)|^2 \right]} \cdot \EE_{(u,v) \sim \mathcal{N}\left( 0, \Lambda^{(h)}_{i,j,i',j'}(y,z) \right)} \left[ \dot{\sigma}(u) \cdot \dot{\sigma}(v) \right]\\
		&= \frac{1}{q^2 \cdot \EE_{w\sim \mathcal{N}(0,1)} \left[ |\sigma(w)|^2 \right]} \cdot \EE_{u \sim \mathcal{N}\left( 0, I_d \right)} \left[ \dot{\sigma}(u^\top f) \cdot \dot{\sigma}(u^\top g) \right]\\
		&= \frac{2}{q^2 \cdot \kappa_1(1)} \cdot \frac{1}{2} \cdot \kappa_0\left(\frac{\langle f, g \rangle}{\|f\|_2 \cdot \|g\|_2}\right)\\
		&= \frac{1}{q^2} \cdot \kappa_0\left(\frac{\Sigma^{(h-1)}_{i,j,i',j'}(y,z)}{\sqrt{N^{(h)}_{i,j}(y) \cdot N^{(h)}_{i',j'}(z)}}\right)\\
		&= \frac{1}{q^2} \cdot \kappa_0\left( \frac{\sum_{a=-\frac{q-1}{2}}^{\frac{q-1}{2}} \sum_{b=-\frac{q-1}{2}}^{\frac{q-1}{2}}  \Gamma^{(h-1)}_{i+a,j+b,i'+a,j'+b}(y,z)}{\sqrt{N^{(h)}_{i,j}(y) \cdot N^{(h)}_{i',j'}(z)}} \right),
	\end{align*}
\end{proof}

\begin{cor}[Consequence of Lemma~\ref{thm:cntk-relu}]\label{N-to-Sigma}
	Consider the preconditions of Lemma~\ref{thm:cntk-relu}. For every $x \in \RR^{d_1\times d_2 \times c}$, $N^{(h)}_{i,j}(x) \equiv \sum_{a=-\frac{q-1}{2}}^{\frac{q-1}{2}} \sum_{b=-\frac{q-1}{2}}^{\frac{q-1}{2}}  \Gamma^{(h-1)}_{i+a,j+b,i+a,j+b}(x,x)$.
\end{cor}

We describes some of the basic properties of the function $\Gamma^{(h)}(y,z)$ defined in \eqref{eq:dp-cntk-covar-simplified} in the following lemma,
\begin{lemma}[Properties of $\Gamma^{(h)}(y,z)$]\label{properties-gamma}
	For every images $y,z \in \RR^{d_1 \times d_2 \times c}$, every integer $h \ge 0$ and every  $i,i' \in [d_1]$ and $j,j' \in [d_2]$ the following properties are satisfied by functions $\Gamma^{(h)}$ and $N^{(h)}$ defined in \eqref{eq:dp-cntk-covar-simplified} and \eqref{eq:dp-cntk-norm-simplified} of Definition~\ref{relu-cntk-def}:
	\begin{enumerate}
		\item {\bf Cauchy–Schwarz inequality:} $\left| \Gamma_{i,j,i',j'}^{(h)}(y,z) \right| \le \frac{\sqrt{N_{i,j}^{(h)}(y) \cdot N_{i',j'}^{(h)}(z)}}{q^2}$.
		\item {\bf Norm value:} $\Gamma_{i,j,i,j}^{(h)}(y,y) = \frac{N_{i,j}^{(h)}(y)}{q^2} \ge 0$.
	\end{enumerate}
\end{lemma}

\begin{proof}
	We prove the lemma by induction on $h$. The {\bf base of induction} corresponds to $h=0$. In the base case, by \eqref{eq:dp-cntk-norm-simplified} and \eqref{eq:dp-cntk-covar-simplified} and Cauchy–Schwarz inequality, we have
	\begin{align*}
		\left| \Gamma_{i,j,i',j'}^{(0)}(y,z) \right| &\equiv \left| \sum_{l=1}^c y_{i,j,l} \cdot z_{i',j',l} \right|\\
		&\le \sqrt{\sum_{l=1}^c |y_{i,j,l}|^2 \cdot \sum_{l=1}^c |z_{i',j',l}|^2}\\
		&= \frac{\sqrt{N_{i,j}^{(0)}(y) \cdot N_{i',j'}^{(0)}(z)}}{q^2}.
	\end{align*}
	This proves the base for the first statement. Additionally we have, $\Gamma_{i,j,i,j}^{(0)}(y,y) = \sum_{l=1}^c y_{i,j,l}^2 = \frac{N^{(0)}_{i,j}(y)}{q^2} \ge 0$ which proves the base for the second statement of the lemma.
	Now, in order to prove the inductive step, suppose that statements of the lemma hold for $h-1$, where $h\ge 1$. Then, conditioned on this, we prove that the lemma holds for $h$. First note that by conditioning on the inductive hypothesis, applying Cauchy–Schwarz inequality, and using the definition of $N^{(h)}$ in \eqref{eq:dp-cntk-norm-simplified}, we can write
	\[ \begin{split}
		\left| \frac{\sum_{a=-\frac{q-1}{2}}^{\frac{q-1}{2}} \sum_{b=-\frac{q-1}{2}}^{\frac{q-1}{2}}  \Gamma^{(h-1)}_{i+a,j+b,i'+a,j'+b}(y,z)}{\sqrt{N^{(h)}_{i,j}(y) \cdot N^{(h)}_{i',j'}(z)}} \right| &\le \frac{\sum_{a=-\frac{q-1}{2}}^{\frac{q-1}{2}} \sum_{b=-\frac{q-1}{2}}^{\frac{q-1}{2}}  \sqrt{\frac{N^{(h-1)}_{i+a,j+b}(y)}{q^2} \cdot \frac{N^{(h-1)}_{i'+a,j'+b}(z)}{q^2}}}{\sqrt{N^{(h)}_{i,j}(y) \cdot N^{(h)}_{i',j'}(z)}}\\ 
		&\le 1.
	\end{split}
	\]
	Thus, by monotonicity of function $\kappa_1:[-1,1] \to \RR$, we can write,
	\begin{align*}
		\left| \Gamma_{i,j,i',j'}^{(h)}(y,z) \right| &\equiv \frac{\sqrt{N^{(h)}_{i,j}(y) \cdot N^{(h)}_{i',j'}(z)}}{q^2} \cdot \kappa_1\left( \frac{\sum_{a=-\frac{q-1}{2}}^{\frac{q-1}{2}} \sum_{b=-\frac{q-1}{2}}^{\frac{q-1}{2}}  \Gamma^{(h-1)}_{i+a,j+b,i'+a,j'+b}(y,z)}{\sqrt{N^{(h)}_{i,j}(y) \cdot N^{(h)}_{i',j'}(z)}} \right)\\
		& \le \frac{\sqrt{N^{(h)}_{i,j}(y) \cdot N^{(h)}_{i',j'}(z)}}{q^2} \cdot \kappa_1(1)\\
		&= \frac{\sqrt{N^{(h)}_{i,j}(y) \cdot N^{(h)}_{i',j'}(z)}}{q^2},
	\end{align*}
	where the second line above follows because of the fact that $\kappa_1(\cdot)$ is a monotonically increasing function. This completes the inductive step for the first statement of lemma. Now we prove the inductive step for the second statement as follows,
	\begin{align*}
		\Gamma_{i,j,i,j}^{(h)}(y,y) &\equiv \frac{N^{(h)}_{i,j}(y)}{q^2} \cdot \kappa_1\left( \frac{\sum_{a=-\frac{q-1}{2}}^{\frac{q-1}{2}} \sum_{b=-\frac{q-1}{2}}^{\frac{q-1}{2}}  \Gamma^{(h-1)}_{i+a,j+b,i+a,j+b}(y,y)}{N^{(h)}_{i,j}(y)} \right)\\
		&= \frac{N^{(h)}_{i,j}(y)}{q^2} \cdot \kappa_1(1)\\
		&= \frac{N^{(h)}_{i,j}(y)}{q^2} \ge 0,
	\end{align*}
	where we used Corollary~\ref{N-to-Sigma} to conclude that $\sum_{a=-\frac{q-1}{2}}^{\frac{q-1}{2}} \sum_{b=-\frac{q-1}{2}}^{\frac{q-1}{2}}  \Gamma^{(h-1)}_{i+a,j+b,i+a,j+b}(y,y) =N^{(h)}_{i,j}(y)$ and then used the fact that $N^{(h)}_{i,j}(y)$ is non-negative. This completes the inductive proof of the lemma.
\end{proof}
We also describes some of the basic properties of the function $\dot{\Gamma}^{(h)}(y,z)$ defined in \eqref{eq:dp-cntk-derivative-covar-simplified} in the following lemma,
\begin{lemma}[Properties of $\dot{\Gamma}^{(h)}(y,z)$]\label{properties-gamma-dot}
	For every images $y,z \in \RR^{d_1 \times d_2 \times c}$, every integer $h \ge 0$ and every $i,i' \in [d_1]$ and $j,j' \in [d_2]$ the following properties are satisfied by function $\dot{\Gamma}^{(h)}$ defined in \eqref{eq:dp-cntk-derivative-covar-simplified} of Definition~\ref{relu-cntk-def}:
	\begin{enumerate}
		\item {\bf Cauchy–Schwarz inequality:} $\left| \dot{\Gamma}_{i,j,i',j'}^{(h)}(y,z) \right| \le \frac{1}{q^2}$.
		\item {\bf Norm value:} $\dot{\Gamma}_{i,j,i,j}^{(h)}(y,y) = \frac{1}{q^2} \ge 0$.
	\end{enumerate}
\end{lemma}
\begin{proof}
	First, note that by Lemma~\ref{properties-gamma} and the definition of $N^{(h)}$ in \eqref{eq:dp-cntk-norm-simplified} we have,
	\[ \begin{split}
		\left| \frac{\sum_{a=-\frac{q-1}{2}}^{\frac{q-1}{2}} \sum_{b=-\frac{q-1}{2}}^{\frac{q-1}{2}}  \Gamma^{(h-1)}_{i+a,j+b,i'+a,j'+b}(y,z)}{\sqrt{N^{(h)}_{i,j}(y) \cdot N^{(h)}_{i',j'}(z)}} \right| &\le \frac{\sum_{a=-\frac{q-1}{2}}^{\frac{q-1}{2}} \sum_{b=-\frac{q-1}{2}}^{\frac{q-1}{2}}  \sqrt{\frac{N^{(h-1)}_{i+a,j+b}(y)}{q^2} \cdot \frac{N^{(h-1)}_{i'+a,j'+b}(z)}{q^2}}}{\sqrt{N^{(h)}_{i,j}(y) \cdot N^{(h)}_{i',j'}(z)}}\\ 
		&\le 1.
	\end{split}
	\]
	Thus, by monotonicity of function $\kappa_0:[-1,1] \to \RR$ and using \eqref{eq:dp-cntk-derivative-covar-simplified}, we can write, $ \dot{\Gamma}_{i,j,i',j'}^{(h)}(y,z) \le \frac{1}{q^2} \cdot \kappa_0(1) = \frac{1}{q^2}$. Moreover, the eqality is achieved when $y=z$ and $i=i'$ and $j=j'$. This proves both statements of the lemma.
\end{proof}

We also need to use some properties of $\Pi^{(h)}( y , z )$ defined in \eqref{eq:dp-cntk} and \eqref{eq:dp-cntk-last-layer}. We present these propertied in the next lemma,
\begin{lemma}[Properties of $\Pi^{(h)}$]\label{prop-pi}
	For every images $y,z \in \RR^{d_1 \times d_2 \times c}$, every integer $h \ge 0$ and every $i \in [d_1]$ and $j \in [d_2]$ the following properties are satisfied by the function $\Pi^{(h)}$ defined in \eqref{eq:dp-cntk} and \eqref{eq:dp-cntk-last-layer} of Definition~\ref{relu-cntk-def}:
	\begin{enumerate}
		\item {\bf Cauchy–Schwarz inequality:} $\Pi_{i,j,i',j'}^{(h)}(y,z) \le \sqrt{\Pi_{i,j,i,j}^{(h)}(y,y) \cdot \Pi_{i',j',i',j'}^{(h)}(z,z)}$.
		\item {\bf Norm value:} $\Pi_{i,j,i,j}^{(h)}(y,y) = \begin{cases}
			h \cdot N_{i,j}^{(h+1)}(y) & \text{if } h < L\\
			\frac{L-1}{q^2} \cdot N_{i,j}^{(L)}(y) & \text{if } h = L
		\end{cases}$.
	\end{enumerate}
\end{lemma}
\begin{proof}
	The proof is by induction on $h$. The base of induction corresponds to $h=0$. By definition of $\Pi_{i,j,i,j}^{(0)} \equiv 0$ in \eqref{eq:dp-cntk}, the base of induction for both statements of the lemma follow immediately. 
	
	Now we prove the inductive hypothesis. Suppose that the lemma statement holds for $h-1$. We prove that conditioned on this, the statements of the lemma hold for $h$. There are two cases. The first case corresponds to $h<L$. In this case, by definition of $\Pi_{i,j,i,j}^{(h)}(x,x)$ in \eqref{eq:dp-cntk} and using Lemma~\ref{properties-gamma} and Lemma~\ref{properties-gamma-dot} we can write,
	\small
	\begin{align*}
		\left| \Pi^{(h)}_{i,j,i',j'}(y,z) \right| &\equiv \left| \sum_{a=-\frac{q-1}{2}}^{\frac{q-1}{2}} \sum_{b=-\frac{q-1}{2}}^{\frac{q-1}{2}}  \left[\Pi^{(h-1)}(y,z) \odot \dot{\Gamma}^{(h)}(y,z) + \Gamma^{(h)}(y,z)\right]_{i+a,j+b,i'+a,j'+b} \right| \\
		&\le \sum_{a=-\frac{q-1}{2}}^{\frac{q-1}{2}} \sum_{b=-\frac{q-1}{2}}^{\frac{q-1}{2}}  \frac{\sqrt{\Pi^{(h-1)}_{i+a,j+b,i+a,j+b}(y,y) \cdot \Pi^{(h-1)}_{i'+a,j'+b,i'+a,j'+b}(z,z)}}{q^2} + \frac{\sqrt{N^{(h)}_{i+a,j+b}(y) \cdot N^{(h)}_{i'+a,j'+b}(z)}}{q^2}\\
		&\le \sum_{a=-\frac{q-1}{2}}^{\frac{q-1}{2}} \sum_{b=-\frac{q-1}{2}}^{\frac{q-1}{2}}  \sqrt{\frac{\Pi^{(h-1)}_{i+a,j+b,i+a,j+b}(y,y) + N^{(h)}_{i+a,j+b}(y)}{q^2}} \cdot \sqrt{\frac{\Pi^{(h-1)}_{i'+a,j'+b,i'+a,j'+b}(z,z) + N^{(h)}_{i'+a,j'+b}(z)}{q^2}}\\
		&\le \sqrt{\Pi^{(h)}_{i,j,i,j}(y,y)} \cdot \sqrt{ \Pi^{(h)}_{i',j',i',j'}(z,z)},
	\end{align*}
	\normalsize
	where the second line above follows from inductive hypothesis along with Lemma~\ref{properties-gamma} and Lemma~\ref{properties-gamma-dot}. The third and fourth lines above follow by Cauchy–Schwarz inequality.
	The second case corresponds to $h=L$. In this case, by definition of $\Pi_{i,j,i',j'}^{(L)}(y,z)$ in \eqref{eq:dp-cntk-last-layer} and using Lemma~\ref{properties-gamma} and Lemma~\ref{properties-gamma-dot} along with the inductive hypothesis we can write,
	\small
	\begin{align*}
		\left| \Pi^{(L)}_{i,j,i',j'}(y,z) \right| &\equiv \left| \Pi^{(L-1)}_{i,j,i',j'}(y,z) \cdot \dot{\Gamma}^{(L)}_{i,j,i',j'}(y,z) \right| \\
		&\le \frac{\sqrt{\Pi^{(L-1)}_{i,j,i,j}(y,y) \cdot \Pi^{(L-1)}_{i',j',i',j'}(z,z)}}{q^2} \\
		&= \sqrt{\Pi^{(L)}_{i,j,i,j}(y,y)} \cdot \sqrt{ \Pi^{(L)}_{i',j',i',j'}(z,z)},
	\end{align*}
	\normalsize
	where the second line above follows from inductive hypothesis along with Lemma~\ref{properties-gamma-dot}.
	This completes the inductive step and in turn proves the first statement of the lemma.
	
	To prove the inductive step for the second statement of lemma we consider two cases again. The first case is $h<L$. In this case, note that by using inductive hypothesis together with Lemma~\ref{properties-gamma} and Lemma~\ref{properties-gamma-dot} we can write,
	\begin{align*}
		\Pi^{(h)}_{i,j,i,j}(y,y) &\equiv \sum_{a=-\frac{q-1}{2}}^{\frac{q-1}{2}} \sum_{b=-\frac{q-1}{2}}^{\frac{q-1}{2}}  \left[\Pi^{(h-1)}(y,y) \odot \dot{\Gamma}^{(h)}(y,y) + \Gamma^{(h)}(y,y)\right]_{i+a,j+b,i+a,j+b}\\
		&= \sum_{a=-\frac{q-1}{2}}^{\frac{q-1}{2}} \sum_{b=-\frac{q-1}{2}}^{\frac{q-1}{2}} \frac{(h-1) \cdot N_{i+a,j+b}^{(h)}(y)}{q^2} + \frac{ N_{i+a,j+b}^{(h)}(y)}{q^2}\\
		&= h \cdot \sum_{a=-\frac{q-1}{2}}^{\frac{q-1}{2}} \sum_{b=-\frac{q-1}{2}}^{\frac{q-1}{2}}  \frac{N^{(h)}_{i+a,j+b}(y)}{q^2}\\
		&= h \cdot N^{(h+1)}_{i,j}(y),
	\end{align*}
	where the last line above follows from definition of $N^{(h)}$ in \eqref{eq:dp-cntk-norm-simplified}. 
	The second case corresponds to $h=L$. In this case, by inductive hypothesis together with Lemma~\ref{properties-gamma} and Lemma~\ref{properties-gamma-dot} we can write,
	\begin{align*}
		\Pi^{(L)}_{i,j,i,j}(y,y) &\equiv \Pi^{(L-1)}_{i,j,i,j}(y,y) \cdot \dot{\Gamma}^{(L)}_{i,j,i,j}(y,y)\\
		&= \frac{(L-1) \cdot N_{i,j}^{(L)}(y)}{q^2}.
	\end{align*}
	This completes the inductive step for the second statement and in turn proves the second statement of the lemma.
	
\end{proof}

\section{CNTK Sketch: Claims and Invariants}\label{app-cntk-sketch}
In this section we prove our main result on sketcing the CNTK kernel, i.e., Theorem~\ref{maintheorem-cntk}.
In the following lemma, we analyze the correctness of the CNTK Sketch method by giving the invariants that the algorithm maintaines at all times,

\begin{lemma}[Invariants of the CNTK Sketch]
	\label{lem:cntk-sketch-corr}
	For every positive integers $d_1, d_2, c$, and $L$, every $\epsilon, \delta>0$, every images $y,z \in \RR^{d_1\times d_2 \times c}$ %such that $\max_{i,j,l}\left|y_{(i,j,l)}\right| = O(1)$ and $\max_{i,j,l}\left|z_{(i,j,l)}\right| = O(1)$
	, if we let $N^{(h)}: \RR^{d_1\times d_2\times c}\to \RR^{d_1\times d_2}$, $\Gamma^{(h)}(y,z) \in \RR^{d_1\times d_2 \times d_1\times d_2}$ and $\Pi^{(h)}(y,z) \in \RR^{d_1\times d_2 \times d_1\times d_2}$ be the tensor functions defined in \eqref{eq:dp-cntk-norm-simplified}, \eqref{eq:dp-cntk-covar-simplified}, \eqref{eq:dp-cntk}, and \eqref{eq:dp-cntk-last-layer} of Definition~\ref{relu-cntk-def}, respectively, then with probability at least $1-\delta$ the following invariants are maintained simultaneously for all $i,i' \in [d_1]$ and $j,j' \in [d_2]$ and every $h =0, 1, 2, \ldots L$:
	\begin{enumerate}
		\item The mapping $\phi_{i,j}^{(h)}(\cdot)$ computed by the CNTK Sketch algorithm in \eqref{cntk-sketch-covar-zero} and \eqref{eq:maping-cntk-covar} of Definition~\ref{alg-def-cntk-sketch} satisfy the following,
		\[ \left| \left< \phi_{i,j}^{(h)}(y), \phi_{i',j'}^{(h)}(z) \right> - \Gamma_{i,j,i',j'}^{(h)}\left( y , z \right) \right| \le ({h+1}) \cdot \frac{\epsilon^2}{60L^3}\cdot \frac{\sqrt{N_{i,j}^{(h)}(y) \cdot N_{i',j'}^{(h)}(z)}}{q^2}. \]
		\item The mapping $\psi_{i,j}^{(h)}(\cdot)$ computed by the CNTK Sketch algorithm in \eqref{psi-cntk} and \eqref{psi-cntk-last} of Definition~\ref{alg-def-cntk-sketch} satisfy the following,
		\[ \left| \left< \psi_{i,j}^{(h)}(y), \psi_{i',j'}^{(h)}(z) \right> - \Pi_{i,j,i',j'}^{(h)}\left( y , z \right) \right| \le \begin{cases}
			\frac{\epsilon}{10} \cdot \frac{h^2}{L+1} \cdot \sqrt{N^{(h+1)}_{i,j}(y) \cdot N^{(h+1)}_{i',j'}(z)} &\text{if } h<L\\
			\frac{\epsilon}{10} \cdot \frac{L-1}{q^2} \cdot \sqrt{N^{(L)}_{i,j}(y) \cdot N^{(L)}_{i',j'}(z)} &\text{if } h=L
		\end{cases}. \]
	\end{enumerate}
\end{lemma}

\begin{proof}
	The proof is by induction on the value of $h=0,1,2, \ldots L$. 
	More formally, consider the following statements for every $h=0,1,2, \ldots L$:
	\begin{enumerate}[leftmargin=1.5cm]
		\item[${\bf P_1(h) :}$] Simultaneously for all $i,i' \in [d_1]$ and $j,j' \in [d_2]$:
		\[ \begin{split}
			&\left| \left< \phi_{i,j}^{(h)}(y), \phi_{i',j'}^{(h)}(z) \right> - \Gamma_{i,j,i',j'}^{(h)}\left( y , z \right) \right| \le ({h+1}) \cdot \frac{\epsilon^2}{60L^3} \cdot \frac{\sqrt{N_{i,j}^{(h)}(y) \cdot N_{i',j'}^{(h)}(z)}}{q^2},\\
			&\left| \left\| \phi_{i,j}^{(h)}(y) \right\|_2^2 - \Gamma_{i,j,i,j}^{(h)}\left( y , y \right) \right| \le  \frac{({h+1}) \cdot\epsilon^2}{60L^3} \cdot \frac{N_{i,j}^{(h)}(y)}{q^2}, \\ 
			&\left| \left\| \phi_{i',j'}^{(h)}(z) \right\|_2^2 - \Gamma_{i',j',i',j'}^{(h)}\left( z , z \right) \right| \le  \frac{({h+1}) \cdot\epsilon^2}{60L^3} \cdot \frac{N_{i',j'}^{(h)}(z)}{q^2}.
		\end{split} \]
		\item[${\bf P_2(h) :}$] Simultaneously for all $i,i' \in [d_1]$ and $j,j' \in [d_2]$:
		\[ \begin{split}
			&\left| \left< \psi_{i,j}^{(h)}(y), \psi_{i',j'}^{(h)}(z) \right> - \Pi_{i,j,i',j'}^{(h)}\left( y , z \right) \right| \le \begin{cases}
				\frac{\epsilon}{10} \cdot \frac{h^2}{L+1} \cdot \sqrt{N^{(h+1)}_{i,j}(y) \cdot N^{(h+1)}_{i',j'}(z)} &\text{if } h<L\\
				\frac{\epsilon}{10} \cdot \frac{L-1}{q^2} \cdot \sqrt{N^{(L)}_{i,j}(y) \cdot N^{(L)}_{i',j'}(z)} &\text{if } h=L
			\end{cases},\\
			&{(\text{only for } h<L):} ~~~ \left| \left\| \psi_{i,j}^{(h)}(y)\right\|_2^2 - \Pi_{i,j,i,j}^{(h)}\left( y , y \right) \right| \le \frac{\epsilon}{10} \cdot \frac{h^2}{L+1} \cdot N^{(h+1)}_{i,j}(y),\\
			&{(\text{only for } h<L):} ~~~ \left| \left\| \psi_{i',j'}^{(h)}(z)\right\|_2^2 - \Pi_{i',j',i',j'}^{(h)}\left( z , z \right) \right| \le \frac{\epsilon}{10} \cdot \frac{h^2}{L+1} \cdot N^{(h+1)}_{i',j'}(z).
		\end{split}
		\]
		
	\end{enumerate}
	We prove that probabilities $\Pr[ P_1(0)]$ and $\Pr[ P_2(0)|P_1(0)]$ are both greater than $1 - O(\delta/L)$. Additionally, for every $h = 1,2, \ldots L$, we prove that the conditional probabilities $\Pr[ P_1(h) | P_1(h-1)]$ and $\Pr[ P_2(h) | P_2(h-1), P_1(h), P_1(h-1)]$ are greater than $1 - O(\delta/L)$.
	
	The {\bf base of induction} corresponds to $h=0$. By \eqref{cntk-sketch-covar-zero}, $\phi_{i,j}^{(0)}(y) = S \cdot y_{(i,j,:)}$ and $\phi_{i',j'}^{(0)}(z) = S \cdot z_{(i',j',:)}$, thus, Lemma~\ref{lem:srht} implies the following
	\[ \Pr\left[ \left| \left< \phi_{i,j}^{(0)}(y), \phi_{i',j'}^{(0)}(z)\right> - \left< y_{(i,j,:)}, z_{(i',j',:)} \right> \right| \le O\left(\frac{\epsilon^2}{L^3}\right)\cdot \| y_{(i,j,:)} \|_2 \| z_{(i',j',:)} \|_2 \right] \ge 1 - O\left(\frac{\delta}{d_1^2d_2^2L}\right), \]
	therefore, by using \eqref{eq:dp-cntk-norm-simplified} and \eqref{eq:dp-cntk-covar-simplified} we have
	\[ \Pr\left[ \left| \left< \phi_{i,j}^{(0)}(y), \phi_{i',j'}^{(0)}(z)\right> - \Gamma_{i,j,i',j'}^{(0)}(y,z) \right| \le O\left(\frac{\epsilon^2}{L^3}\right) \cdot \frac{\sqrt{N^{(0)}_{i,j}(y) \cdot N_{i',j'}^{(0)}(z) }}{q^2} \right] \ge 1 - O\left(\frac{\delta}{d_1^2d_2^2L}\right). \]
	Similarly, we can prove that with probability at least $1 - O\left(\frac{\delta}{d_1^2d_2^2L}\right)$, the following hold
	\[ \begin{split}
		&\left| \left\| \phi_{i,j}^{(0)}(y)\right\|_2^2 - \Gamma_{i,j,i,j}^{(0)}\left( y , y \right) \right| \le O\left(\frac{\epsilon^2}{L^3}\right) \cdot \frac{N^{(0)}_{i,j}(y)}{q^2}, \\ 
		&\left| \left\| \phi_{i',j'}^{(0)}(z)\right\|_2^2 - \Gamma_{i',j',i',j'}^{(0)}\left( z , z \right) \right| \le O\left(\frac{\epsilon^2}{L^3}\right) \cdot \frac{N^{(0)}_{i',j'}(z)}{q^2}.
	\end{split}\]
	By union bounding over all $i,i' \in [d_1]$ and $j,j' \in [d_2]$, this proves the base of induction for statement $P_1(h)$, i.e., $\Pr[ P_1(0) ] \ge 1 - O(\delta/L)$.

	Moreover, by \eqref{psi-cntk}, we have that $\psi_{i,j}^{(0)}(y) = 0$ and $\psi_{i',j'}^{(0)}(z) = 0$, thus, by \eqref{eq:dp-cntk}, it trivially holds that $\Pr[P_2(0)|P_1(0)] = 1 \ge 1 - O(\delta/L)$. This completes the base of induction.

	Now, we proceed to prove the {\bf inductive step}. That is, by assuming the inductive hypothesis for $h-1$, we prove that statements $P_1(h)$ and $P_2(h)$ hold. More precisely, first we condition on the statement $P_1(h-1)$ being true for some $h \ge 1$, and then prove that $P_1(h)$ holds with probability at least $1 - O(\delta / L)$. Next we show that conditioned on statements $P_2(h-1), P_1(h), P_1(h-1)$ being true, $P_2(h)$ holds with probability at least $1 - O(\delta / L)$. This will complete the induction.
	
	First, note that by Lemma~\ref{lem:srht}, union bound, and using \eqref{eq:maping-cntk-covar}, the following holds simultaneously for all $i,i' \in [d_1]$ and all $j,j' \in [d_2]$, with probability at least $1 - O\left(\frac{\delta}{L}\right)$,
	\begin{equation}\label{eq:phi-inner-prod-bound1}
		\left| \left< \phi_{i,j}^{(h)}(y), \phi_{i',j'}^{(h)}(z)\right> - \frac{\sqrt{N^{(h)}_{i,j}(y) N^{(h)}_{i',j'}(z)}}{q^2} \cdot \sum_{l=0}^{2p+2} c_l \left<\left[Z^{(h)}_{i,j}(y)\right]_l, \left[Z^{(h)}_{i',j'}(z)\right]_l\right> \right| \le O\left(\frac{\epsilon^2}{L^3}\right) \cdot A,
	\end{equation}
	where $A := \frac{\sqrt{N^{(h)}_{i,j}(y) N^{(h)}_{i',j'}(z)}}{q^2} \cdot \sqrt{\sum_{l=0}^{2p+2} c_l \left\| \left[ Z^{(h)}_{i,j}(y) \right]_l \right\|_2^2} \cdot \sqrt{\sum_{l=0}^{2p+2} c_l \left\| \left[Z^{(h)}_{i',j'}(z)\right]_l \right\|_2^2}$ and the collection of vectors $\left\{\left[Z^{(h)}_{i,j}(y)\right]_l\right\}_{l=0}^{2p+2}$ and $\left\{\left[Z^{(h)}_{i',j'}(z)\right]_l\right\}_{l=0}^{2p+2}$ and coefficients $c_0,c_1, c_2, \ldots c_{2p+2}$ are defined as per \eqref{eq:maping-cntk-covar} and \eqref{eq:poly-approx-krelu}, respectively. 
	Additionally, by Lemma~\ref{soda-result} and union bound, the following inequalities hold, with probability at least $1 - O\left( \frac{\delta}{L} \right)$, simultaneously for all $l = 0,1,2, \ldots 2p+2$, all $i,i' \in [d_1]$ and all $j,j' \in [d_2]$:
	\begin{align}
		&\left|\left<\left[Z^{(h)}_{i,j}(y)\right]_l, \left[Z^{(h)}_{i',j'}(z)\right]_l\right> - \left<\mu_{i,j}^{(h)}(y), \mu_{i',j'}^{(h)}(z)\right>^l \right| \le O\left( \frac{\epsilon^2}{L^3} \right) \left\| \mu_{i,j}^{(h)}(y) \right\|_2^l \left\| \mu_{i',j'}^{(h)}(z)\right\|_2^l \nonumber\\
		&\left\| \left[Z^{(h)}_{i,j}(y)\right]_l \right\|_2^2 \le \frac{11}{10} \cdot \left\| \mu_{i,j}^{(h)}(y) \right\|_2^{2l} \label{eq:Zinner-prod-bound}\\
		& \left\| \left[Z^{(h)}_{i',j'}(z)\right]_l \right\|_2^2 \le \frac{11}{10} \cdot \left\| \mu_{i',j'}^{(h)}(z) \right\|_2^{2l} \nonumber
	\end{align}
	Therefore, by plugging \eqref{eq:Zinner-prod-bound} back to \eqref{eq:phi-inner-prod-bound1} and using union bound and triangle inequality as well as 
	Cauchy–Schwarz inequality, we find that with probability at least $1 - O\left( \frac{\delta}{L} \right)$, the following holds simultaneously for all $i,i' \in [d_1]$ and $j,j' \in [d_2]$
	\begin{equation} \label{eq:phi-inner-prod-bound2}
		\left| \left< \phi_{i,j}^{(h)}(y), \phi_{i',j'}^{(h)}(z)\right> - \frac{\sqrt{N^{(h)}_{i,j}(y) N^{(h)}_{i',j'}(z)}}{q^2} \cdot P^{(p)}_{relu}\left( \left<\mu_{i,j}^{(h)}(y), \mu_{i',j'}^{(h)}(z)\right> \right) \right| \le O\left(\frac{\epsilon^2}{L^3}\right) \cdot B,
	\end{equation}
	where $B:= \frac{\sqrt{N^{(h)}_{i,j}(y) N^{(h)}_{i',j'}(z)}}{q^2} \cdot \sqrt{P^{(p)}_{relu}\left(\|\mu^{(h)}_{i,j}(y)\|_2^2\right) \cdot P^{(p)}_{relu}\left(\|\mu^{(h)}_{i',j'}(z)\|_2^2\right)}$ and $P^{(p)}_{relu}(\alpha) = \sum_{l=0}^{2p+2} c_l \cdot \alpha^l$ is the polynomial defined in \eqref{eq:poly-approx-krelu}. 
	By using the definition of $\mu_{i,j}^{(h)}(\cdot)$ in \eqref{eq:maping-cntk-covar} we have,
	\begin{equation}\label{mu-equality}
		\begin{split}
			&\left<\mu_{i,j}^{(h)}(y), \mu_{i',j'}^{(h)}(z)\right> = \frac{\sum_{a=-\frac{q-1}{2}}^{\frac{q-1}{2}} \sum_{b=-\frac{q-1}{2}}^{\frac{q-1}{2}}  \left< \phi_{i+a,j+b}^{(h-1)}(y), \phi_{i'+a,j'+b}^{(h-1)}(z) \right>}{\sqrt{N^{(h)}_{i,j}(y) N^{(h)}_{i',j'}(z)}},\\
			&\left\|\mu_{i,j}^{(h)}(y)\right\|_2^2 = \frac{\sum_{a=-\frac{q-1}{2}}^{\frac{q-1}{2}} \sum_{b=-\frac{q-1}{2}}^{\frac{q-1}{2}}  \left\| \phi_{i+a,j+b}^{(h-1)}(y) \right\|_2^2}{N^{(h)}_{i,j}(y)},\\
			&\left\|\mu_{i',j'}^{(h)}(z)\right\|_2^2 = \frac{\sum_{a=-\frac{q-1}{2}}^{\frac{q-1}{2}} \sum_{b=-\frac{q-1}{2}}^{\frac{q-1}{2}}  \left\| \phi_{i'+a,j'+b}^{(h-1)}(z) \right\|_2^2}{N^{(h)}_{i,j}(z)}.
		\end{split} 
	\end{equation}
	Hence, by conditioning on the inductive hypothesis $P_1(h-1)$ and using \eqref{mu-equality} and Corollary~\ref{N-to-Sigma} we have, 
	\[
	\left| \left\| \mu_{i,j}^{(h)}(y) \right\|_2^2 - 1 \right| \le h \cdot \frac{\epsilon^2}{60L^3}, \text{ and } \left| \left\| \mu_{i',j'}^{(h)}(z) \right\|_2^2 - 1 \right| \le h \cdot \frac{\epsilon^2}{60L^3}.
	\]
	Therefore, by invoking  Lemma~\ref{lema:sensitivity-polynomial}, it follows that $\left| P_{relu}^{(p)}\left(\|\mu^{(h)}_{i,j}(y)\|_2^2\right) - P_{relu}^{(p)}(1) \right| \le h \cdot \frac{\epsilon^2}{60L^3}$ and $\left| P_{relu}^{(p)}\left(\|\mu^{(h)}_{i',j'}(z)\|_2^2\right) - P_{relu}^{(p)}(1) \right| \le h \cdot \frac{\epsilon^2}{60L^3}$. Consequently, because $P_{relu}^{(p)}(1) \le P_{relu}^{(+\infty)}(1) = 1$, we find that
	\[ B \le \frac{11}{10} \cdot \frac{\sqrt{N^{(h)}_{i,j}(y) N^{(h)}_{i',j'}(z)}}{q^2}.\]
	For shorthand we use the notation $\beta:=\frac{\sqrt{N^{(h)}_{i,j}(y) N^{(h)}_{i',j'}(z)}}{q^2}$. By plugging this into \eqref{eq:phi-inner-prod-bound2} and using the notation $\beta$, we find that the following holds simultaneously for all $i,i' \in [d_1]$ and all $j,j' \in [d_2]$, with probability at least $1 - O\left( \frac{\delta}{ L} \right)$,
	\begin{equation} \label{eq:phi-inner-prod-bound3}
		\left| \left< \phi_{i,j}^{(h)}(y), \phi_{i',j'}^{(h)}(z)\right> - \beta \cdot P^{(p)}_{relu}\left( \left<\mu_{i,j}^{(h)}(y), \mu_{i',j'}^{(h)}(z)\right> \right) \right| \le O\left(\frac{\epsilon^2}{L^3}\right) \cdot \beta.
	\end{equation}

	Furthermore, by conditioning on the inductive hypothesis $P_1(h-1)$ and combining it with \eqref{mu-equality} and applying Cauchy–Schwarz inequality and invoking Corollary~\ref{N-to-Sigma} we find that,
	\small
	\begin{equation}\label{eq:mu-innerprod-bound}
		\begin{split}
			&\left| \left<\mu_{i,j}^{(h)}(y), \mu_{i',j'}^{(h)}(z)\right> - \frac{\sum_{a=-\frac{q-1}{2}}^{\frac{q-1}{2}} \sum_{b=-\frac{q-1}{2}}^{\frac{q-1}{2}}  \Gamma_{i+a,j+b,i'+a,j'+b}^{(h-1)}\left( y , z \right)}{\sqrt{N^{(h)}_{i,j}(y) \cdot N^{(h)}_{i',j'}(z)}} \right|\\ 
			&\qquad \le \frac{\sum_{a=-\frac{q-1}{2}}^{\frac{q-1}{2}} \sum_{b=-\frac{q-1}{2}}^{\frac{q-1}{2}}  \sqrt{N_{i+a,j+b}^{(h-1)}(y) \cdot N_{i'+a,j'+b}^{(h-1)}(z)}}{q^2 \cdot \sqrt{N^{(h)}_{i,j}(y) \cdot N^{(h)}_{i',j'}(z)}} \cdot  \frac{h \cdot \epsilon^2}{60L^3}\\
			&\qquad \le \frac{\sqrt{\sum_{a=-\frac{q-1}{2}}^{\frac{q-1}{2}} \sum_{b=-\frac{q-1}{2}}^{\frac{q-1}{2}} N_{i+a,j+b}^{(h-1)}(y) / q^2 } \cdot \sqrt{\sum_{a=-\frac{q-1}{2}}^{\frac{q-1}{2}} \sum_{b=-\frac{q-1}{2}}^{\frac{q-1}{2}} N_{i'+a,j'+b}^{(h-1)}(z) / q^2}}{ \sqrt{N^{(h)}_{i,j}(y) \cdot N^{(h)}_{i',j'}(z)}} \cdot  \frac{h \cdot \epsilon^2}{60L^3}\\
			&\qquad= h \cdot \frac{\epsilon^2}{60L^3},
		\end{split}
	\end{equation}
	\normalsize
	where the last line follows from \eqref{eq:dp-cntk-norm-simplified}.
	
	For shorthand, we use the notation $\gamma := \frac{\sum_{a=-\frac{q-1}{2}}^{\frac{q-1}{2}} \sum_{b=-\frac{q-1}{2}}^{\frac{q-1}{2}}  \Gamma_{i+a,j+b,i'+a,j'+b}^{(h-1)}\left( y , z \right)}{\sqrt{N^{(h)}_{i,j}(y) \cdot N^{(h)}_{i',j'}(z)}}$. Note that by Lemma~\ref{properties-gamma} and \eqref{eq:dp-cntk-norm-simplified}, $-1 \le \gamma \le 1$. Hence, we can invoke Lemma~\ref{lema:sensitivity-polynomial} and use \eqref{eq:mu-innerprod-bound} to find that,
	\[ \left|P^{(p)}_{relu}\left( \left<\mu_{i,j}^{(h)}(y), \mu_{i',j'}^{(h)}(z)\right> \right) - P^{(p)}_{relu}\left(\gamma\right) \right| \le h \cdot \frac{\epsilon^2}{60L^3}. \]
	By incorporating the above inequality into \eqref{eq:phi-inner-prod-bound3} using triangle inequality we find that, with probability at least $1 - O\left( \frac{\delta}{ L} \right)$, the following holds simultaneously for all $i,i' \in [d_1]$ and all $j,j'\in [d_2]$:
	\begin{equation} \label{eq:phi-inner-prod-bound4}
		\left| \left< \phi_{i,j}^{(h)}(y), \phi_{i',j'}^{(h)}(z)\right> - \beta \cdot P^{(p)}_{relu}\left(\gamma \right) \right| \le \left(O\left(\frac{\epsilon^2}{L^3}\right) + \frac{h \cdot\epsilon^2}{60L^3}\right) \cdot \beta.
	\end{equation}
	Additionally, since $-1 \le \gamma \le 1$, we can invoke Lemma~\ref{lem:polynomi-approx-krelu} and use the fact that $p = \left\lceil 2L^2/{\epsilon}^{4/3} \right\rceil$ to conclude,
	\[ \left| P_{relu}^{(p)}\left(\gamma\right) - \kappa_1(\gamma ) \right| \le \frac{\epsilon^2}{76 L^3}. \]
	By combining the above inequality with \eqref{eq:phi-inner-prod-bound4} via triangle inequality and using the fact that, by \eqref{eq:dp-cntk-covar-simplified}, $\beta \cdot \kappa_1(\gamma) \equiv \Gamma_{i,j,i',j'}^{(h)}(y,z)$ we get the following inequality, with probability at least $1 - O\left( \frac{\delta}{ L} \right)$
	\[  \left| \left< \phi_{i,j}^{(h)}(y), \phi_{i',j'}^{(h)}(z)\right> - \Gamma_{i,j,i',j'}^{(h)}(y,z) \right| \le (h+1) \cdot \frac{\epsilon^2}{60L^3} \cdot \frac{\sqrt{N^{(h)}_{i,j}(y) N^{(h)}_{i',j'}(z)}}{q^2}. \]
	Similarly, we can prove that with probability at least $1 - O\left( \frac{\delta}{ L} \right)$ the following hold, simultaneously for all $i,i' \in [d_1]$ and $j,j' \in [d_2]$,
	\begin{align*}
	&\left| \left\| \phi_{i,j}^{(h)}(y)\right\|_2^2 - \Gamma_{i,j,i,j}^{(h)}(y,y) \right| \le \frac{(h+1) \epsilon^2}{60L^3} \cdot \frac{N^{(h)}_{i,j}(y)}{q^2},\\
	&\left| \left\| \phi_{i',j'}^{(h)}(z)\right\|_2^2 - \Gamma_{i',j',i',j'}^{(h)}(z,z) \right| \le \frac{(h+1) \epsilon^2}{60L^3} \cdot \frac{N^{(h)}_{i',j'}(z)}{q^2}.
	\end{align*}
	This is sufficient to prove the inductive step for statement $P_1(h)$, i.e., $\Pr[P_1(h)|P_1(h-1)] \ge 1 - O(\delta/L)$.

	Now we prove the inductive step for statement $P_2(h)$. That is, we prove that conditioned on $P_2(h-1), P_1(h)$, and $P_1(h-1)$, $P_2(h)$ holds with probability at least $1-O(\delta/L)$.
	First, note that by Lemma~\ref{lem:srht} and using \eqref{eq:cntk-map-phidot} and union bound, we have the following simultaneously for all $i,i' \in [d_1]$ and all $j,j' \in [d_2]$, with probability at least $1 - O\left( \frac{\delta}{L} \right)$,
	\begin{equation}\label{cntk:phi-dot-bound1}
		\left| \left< \dot{\phi}_{i,j}^{(h)}(y), \dot{\phi}_{i',j'}^{(h)}(z)\right> - \frac{1}{q^2}  \sum_{l=0}^{2p'+1} b_l \left<\left[Y^{(h)}_{i,j}(y)\right]_l, \left[Y^{(h)}_{i',j'}(z)\right]_l\right> \right| \le O\left(\frac{\epsilon}{L}\right) \widehat{A},
	\end{equation}
	where $\widehat{A} := \frac{1}{q^2} \cdot \sqrt{\sum_{l=0}^{2p'+1} b_l \left\| \left[Y^{(h)}_{i,j}(y)\right]_l \right\|_2^2} \cdot \sqrt{\sum_{l=0}^{2p'+1} b_l \left\| \left[Y^{(h)}_{i',j'}(z)\right]_l \right\|_2^2}$ and the collection of vectors $\left\{\left[Y^{(h)}_{i,j}(y)\right]_l\right\}_{l=0}^{2p'+1}$ and $\left\{\left[Y^{(h)}_{i',j'}(z)\right]_l\right\}_{l=0}^{2p'+1}$ and coefficients $b_0, b_1, b_2, \ldots b_{2p'+1}$ are defined as per \eqref{eq:cntk-map-phidot} and \eqref{eq:poly-approx-krelu}, respectively. 
	By Lemma~\ref{soda-result} and union bound, with probability at least $1 - O\left( \frac{\delta}{L} \right)$, the following inequalities hold true simultaneously for all $l \in \{0,1,2, \ldots 2p'+1\}$, all $i,i' \in [d_1]$ and all $j,j' \in [d_2]$,
	\begin{align}
		& \left|\left<\left[Y^{(h)}_{i,j}(y)\right]_l, \left[Y^{(h)}_{i',j'}(z)\right]_l\right> - \left<\mu_{i,j}^{(h)}(y), \mu_{i',j'}^{(h)}(z)\right>^l \right| \le O\left( \frac{\epsilon}{L} \right) \cdot \left\| \mu_{i,j}^{(h)}(y) \right\|_2^l \left\| \mu_{i',j'}^{(h)}(z)\right\|_2^l \nonumber\\
		& \left\| \left[Y^{(h)}_{i,j}(y)\right]_l\right\|_2^2 \le \frac{11}{10} \cdot \left\| \mu_{i,j}^{(h)}(y) \right\|_2^{2l} \label{eq:Y-innerprod-bound}\\
		& \left\| \left[Y^{(h)}_{i',j'}(z)\right]_l \right\|_2^2 \le \frac{11}{10} \cdot \left\| \mu_{i',j'}^{(h)}(z) \right\|_2^{2l} \nonumber
	\end{align}
	Therefore, by plugging \eqref{eq:Y-innerprod-bound} into \eqref{cntk:phi-dot-bound1} and using union bound and triangle inequality as well as 
	Cauchy–Schwarz inequality, we find that with probability at least $1 - O\left( \frac{\delta}{L} \right)$, the following holds simultaneously for all $i,i' \in [d_1]$ and $j,j' \in [d_2]$ 
	\begin{equation} \label{cntk:phi-dot-bound2}
		\left| \left< \dot{\phi}_{i,j}^{(h)}(y), \dot{\phi}_{i',j'}^{(h)}(z)\right> - \frac{1}{q^2} \cdot \dot{P}^{(p')}_{relu}\left( \left<\mu_{i,j}^{(h)}(y), \mu_{i',j'}^{(h)}(z)\right> \right) \right| \le O\left(\frac{\epsilon}{L}\right) \cdot \widehat{B},
	\end{equation}
	where $\widehat{B}:= \frac{1}{q^2} \cdot \sqrt{\dot{P}^{(p')}_{relu}\left(\|\mu_{i,j}^{(h)}(y)\|_2^2\right) \cdot \dot{P}^{(p')}_{relu}\left(\|\mu_{i',j'}^{(h)}(z)\|_2^2\right)}$ and $\dot{P}^{(p)}_{relu}(\alpha) = \sum_{l=0}^{2p'+1} b_l \cdot \alpha^l$ is the polynomial defined in \eqref{eq:poly-approx-krelu}.
	By conditioning on the inductive hypothesis $P_1(h-1)$ and using \eqref{mu-equality} and Corollary~\ref{N-to-Sigma} we have $\left| \left\| \mu_{i,j}^{(h)}(y) \right\|_2^2 - 1 \right| \le h \cdot \frac{\epsilon^2}{60L^3}$ and $\left| \left\| \mu_{i',j'}^{(h)}(z) \right\|_2^2 - 1 \right| \le h \cdot \frac{\epsilon^2}{60L^3}$. 
	Therefore, using the fact that $p' = \left\lceil 9L^2 /\epsilon^{2} \right\rceil$ and by invoking Lemma~\ref{lema:sensitivity-polynomial}, it follows that $\left| \dot{P}_{relu}^{(p')}\left(\|\mu^{(h)}_{i,j}(y)\|_2^2\right) - \dot{P}_{relu}^{(p')}(1) \right| \le  \frac{h \cdot\epsilon}{20L^2}$ and $\left| \dot{P}_{relu}^{(p')}\left(\|\mu^{(h)}_{i',j'}(z)\|_2^2\right) - \dot{P}_{relu}^{(p')}(1) \right| \le \frac{ h \cdot \epsilon}{20L^2}$. Consequently, because $\dot{P}_{relu}^{(p')}(1) \le \dot{P}_{relu}^{(+\infty)}(1) = 1$, we find that
	\[ \widehat{B} \le \frac{11}{10\cdot q^2}.\]
	By plugging this into \eqref{cntk:phi-dot-bound2} we get the following, with probability at least $1 - O\left( \frac{\delta}{L} \right)$,
	\begin{equation} \label{cntk:phi-dot-bound3}
		\left| \left< \dot{\phi}_{i,j}^{(h)}(y), \dot{\phi}_{i',j'}^{(h)}(z)\right> - \frac{1}{q^2} \cdot \dot{P}^{(p')}_{relu}\left( \left<\mu_{i,j}^{(h)}(y), \mu_{i',j'}^{(h)}(z)\right> \right) \right| \le O\left(\frac{\epsilon}{q^2\cdot L}\right).
	\end{equation}

	Furthermore, recall the notation $\gamma = \frac{\sum_{a=-\frac{q-1}{2}}^{\frac{q-1}{2}} \sum_{b=-\frac{q-1}{2}}^{\frac{q-1}{2}}  \Gamma_{i+a,j+b,i'+a,j'+b}^{(h-1)}\left( y , z \right)}{\sqrt{N^{(h)}_{i,j}(y) \cdot N^{(h)}_{i',j'}(z)}}$ and note that by Lemma~\ref{properties-gamma} and \eqref{eq:dp-cntk-norm-simplified}, $-1 \le \gamma \le 1$. Hence, we can invoke Lemma~\ref{lema:sensitivity-polynomial} and use the fact that $p' = \lceil 9 L^2 / \epsilon^2 \rceil$ to find that \eqref{eq:mu-innerprod-bound} implies the following,
	\[ \left|\dot{P}^{(p')}_{relu}\left( \left<\mu_{i,j}^{(h)}(y), \mu_{i',j'}^{(h)}(z)\right> \right) - \dot{P}^{(p')}_{relu}\left(\gamma \right) \right| \le \frac{h \cdot \epsilon}{20L^2}. \]
	By incorporating the above inequality into \eqref{cntk:phi-dot-bound3} using triangle inequality, we find that, with probability at least $1 - O\left( \frac{\delta}{ L} \right)$, the following holds simultaneously for all $i,i' \in [d_1]$ and all $j,j'\in [d_2]$:
	\begin{equation} \label{cntk:phi-dot-bound4}
		\left| \left< \dot{\phi}_{i,j}^{(h)}(y), \dot{\phi}_{i',j'}^{(h)}(z)\right> - \frac{1}{q^2} \cdot \dot{P}^{(p')}_{relu}\left(\gamma \right) \right| \le O\left(\frac{\epsilon}{q^2 L^2}\right) + \frac{h}{q^2} \cdot \frac{\epsilon}{20L^2}.
	\end{equation}
	Since $-1 \le \gamma \le 1$, we can invoke Lemma~\ref{lem:polynomi-approx-krelu} and use the fact that $p' = \left\lceil 9L^2 / {\epsilon}^2 \right\rceil$ to conclude,
	\[ \left| \dot{P}_{relu}^{(p')}\left(\gamma\right) - \kappa_0\left(\gamma\right) \right| \le \frac{\epsilon}{15 L}. \]
	By combining the above inequality with \eqref{cntk:phi-dot-bound4} via triangle inequality and using the fact that, by \eqref{eq:dp-cntk-derivative-covar-simplified}, $\frac{1}{q^2} \cdot \kappa_0(\gamma) \equiv \dot{\Gamma}_{i,j,i',j'}^{(h)}(y,z)$ we get the following bound simultaneously for all $i,i' \in [d_1]$ and all $j,j'\in [d_2]$, with probability at least $1 - O\left( \frac{\delta}{ L} \right)$:
	\begin{equation}\label{cntk:phi-dot-bound-final}
		\left| \left< \dot{\phi}_{i,j}^{(h)}(y), \dot{\phi}_{i',j'}^{(h)}(z)\right> - \dot{\Gamma}_{i,j,i',j'}^{(h)}(y,z) \right| \le \frac{1}{q^2} \cdot \frac{\epsilon}{8L}.
	\end{equation}
	Similarly we can prove that with probability at least $1 - O\left( \frac{\delta}{L} \right)$, the following hold simultaneously for all $i,i' \in [d_1]$ and all $j,j'\in [d_2]$,
	\begin{equation}\label{cntk:phi-norm-bound-final}
		\left| \left\| \dot{\phi}_{i,j}^{(h)}(y)\right\|_2^2 - \dot{\Gamma}_{i,j,i,j}^{(h)}(y,y) \right| \le \frac{1}{q^2} \cdot \frac{\epsilon}{8L}, \text{ and } \left| \left\| \dot{\phi}_{i',j'}^{(h)}(z)\right\|_2^2 - \dot{\Gamma}_{i',j',i',j'}^{(h)}(z,z) \right| \le \frac{1}{q^2} \cdot \frac{\epsilon}{8L}.
	\end{equation}
	We will use \eqref{cntk:phi-dot-bound-final} and \eqref{cntk:phi-norm-bound-final} to prove the inductive step for $P_2(h)$.

	Next, we consider two cases for the value of $h$. When $h<L$, the vectors $\psi_{i,j}^{(h)}(y) , \psi_{i',j'}^{(h)}(z)$ are defined in \eqref{psi-cntk} and when $h=L$, these vectors are defined differently in \eqref{psi-cntk-last}. First we consider the case of $h<L$. Note that in this case, if we let $\eta_{i,j}^{(h)}(y)$ and $\eta_{i',j'}^{(h)}(z)$ be the vectors defined in \eqref{psi-cntk}, then by Lemma~\ref{lem:srht} and union bound, the following holds simultaneously for all $i,i'\in [d_1]$ and all $j,j' \in [d_2]$, with probability at least $1 - O\left(\frac{\delta}{L}\right)$:
	\begin{equation}\label{eq:psi-bound1}
		\left| \left< \psi_{i,j}^{(h)}(y) , \psi_{i',j'}^{(h)}(z) \right> - \sum_{a=-\frac{q-1}{2}}^{\frac{q-1}{2}} \sum_{b=-\frac{q-1}{2}}^{\frac{q-1}{2}}  \left< \eta_{i+a,j+b}^{(h)}(y), \eta_{i'+a,j'+b}^{(h)}(z) \right> \right| \le O\left( \frac{\epsilon}{L} \right) \cdot D, 
	\end{equation}
	where $D := \sqrt{\sum_{a=-\frac{q-1}{2}}^{\frac{q-1}{2}} \sum_{b=-\frac{q-1}{2}}^{\frac{q-1}{2}} \|\eta_{i+a,j+b}^{(h)}(y)\|_2^2} \cdot \sqrt{\sum_{a=-\frac{q-1}{2}}^{\frac{q-1}{2}} \sum_{b=-\frac{q-1}{2}}^{\frac{q-1}{2}} \|\eta_{i'+a,j'+b}^{(h)}(z)\|_2^2}$. 
	Now, if we let $f_{i,j} := \psi^{(h-1)}_{i,j}(y) \otimes \dot{\phi}_{i,j}^{(h)}(y)$ and $g_{i',j'} := \psi^{(h-1)}_{i',j'}(z) \otimes \dot{\phi}_{i',j'}^{(h)}(z)$, then by \eqref{psi-cntk}, $\eta_{i,j}^{(h)}(y) = \left(Q^2\cdot f_{i,j}\right) \oplus \phi_{i,j}^{(h)}(y)$ and $\eta_{i',j'}^{(h)}(z) = \left(Q^2\cdot g_{i',j'}\right) \oplus \phi_{i',j'}^{(h)}(z)$. 
	Thus by Lemma~\ref{soda-result} and union bound, with probability at least $1 - O\left( \frac{\delta}{L} \right)$, we have the following inequalities simultaneously for all $i,i' \in [d_1]$ and $j,j' \in [d_2]$:
	\begin{align}
		& \left|\left<\eta^{(h)}_{i,j}(y), \eta^{(h)}_{i',j'}(z)\right> - \langle f_{i,j},g_{i',j'} \rangle - \left<\phi_{i,j}^{(h)}(y), \phi_{i',j'}^{(h)}(z)\right> \right| \le O\left( \frac{\epsilon}{L} \right) \cdot \left\| f_{i,j} \right\|_2 \left\| g_{i',j'} \right\|_2  \nonumber\\
		& \left\| \eta^{(h)}_{i,j}(y)\right\|_2^2 \le \frac{11}{10} \cdot  \|f_{i,j}\|_2^2 + \left\| \phi_{i,j}^{(h)}(y) \right\|_2^2 \label{eta-innerprod-bound}\\
		& \left\| \eta^{(h)}_{i',j'}(z)\right\|_2^2 \le \frac{11}{10} \cdot \|g_{i',j'}\|_2^2 + \left\| \phi_{i',j'}^{(h)}(z) \right\|_2^2 \nonumber
	\end{align}
	Therefore, if we condition on inductive hypotheses $P_1(h)$ and $P_2(h-1)$, then by using Corollary~\ref{N-to-Sigma}, Lemma~\ref{prop-pi}, inequality \eqref{cntk:phi-norm-bound-final} and Lemma~\ref{properties-gamma-dot} along with the fact that $\|f_{i,j}\|_2^2 = \|\psi^{(h-1)}_{i,j}(y)\|_2^2 \cdot \|\dot{\phi}_{i,j}^{(h)}(y)\|_2^2$, we have:
	\small
	\begin{align*}
		&\sum_{a=-\frac{q-1}{2}}^{\frac{q-1}{2}} \sum_{b=-\frac{q-1}{2}}^{\frac{q-1}{2}} \|\eta_{i+a,j+b}^{(h)}(y)\|_2^2\\
		&\qquad\le \sum_{a=-\frac{q-1}{2}}^{\frac{q-1}{2}} \sum_{b=-\frac{q-1}{2}}^{\frac{q-1}{2}} \frac{11}{10} \|f_{i+a,j+b}\|_2^2 + \Gamma_{i+a,j+b,i+a,j+b}^{(h)}( y , y) + \frac{N_{i+a,j+b}^{(h)}(y)}{10q^2}\\
		&\qquad= \frac{11}{10} \cdot \sum_{a=-\frac{q-1}{2}}^{\frac{q-1}{2}} \sum_{b=-\frac{q-1}{2}}^{\frac{q-1}{2}} \|\psi^{(h-1)}_{i+a,j+b}(y)\|_2^2 \cdot \|\dot{\phi}_{i+a,j+b}^{(h)}(y)\|_2^2 + \Gamma_{i+a,j+b,i+a,j+b}^{(h)}( y , y)\\
		&\qquad\le \frac{12}{10} \sum_{a=-\frac{q-1}{2}}^{\frac{q-1}{2}} \sum_{b=-\frac{q-1}{2}}^{\frac{q-1}{2}} \Pi^{(h-1)}_{i+a,j+b, i+a, j+b}(y,y) \cdot \dot{\Gamma}_{i+a,j+b,i+a,j+b}^{(h)}(y,y) + \Gamma_{i+a,j+b,i+a,j+b}^{(h)}( y , y)\\
		&\qquad = \frac{12}{10} \cdot \Pi^{(h)}_{i,j, i, j}(y,y) = \frac{12}{10} \cdot h \cdot N^{(h+1)}_{i,j}(y),
	\end{align*}
	\normalsize
	where the fourth line above follows from the inductive hypothesis $P_2(h-1)$ along with \eqref{cntk:phi-norm-bound-final} and Lemma~\ref{properties-gamma-dot} and Lemma~\ref{prop-pi}. The last line above follows from \eqref{eq:dp-cntk} and Lemma~\ref{prop-pi}.
	Similarly we can prove, $\sum_{a=-\frac{q-1}{2}}^{\frac{q-1}{2}} \sum_{b=-\frac{q-1}{2}}^{\frac{q-1}{2}} \|\eta_{i'+a,j'+b}^{(h)}(z)\|_2^2 \le \frac{12}{10} \cdot h \cdot N^{(h+1)}_{i',j'}(z)$, thus conditioned on $P_2(h-1), P_1(h), P_1(h-1)$, with probability at least $1 - O\left(\frac{\delta}{L}\right)$:
	\[ D \le \frac{12}{10} \cdot h \cdot \sqrt{ N^{(h+1)}_{i,j}(y) \cdot N^{(h+1)}_{i',j'}(z)}. \]
	By incorporating this into \eqref{eq:psi-bound1} it follows that if we condition on $P_2(h-1), P_1(h), P_1(h-1)$, then, with probability at least $1 - O\left(\frac{\delta}{L}\right)$, the following holds simultaneously for all $i,i' \in [d_1]$ and all $j,j' \in [d_2]$,
	\begin{equation}\label{eq:psi-bound2}
		\begin{split}
			&\left| \left< \psi_{i,j}^{(h)}(y) , \psi_{i',j'}^{(h)}(z) \right> - \sum_{a=-\frac{q-1}{2}}^{\frac{q-1}{2}} \sum_{b=-\frac{q-1}{2}}^{\frac{q-1}{2}}  \left< \eta_{i+a,j+b}^{(h)}(y), \eta_{i'+a,j'+b}^{(h)}(z) \right> \right|\\ 
			&\qquad\qquad\le O\left( \epsilon h/{L} \right) \cdot \sqrt{ N^{(h+1)}_{i,j}(y) \cdot N^{(h+1)}_{i',j'}(z)}.
		\end{split}
	\end{equation}
	
	Now we bound the term $\left|\left<\eta^{(h)}_{i,j}(y), \eta^{(h)}_{i',j'}(z)\right> - \langle f_{i,j},g_{i',j'} \rangle - \left<\phi_{i,j}^{(h)}(y), \phi_{i',j'}^{(h)}(z)\right> \right|$ using \eqref{eta-innerprod-bound}, \eqref{cntk:phi-norm-bound-final}, and Lemma~\ref{properties-gamma-dot} along with inductive hypotheses $P_2(h-1)$ and Lemma~\ref{prop-pi}. With probability at least $1 - O\left(\frac{\delta}{L}\right)$ the following holds simultaneously for all $i,i' \in [d_1]$ and all $j,j' \in [d_2]$:
	\[
	\begin{split}
		&\left|\left<\eta^{(h)}_{i,j}(y), \eta^{(h)}_{i',j'}(z)\right> - \langle f_{i,j},g_{i',j'} \rangle - \left<\phi_{i,j}^{(h)}(y), \phi_{i',j'}^{(h)}(z)\right> \right| \\
		&\qquad \le O\left( \frac{\epsilon}{L} \right) \cdot \sqrt{\Pi_{i,j,i,j}^{(h-1)}(y,y) \cdot \dot{\Gamma}_{i,j,i,j}^{(h)}(y,y)\cdot \Pi_{i',j',i',j'}^{(h-1)}(z,z) \cdot \dot{\Gamma}_{i',j',i',j'}^{(h)}(z,z)}\\
		&\qquad = O\left( \frac{\epsilon \cdot h}{L} \right) \cdot \frac{ \sqrt{N_{i,j}^{(h)}(y) \cdot N_{i',j'}^{(h)}(z)}}{q^2},
	\end{split}
	\]
	where the last line above follows from Lemma~\ref{prop-pi} together with the fact that $\dot{\Gamma}_{i,j,i,j}^{(h)}(y,y) = \dot{\Gamma}_{i',j',i',j'}^{(h)}(z,z) = \frac{1}{q^2}$.
	
	By combining the above with inductive hypotheses $P_1(h), P_2(h-1)$ and \eqref{cntk:phi-dot-bound-final} via triangle inequality and invoking Lemma~\ref{prop-pi} we get that the following holds simultaneously for all $i,i' \in [d_1]$ and all $j,j' \in [d_2]$, with probability at least $1 - O\left(\frac{\delta}{L}\right)$,
	\small
	\begin{align}\label{eta-innerprod-bound2}
		&\left|\left<\eta^{(h)}_{i,j}(y), \eta^{(h)}_{i',j'}(z)\right> - \Pi_{i,j,i',j'}^{(h-1)}(y,z)\cdot \dot{\Gamma}_{i,j,i',j'}^{(h)}(y,z) - \Gamma_{i,j,i',j'}^{(h)}(y,z) \right| \nonumber\\
		& \le \frac{\epsilon}{10} \cdot \frac{(h-1)^2}{L+1} \cdot \sqrt{N_{i,j}^{(h)}(y) \cdot N_{i',j'}^{(h)}(z)} \cdot \left(\left| \dot{\Gamma}_{i,j,i',j'}^{(h)}(y,z) \right|+ \frac{1}{q^2} \cdot \frac{\epsilon}{8L} \right) + \frac{1}{q^2} \cdot \frac{\epsilon}{8L} \cdot \left| \Pi_{i,j,i',j'}^{(h-1)}(y,z)\right| \nonumber\\ 
		& + \frac{(h+1) \cdot \epsilon^2}{60L^3} \cdot \frac{\sqrt{N_{i,j}^{(h)}(y) \cdot N_{i',j'}^{(h)}(z)}}{q^2} + O\left( \frac{\epsilon \cdot h}{L} \right) \cdot \frac{ \sqrt{N_{i,j}^{(h)}(y) \cdot N_{i',j'}^{(h)}(z)}}{q^2}\nonumber\\
		& \le \frac{\epsilon}{10} \cdot \frac{(h-1)^2}{L+1} \cdot \frac{\sqrt{N_{i,j}^{(h)}(y) \cdot N_{i',j'}^{(h)}(z)}}{q^2} \cdot \left(1 + \frac{\epsilon}{8L} \right) + \frac{h-1}{q^2} \cdot \frac{\epsilon}{8L} \cdot \sqrt{N_{i,j}^{(h)}(y) \cdot N_{i',j'}^{(h)}(z)}\nonumber\\
		&+ \left( \frac{(h+1) \cdot \epsilon^2}{60L^3} + O\left( \frac{\epsilon \cdot h}{L} \right) \right) \cdot \frac{\sqrt{N_{i,j}^{(h)}(y) \cdot N_{i',j'}^{(h)}(z)}}{q^2} \nonumber\\
		&\le \frac{\epsilon}{10} \cdot \frac{h^2-h/2}{L+1} \cdot \frac{\sqrt{N_{i,j}^{(h)}(y) \cdot N_{i',j'}^{(h)}(z)}}{q^2}.\nonumber
	\end{align}
	\normalsize
	By plugging the above bound into \eqref{eq:psi-bound2} using triangle inequality and using \eqref{eq:dp-cntk} we get the following with probability at least $1 - O\left(\frac{\delta}{L}\right)$:
	\begin{equation}\label{eq:psi-bound3}
		\begin{split}
			&\left| \left< \psi_{i,j}^{(h)}(y) , \psi_{i',j'}^{(h)}(z) \right> - \Pi_{i,j,i',j'}^{(h)}(y,z) \right|\\ 
			&\le O\left( \epsilon h/{L} \right) \cdot \sqrt{ N^{(h+1)}_{i,j}(y) \cdot N^{(h+1)}_{i',j'}(z)} \\ 
			&\qquad+ \frac{\epsilon}{10} \cdot \frac{h^2-h/2}{L+1} \cdot \sum_{a=-\frac{q-1}{2}}^{\frac{q-1}{2}} \sum_{b=-\frac{q-1}{2}}^{\frac{q-1}{2}}  \frac{\sqrt{N_{i+a,j+b}^{(h)}(y) \cdot N_{i'+a,j'+b}^{(h)}(z)}}{q^2}\\
			&\le O\left( \epsilon h/{L} \right) \cdot \sqrt{ N^{(h+1)}_{i,j}(y) \cdot N^{(h+1)}_{i',j'}(z)} \\ 
			&\qquad+ \frac{\epsilon}{10} \cdot \frac{h^2-h/2}{L+1} \cdot \sqrt{\sum_{a=-\frac{q-1}{2}}^{\frac{q-1}{2}} \sum_{b=-\frac{q-1}{2}}^{\frac{q-1}{2}} \frac{N_{i+a,j+b}^{(h)}(y)}{q^2}} \cdot \sqrt{ \sum_{a=-\frac{q-1}{2}}^{\frac{q-1}{2}} \sum_{b=-\frac{q-1}{2}}^{\frac{q-1}{2}} \frac{N_{i'+a,j'+b}^{(h)}(z)}{q^2}}\\
			&\le \frac{\epsilon}{10} \cdot \frac{h^2}{L+1} \cdot\sqrt{ N^{(h+1)}_{i,j}(y) \cdot N^{(h+1)}_{i',j'}(z)}.
		\end{split}
	\end{equation}
	Similarly, we can prove that with probability at least $1 - O\left( \frac{\delta}{L} \right)$ the following hold simultaneously for all $i,i' \in [d_1]$ and all $j,j' \in [d_2]$,
	\[ \begin{split}
		&\left| \left\| \psi_{i,j}^{(h)}(y)\right\|_2^2 - \Pi_{i,j,i,j}^{(h)}(y,y) \right| \le \frac{\epsilon}{10} \cdot \frac{h^2}{L+1} \cdot N^{(h+1)}_{i,j}(y),\\ 
		&\left| \left\| \psi_{i',j'}^{(h)}(z)\right\|_2^2 - \Pi_{i',j',i',j'}^{(h)}(z,z) \right| \le \frac{\epsilon}{10} \cdot \frac{h^2}{L+1} \cdot N^{(h+1)}_{i',j'}(z).
	\end{split} \]
	This is sufficient to prove the inductive step for statement $P_2(h)$, in the case of $h<L$, i.e., $\Pr[P_2(h)|P_2(h-1), P_1(h), P_1(h-1)] \ge 1 - O(\delta/L)$.
	
	Now we prove the inductive step for $P_2(h)$ in the case of $h=L$. Similar to before, if we let $f_{i,j} := \psi^{(L-1)}_{i,j}(y) \otimes \dot{\phi}_{i,j}^{(L)}(y)$ and $g_{i',j'} := \psi^{(L-1)}_{i',j'}(z) \otimes \dot{\phi}_{i',j'}^{(L)}(z)$, then by \eqref{psi-cntk-last}, we have $\psi_{i,j}^{(L)}(y) = \left(Q^2\cdot f_{i,j}\right)$ and $\psi_{i',j'}^{(L)}(z) = \left(Q^2\cdot g_{i',j'}\right)$. Thus by Lemma~\ref{soda-result} and union bound, we find that, with probability at least $1 - O\left( \frac{\delta}{ L} \right)$, the following inequality holds simultaneously for all $i,i' \in [d_1]$ and $j,j' \in [d_2]$:
	\[ \left|\left<\psi^{(L)}_{i,j}(y), \psi^{(L)}_{i',j'}(z)\right> - \langle f_{i,j},g_{i',j'} \rangle \right| \le O\left( \frac{\epsilon}{L} \right) \cdot \left\| f_{i,j} \right\|_2 \left\| g_{i',j'} \right\|_2.\]
	Therefore, using \eqref{cntk:phi-norm-bound-final} and Lemma~\ref{properties-gamma-dot} along with inductive hypotheses $P_2(L-1)$ and Lemma~\ref{prop-pi}, with probability at least $1 - O\left(\frac{\delta}{L}\right)$, the following holds simultaneously for all $i,i' \in [d_1]$ and $j,j' \in [d_2]$,
	\[
	\begin{split}
		&\left|\left<\psi^{(L)}_{i,j}(y), \psi^{(L)}_{i',j'}(z)\right> - \langle f_{i,j},g_{i',j'} \rangle \right| \\
		&\qquad \le O\left( \frac{\epsilon}{L} \right) \cdot \sqrt{\Pi_{i,j,i,j}^{(L-1)}(y,y) \cdot \dot{\Gamma}_{i,j,i,j}^{(L)}(y,y)\cdot \Pi_{i',j',i',j'}^{(L-1)}(z,z) \cdot \dot{\Gamma}_{i',j',i',j'}^{(L)}(z,z)}\\
		&\qquad = O\left( {\epsilon} \right) \cdot \frac{\sqrt{N_{i,j}^{(L)}(y) \cdot N_{i',j'}^{(L)}(z)}}{q^2}.
	\end{split}
	\]

	By combining the above with inductive hypotheses $P_1(L), P_2(L-1)$ and \eqref{cntk:phi-dot-bound-final} via triangle inequality and invoking Lemma~\ref{prop-pi} and also using the definition of $\Pi^{(L)}(y,z)$ given in \eqref{eq:dp-cntk-last-layer}, we get that the following holds, simultaneously for all $i,i' \in [d_1]$ and $j,j' \in [d_2]$, with probability at least $1 - O\left(\frac{\delta}{L}\right)$,
	\begin{align}
		&\left|\left<\psi^{(L)}_{i,j}(y), \psi^{(L)}_{i',j'}(z)\right> - \Pi_{i,j,i',j'}^{(L)}(y,z) \right| \nonumber\\
		& \le \frac{\epsilon}{10} \cdot \frac{(L-1)^2}{L+1} \cdot \sqrt{N_{i,j}^{(L)}(y) \cdot N_{i',j'}^{(L)}(z)} \cdot \left(\left| \dot{\Gamma}_{i,j,i',j'}^{(L)}(y,z) \right|+ \frac{1}{q^2} \cdot \frac{\epsilon}{8L} \right) + \frac{1}{q^2} \cdot \frac{\epsilon}{8L} \cdot \left| \Pi_{i,j,i',j'}^{(L-1)}(y,z)\right| \nonumber\\ 
		& + \frac{(L+1) \cdot \epsilon^2}{60L^3} \cdot \frac{\sqrt{N_{i,j}^{(L)}(y) \cdot N_{i',j'}^{(L)}(z)}}{q^2} + O\left( {\epsilon} \right) \cdot \frac{ \sqrt{N_{i,j}^{(L)}(y) \cdot N_{i',j'}^{(L)}(z)}}{q^2}\nonumber\\
		& \le \frac{\epsilon}{10} \cdot \frac{(L-1)^2}{L+1} \cdot \frac{\sqrt{N_{i,j}^{(L)}(y) \cdot N_{i',j'}^{(L)}(z)}}{q^2} \cdot \left(1 + \frac{\epsilon}{8L} \right) +  \frac{\epsilon}{8q^2} \cdot \sqrt{N_{i,j}^{(L)}(y) \cdot N_{i',j'}^{(L)}(z)}\nonumber\\
		&+ \left( \frac{(L+1) \cdot \epsilon^2}{60L^3} + O\left( {\epsilon} \right) \right) \cdot \frac{\sqrt{N_{i,j}^{(L)}(y) \cdot N_{i',j'}^{(L)}(z)}}{q^2} \nonumber\\
		&\le \frac{\epsilon \cdot (L-1)}{10} \cdot \frac{\sqrt{N_{i,j}^{(L)}(y) \cdot N_{i',j'}^{(L)}(z)}}{q^2}.\nonumber
	\end{align}
	This proves the inductive step for statement $P_2(h)$, in the case of $h=L$, i.e., $\Pr[P_2(L)|P_2(L-1), P_1(L), P_1(L-1)] \ge 1 - O(\delta/L)$.
	The induction is complete and hence the statements of lemma are proved by union bounding over all $h = 0,1,2, \ldots L$.
\end{proof}

In the following lemma we analyze the runtime of the CNTK Sketch algorithm,
\begin{lemma}[Runtime of the CNTK Sketch]
	\label{thm:cntk-sketch-runtime}
	For every positive integers $d_1,d_2,c$, and $L$, every $\epsilon, \delta>0$, every image $x \in \RR^{d_1\times d_2\times c}$, the time to compute the CNTK Sketch $\Psi_{cntk}^{(L)}(x) \in \RR^{s^*}$, for $s^*=O\left( \frac{1}{\epsilon^2} \cdot \log \frac{1}{\delta} \right)$, using the procedure given in Definition~\ref{alg-def-cntk-sketch} is bounded by $O\left( \frac{L^{11}}{\epsilon^{6.7}} \cdot (d_1d_2) \cdot \log^3 \frac{d_1d_2L}{\epsilon\delta}\right)$.
\end{lemma}
\begin{proof}
	First note that the total time to compute $N_{i,j}^{(h)}(x)$ for all $i \in [d_1]$ and $j \in [d_2]$ and $h=0,1, \ldots L$ as per \eqref{eq:dp-cntk-norm-simplified} is bounded by $O\left(q^2 L \cdot d_1d_2\right)$.
	Besides the time to compute $N_{i,j}^{(h)}(x)$, there are two other main components to the runtime of this procedure. The first heavy operation corresponds to computing vectors $\left[Z_{i,j}^{(h)}(x)\right]_l = Q^{2p+2} \cdot \left(\left[ \mu_{i,j}^{(h)}(x) \right]^{\otimes l} \otimes e_1^{\otimes 2p+2-l}\right)$ for $l=0,1,2, \ldots 2p+2$ and $h=1,2, \ldots L$ and all indices $i \in[d_1]$ and $j \in [d_2]$, in \eqref{eq:maping-cntk-covar}. By Lemma~\ref{soda-result}, the time to compute $\left[Z_{i,j}^{(h)}(x)\right]_l$ for a fixed $h$, fixed $i \in[d_1]$ and $j \in [d_2]$, and all $l=0,1,2, \ldots 2p+2$ is bounded by,
	\[ O\left( L^{10}/\epsilon^{20/3} \cdot \log^2\frac{L}{\epsilon} \log^3 \frac{d_1d_2L}{\epsilon\delta} + q^2 \cdot L^8/\epsilon^{16/3} \cdot \log^3 \frac{d_1d_2L}{\epsilon\delta} \right) = O\left( \frac{ L^{10}}{\epsilon^{6.7}} \cdot \log^3 \frac{d_1d_2L}{\epsilon\delta} \right). \]
	The total time to compute vectors $\left[Z_{i,j}^{(h)}(x)\right]_l$ for all $h=1,2, \ldots L$ and all $l=0,1,2, \ldots 2p+2$ and all indiced $i \in [d_1]$ and $j \in [d_2]$ is thus bounded by $O\left( \frac{L^{11}}{\epsilon^{6.7}} \cdot (d_1d_2) \cdot \log^3 \frac{d_1d_2L}{\epsilon\delta} \right)$. 
	The next computationally expensive operation is computing vectors $\left[Y_{i,j}^{(h)}(x)\right]_l$ for $l=0,1,2, \ldots 2p'+1$ and $h=1,2, \ldots L$, and all indices $i \in [d_1]$ and $j \in [d_2]$, in \eqref{eq:cntk-map-phidot}.
	By Lemma~\ref{soda-result}, the runtime of computing $\left[Y_{i,j}^{(h)}(x)\right]_l$ for a fixed $h$, fixed $i \in [d_1]$ and $j \in [d_2]$, and all $l=0,1,2, \ldots 2p'+1$ is bounded by,
	\[ O\left( \frac{L^{6}}{\epsilon^6} \cdot \log^2\frac{L}{\epsilon} \log^3 \frac{d_1d_2L}{\epsilon\delta} + \frac{q^2 \cdot L^{8}}{\epsilon^6} \cdot \log^3 \frac{d_1d_2L}{\epsilon\delta} \right) = O\left( \frac{ L^{8}}{\epsilon^6} \log^2\frac{L}{\epsilon} \cdot \log^3 \frac{d_1d_2L}{\epsilon\delta} \right). \]
	Hence, the total time to compute vectors $\left[Y_{i,j}^{(h)}(x)\right]_l$ for all $h=1,2, \ldots L$ and $l=0,1,2, \ldots 2p'+1$ and all indiced $i \in [d_1]$ and $j \in [d_2]$ is $O\left( \frac{ L^{9}}{\epsilon^6} \log^2\frac{L}{\epsilon} \cdot (d_1d_2) \cdot \log^3 \frac{d_1d_2L}{\epsilon\delta} \right)$. The total runtime bound is obtained by summing up these three contributions.
\end{proof}

Now we are ready to prove Theorem~\ref{maintheorem-cntk},

\begin{proofof}{Theorem~\ref{maintheorem-cntk}}
	Let $\psi^{(L)}:\RR^{d_1\times d_2\times c} \to \RR^{d_1 \times d_2 \times s}$ for $s=O\left( \frac{L^4}{\epsilon^2} \cdot \log^3 \frac{d_1d_2L}{\epsilon\delta} \right)$ be the mapping defined in \eqref{psi-cntk-last} of Definition~\ref{alg-def-cntk-sketch}. By \eqref{Psi-cntk-def}, the CNTK Sketch $\Psi_{cntk}^{(L)}(x)$ is defined as 
	\[\Psi_{ntk}^{(L)}(x):=\frac{1}{d_1d_2} \cdot  G \cdot \left(\sum_{i \in [d_1]} \sum_{j \in [d_2]} \psi^{(L)}_{i,j}(x)\right).\]
	The matrix $G$ is defined in \eqref{Psi-cntk-def} to be a matrix of i.i.d. normal entries with $s^{*} = C \cdot \frac{1}{\epsilon^2} \cdot \log\frac{1}{\delta}$ rows for large enough constant $C$. \cite{dasgupta2003elementary} shows that $G$ is a JL transform and hence $\Psi_{cntk}^{(L)}$ satisfies the following,
	\[ \Pr \left[ \left| \left< \Psi_{cntk}^{(L)}(y) , \Psi_{cntk}^{(L)}(z) \right> - \frac{1}{d_1^2 d_2^2} \cdot \sum_{i,i'\in[d_1]}\sum_{j ,j'\in [d_2]} \left< \psi_{i,j}^{(L)}( y),  \psi_{i',j'}^{(L)}( z) \right> \right| \le O(\epsilon) \cdot A \right] \ge 1 - O(\delta), \]
	where $A := \frac{1}{d_1^2 d_2^2} \cdot \left\| \sum_{i\in[d_1]}\sum_{j \in [d_2]}  \psi_{i,j}^{(L)}( y)\right\|_2 \cdot \left\|\sum_{i\in[d_1]}\sum_{j \in [d_2]} \psi_{i,j}^{(L)}(z)\right\|_2$. By triangle inequality together with Lemma~\ref{lem:cntk-sketch-corr} and Lemma~\ref{prop-pi}, the following bounds hold with probability at least $1 - O(\delta)$:
	\[ \begin{split}
		&\left\| \sum_{i\in[d_1]}\sum_{j \in [d_2]}  \psi_{i,j}^{(L)}( y)\right\|_2 \le \frac{11}{10} \cdot \frac{\sqrt{L-1}}{q} \cdot \sum_{i\in[d_1]}\sum_{j \in [d_2]} \sqrt{N^{(L)}_{i,j}(y)},\\
		&\sum_{i\in[d_1]}\sum_{j \in [d_2]} \left\| \psi_{i,j}^{(L)}( z)\right\|_2^2 \le \frac{11}{10} \cdot \frac{\sqrt{L-1}}{q} \cdot \sum_{i\in[d_1]}\sum_{j \in [d_2]} \sqrt{N^{(L)}_{i,j}(z)},
	\end{split} \]
	Therefore, by union bound we find that, with probability at least $1 - O(\delta)$:
	\[\begin{split}
		&\left| \left< \Psi_{cntk}^{(L)}(y) , \Psi_{cntk}^{(L)}(z) \right> - \frac1{d_1^2 d_2^2} \cdot {\sum_{i,i' \in [d_1]}\sum_{j,j' \in [d_2]} \left< \psi_{i,j}^{(L)}( y),  \psi_{i',j'}^{(L)}( z) \right>} \right|\\ 
		&\qquad \le O\left(\frac{\epsilon L}{q^2 \cdot d_1^2 d_2^2}\right) \cdot \sum_{i, i' \in[d_1]}\sum_{j, j' \in [d_2]} \sqrt{N^{(L)}_{i,j}(y)\cdot  N^{(L)}_{i',j'}(z)}.
	\end{split} \]
	Be combining the above with Lemma~\ref{lem:cntk-sketch-corr} using triangle inequality and union bound and also using \eqref{eq:dp-cntk-finalkernel}, the following holds with probability at least $1 - O(\delta)$:
	\begin{equation}\label{Psi-error-bound1}
		\left| \left< \Psi_{cntk}^{(L)}(y) , \Psi_{cntk}^{(L)}(z) \right> - \Theta_{cntk}^{(L)}(y,z) \right| \le \frac{\epsilon \cdot (L-1)}{9 q^2 \cdot d_1^2 d_2^2} \cdot \sum_{i, i' \in[d_1]}\sum_{j, j' \in [d_2]} \sqrt{N^{(L)}_{i,j}(y)\cdot  N^{(L)}_{i',j'}(z)}. 
	\end{equation}
	
	Now we prove that $\Theta_{cntk}^{(L)}(y,z) \ge \frac{L-1}{9q^2 d_1^2 d_2^2} \cdot \sum_{i,i'\in[d_1]}\sum_{j,j' \in [d_2]} \sqrt{N^{(L)}_{i,j}(y) \cdot N^{(L)}_{i',j'}(z)}$ for every $L \ge 2$.
	First note that, it follows from \eqref{eq:dp-cntk-covar-simplified} that $\Gamma_{i,j,i',j'}^{(1)}(y,z) \ge 0$ for any $i,i',j,j'$ because the function $\kappa_1$ is non-negative everywhere on $[-1,1]$. This also implies that $\Gamma_{i,j,i',j'}^{(2)}(y,z) \ge \frac{\sqrt{N^{(2)}_{i,j}(y) \cdot N^{(2)}_{i',j'}(z)}}{\pi \cdot q^2}$ because $\kappa_1(\alpha) \ge \frac{1}{\pi}$ for every $\alpha \in [0,1]$. Since, $\kappa_1(\cdot)$ is a monotone increasing function, by recursively using \eqref{eq:dp-cntk-norm-simplified} and \eqref{eq:dp-cntk-covar-simplified} along with Lemma~\ref{properties-gamma}, we can show that for every $h \ge 1$, the value of $\Gamma_{i,j,i',j'}^{(h)}(y,z)$ is lower bounded by $\frac{\sqrt{N^{(h)}_{i,j}(y) \cdot N^{(h)}_{i',j'}(z)}}{q^2} \cdot \Sigma_{relu}^{(h)}(-1)$, where $\Sigma_{relu}^{(h)}:[-1,1] \to \RR$ is the function defined in \eqref{eq:dp-covar-relu}.

	Furthermore, it follows from \eqref{eq:dp-cntk-derivative-covar-simplified} that $\dot{\Gamma}_{i,j,i',j'}^{(1)}(y,z) \ge 0$ for any $i,i',j,j'$ because the function $\kappa_0$ is non-negative everywhere on $[-1,1]$. Additionally, $\dot{\Gamma}_{i,j,i',j'}^{(2)}(y,z) \ge \frac{1}{2 q^2}$ because $\kappa_0(\alpha) \ge \frac{1}{2}$ for every $\alpha \in [0,1]$. By using the inequality $\Gamma_{i,j,i',j'}^{(h)}(y,z) \ge \frac{\sqrt{N^{(h)}_{i,j}(y) \cdot N^{(h)}_{i',j'}(z)}}{q^2} \cdot \Sigma_{relu}^{(h)}(-1)$ that we proved above along with the fact that $\kappa_0(\cdot)$ is a monotone increasing function and recursively using  \eqref{eq:dp-cntk-derivative-covar-simplified} and Lemma~\ref{properties-gamma}, it follows that for every $h \ge 1$, we have $\dot{\Gamma}_{i,j,i',j'}^{(h)}(y,z) \ge \frac{1}{q^2} \cdot \dot{\Sigma}_{relu}^{(h)}(-1)$. 
	
	By using these inequalities and Definition of $\Pi^{(h)}$ in \eqref{eq:dp-cntk} together with \eqref{eq:dp-cntk-norm-simplified}, recursively, it follows that, for every $i,i',j,j'$ and $h = 2,\ldots L-1$:
	\begin{align*}
		\Pi_{i,j,i',j'}^{(h)}(y,z) \ge \frac{h}{4} \cdot \sqrt{N_{i,j}^{(h+1)}(y) \cdot N_{i',j'}^{(h+1)}(z)},
	\end{align*}
	Therefore, using this inequality and \eqref{eq:dp-cntk-last-layer} we have that for every $L\ge2$:
	\begin{align*}
		\Pi_{i,j,i',j'}^{(L)}(y,z) &\ge \frac{L-1}{4} \cdot \sqrt{N_{i,j}^{(L)}(y) \cdot N_{i',j'}^{(L)}(z)} \cdot \frac{\dot{\Sigma}_{relu}^{(L)}(-1)}{q^2}\\
		&\ge \frac{L-1}{9q^2} \cdot \sqrt{N_{i,j}^{(L)}(y) \cdot N_{i',j'}^{(L)}(z)}.
	\end{align*}
	Now using this inequality and \eqref{eq:dp-cntk-finalkernel}, the following holds for every $L\ge2$:
	\[ \Theta_{cntk}^{(L)}(y,z) \ge \frac{L-1}{9q^2 d_1^2 d_2^2} \cdot \sum_{i,i'\in[d_1]}\sum_{j,j' \in [d_2]} \sqrt{N^{(L)}_{i,j}(y) \cdot N^{(L)}_{i',j'}(z)}. \]
	Therefore, by incorporating the above into \eqref{Psi-error-bound1} we get that,
	\[ \Pr \left[ \left| \left< \Psi_{cntk}^{(L)}(y) , \Psi_{cntk}^{(L)}(z) \right> - \Theta_{cntk}^{(L)}(y,z) \right| \le \epsilon \cdot \Theta_{cntk}^{(L)}(y,z) \right] \ge 1 - \delta. \]
	
	{\bf Runtime:} By Lemma~\ref{thm:cntk-sketch-runtime}, time to compute the CNTK Sketch is $O\left( \frac{ L^{11}}{\epsilon^{6.7}} \cdot (d_1d_2) \cdot \log^3 \frac{d_1d_2L}{\epsilon\delta} \right)$.
\end{proofof}

%%%%%%%%%%%%%%%%%%%%%%%%%%%%%%%%%%%%%%%%%%%%%%%%%%%%%%%%%%%%%%%%%%%%%%%%%%%%%%%
%%%%%%%%%%%%%%%%%%%%%%%%%%%%%%%%%%%%%%%%%%%%%%%%%%%%%%%%%%%%%%%%%%%%%%%%%%%%%%%

\end{document}